\documentclass[journal]{IEEEtran}

\newtheorem{theorem}{Theorem}[section]

\newtheorem{prop}[theorem]{Proposition}

\newtheorem{problem}{Problem}
\newtheorem{definition}[theorem]{Definition}
\newtheorem{rem}[theorem]{Remark}

\IEEEoverridecommandlockouts
\usepackage{graphicx} 
\graphicspath{{figures/}}
\usepackage{epsfig} 
\usepackage{ragged2e}
\usepackage{amsmath}
\usepackage{amssymb}
\usepackage{epstopdf}
\usepackage{cite}
\usepackage[noend,ruled,linesnumbered]{algorithm2e}
\SetKwComment{Comment}{$\triangleright$\ } {}
\usepackage{multirow}
\usepackage{rotating}
\usepackage{subfigure}
\usepackage{color}
\usepackage{mysymbol}
\usepackage{url}
\usepackage{ltlfonts}	

\begin{document}

\title{STyLuS*: A Temporal Logic Optimal Control Synthesis Algorithm for Large-Scale Multi-Robot Systems
}
\author{Yiannis~Kantaros,~\IEEEmembership{Student Member,~IEEE,} and Michael~M.~Zavlanos,~\IEEEmembership{Member,~IEEE}
\thanks{Yiannis Kantaros and Michael M. Zavlanos are with the Department of Mechanical Engineering and Materials Science, Duke University, Durham, NC 27708, USA. $\left\{\text{yiannis.kantaros,michael.zavlanos}\right\}$@duke.edu. This work is supported in part by ONR under agreement $\#N00014-18-1-2374$ and by AFOSR under the award $\#FA9550-19-1-0169$.}}
\maketitle


\begin{abstract}
This paper proposes a new highly scalable and asymptotically optimal control synthesis algorithm from linear temporal logic specifications, called  $\text{STyLuS}^{*}$ for large-Scale optimal Temporal Logic Synthesis, that is designed to solve complex temporal planning problems in large-scale multi-robot systems. Existing planning approaches with temporal logic specifications rely on graph search techniques applied to a product automaton constructed among the robots. In our previous work, we have proposed a more tractable sampling-based algorithm that builds incrementally trees that approximate the
state-space and transitions of the synchronous product automaton and does not require sophisticated graph search techniques.  Here, we extend our previous work by introducing bias in the sampling process which is guided by transitions in the B$\ddot{\text{u}}$chi automaton that belong to the shortest path to the accepting states. This allows us to synthesize optimal motion plans from product automata with hundreds of orders of magnitude more states than those that existing optimal control synthesis methods or off-the-shelf model checkers can manipulate. We show that $\text{STyLuS}^{*}$ is probabilistically complete and asymptotically optimal and has exponential convergence rate. This is the first time that convergence rate results are provided for sampling-based optimal control synthesis methods. We provide simulation results that show that $\text{STyLuS}^{*}$ can synthesize optimal motion plans for very large multi-robot systems which is impossible using state-of-the-art methods.
\end{abstract}

\section{Introduction}
\IEEEPARstart{C}{ontrol} synthesis for mobile robots under complex tasks, captured by Linear Temporal Logic (LTL) formulas, builds upon either bottom-up approaches when independent LTL expressions are assigned to robots \cite{kress2009temporal,kress2007s,bhatia2010sampling,ulusoy2014receding} or top-down approaches when a global LTL formula describing a collaborative task is assigned to a team of robots \cite{chen2011synthesis,chen2012formal}, as in this work. Common in the above works is that they rely on model checking theory \cite{baier2008principles,clarke1999model} to find paths that satisfy LTL-specified tasks, without optimizing task performance. Optimal control synthesis under local and global LTL specifications has been addressed in \cite{smith2011optimal,guo2015multi} and \cite{kloetzer2010automatic,ulusoy2013optimality,ulusoy2014optimal}, respectively. In top-down approaches \cite{kloetzer2010automatic,ulusoy2013optimality,ulusoy2014optimal}, optimal discrete plans are derived for every robot using the individual transition systems that capture robot mobility and a Non-deterministic B$\ddot{\text{u}}$chi Automaton (NBA) that represents the global LTL specification. Specifically, by taking the synchronous product among the transition systems and the NBA, a synchronous Product B$\ddot{\text{u}}$chi Automaton (PBA) can be constructed. Then, representing the latter automaton as a graph and using graph-search techniques, optimal motion plans can be derived that satisfy the global LTL specification and optimize a cost function. As the number of robots or the size of the NBA increases, the state-space of the product automaton grows exponentially and, as a result, graph-search techniques become intractable. The same holds for recent search-based $\text{A}^*$-type methods \cite{khalidi2018t}, although they can solve problems an order of magnitude larger than those that graph-search approaches can handle. Consequently, these motion planning algorithms scale poorly with the number of robots and the complexity of the assigned task. A more tractable approach is presented in \cite{schillinger2016decomposition,schillinger2017simultaneous} that identifies independent parts of the LTL formula and builds a local product automaton for each agent. Nevertheless, this approach can be applied only to finite LTL missions and does not have optimality guarantees.

To mitigate these issues, in our previous work we proposed a sampling-based optimal control synthesis algorithm that avoids the explicit construction of the product among the transition systems and the NBA \cite{kantaros2017Csampling}. Specifically, this algorithm builds incrementally directed trees that approximately represent the state-space and transitions among states of the product automaton. The advantage is that approximating the product automaton by a tree rather than representing it explicitly by an arbitrary graph, as existing works do, results in significant savings in resources both in terms of memory to save the associated data structures and in terms of computational cost in applying graph search techniques.


In this work, we propose a new highly scalable optimal control synthesis algorithm from LTL specifications, called STyLuS* for large-Scale optimal Temporal Logic Synthesis, that is designed to solve complex temporal planning problems in large-scale multi-robot systems. In fact, $\text{STyLuS}^{*}$ extends the sampling-based synthesis algorithm proposed in \cite{kantaros2017Csampling} by introducing bias in the sampling
process.  For this, we first exploit the structure of the atomic propositions to prune the NBA by removing transitions that can never be enabled. Then, we define a metric over the state-space of the NBA that captures the shortest path, i.e., the minimum number of feasible transitions, between any two NBA states. Given this metric, we define a probability distribution over the nodes that are reachable from the current tree so that nodes that are closer to the final/accepting states of the NBA are sampled with higher probability; no particular sampling probability is proposed in  \cite{kantaros2017Csampling}.  We show that introducing bias in the sampling process does not violate the sampling assumptions in \cite{kantaros2017Csampling} so that $\text{STyLuS}^{*}$ inherits the same probabilistic completeness and asymptotic optimality guarantees. Moreover, we provide exponential convergence bounds; the first in the field of optimal control synthesis methods. Note that such guarantees are not provided for the algorithm in \cite{kantaros2017Csampling}.
We show that by biasing the sampling process we can synthesize optimal motion plans from product automata with order $10^{800}$ states and beyond, which is hundreds of orders more states than those that any existing optimal control synthesis algorithms \cite{kantaros15asilomar,kantaros2017sampling,ulusoy2013optimality,ulusoy2014optimal,kantaros2017Csampling} can handle. For example, our algorithm in \cite{kantaros2017Csampling}, when implemented with uniform sampling, can optimally solve problems with order $10^{10}$ states, many more than existing methods \cite{kantaros15asilomar,kantaros2017sampling,ulusoy2013optimality,ulusoy2014optimal} but orders of magnitude fewer than $\text{STyLuS}^*$. Compared to off-the-shelf model-checkers, such as NuSMV \cite{cimatti2002nusmv} and nuXmv \cite{cavada2014nuxmv}, that can design feasible but not optimal motion plans, our proposed biased sampling-based algorithm can find feasible plans much faster. NuSMV can solve problems with order $10^{30}$ states, while nuXmv can handle infinite-state synchronous transition systems but it is slower than $\text{STyLuS}^{*}$. Note that $\text{STyLuS}^{*}$ can be implemented in a distributed way, as in our recent work \cite{kantaros2017Dsampling}, which can further decrease the computational time. 

Relevant sampling-based control synthesis methods are also presented in \cite{karaman2012sampling,vasile2013sampling}. These methods consider continuous state spaces and employ sampling-based methods to build discrete abstractions of the environment that are represented by graphs of arbitrary structure, e.g., as in \cite{karaman2011sampling,janson2015fast}  for point-to-point navigation. Once these abstractions become expressive enough to generate motion plans that satisfy the LTL specification, graph search methods are applied to the respective PBA to design a feasible path. However, representing the environment using graphs of arbitrary structure compromises scalability of temporal logic planning methods since, as the size of these graphs increases, more resources are needed to save the associated structure and search for optimal plans using graph search methods.
While our proposed sampling-based approach assumes that a discrete abstraction of the environment is available, as in \cite{conner2003composition,belta2004constructing,belta2005discrete,kloetzer2006reachability,boskos2015decentralized}, it build trees, instead of arbitrary graphs, to approximate the product automaton. Therefore, it is more economical in terms of memory requirements and does not require the application of expensive graph search techniques to find the optimal motion plan. Instead, it tracks sequences of parent nodes starting from desired accepting states. Combined with the proposed biased sampling approach, our method can handle much more complex planning problems with more robots and LTL tasks that correspond to larger NBAs. A more detailed comparison with \cite{karaman2012sampling,vasile2013sampling}  can be found in \cite{kantaros2017Csampling}. 

The advantage of using non-uniform sampling functions for sampling-based motion planning has also been demonstrated before, e.g., in \cite{van2005using,zucker2008adaptive,ichter2017learning}. For example, \cite{ichter2017learning} learns sampling distributions from demonstrations that are then used to bias sampling. This results in an order of magnitude improvement in success rates and path costs compared to uniform sampling-based methods. Probabilistically safe corridors that
are constructed using a learned approximate probabilistic model of a configuration space have also been recently proposed to enhance scalability of sampling-based motion planning and also minimize collision likelihood \cite{huh2019probabilistically}. Nevertheless,  all these approaches focus on simple point-to-point navigation tasks, unlike the method proposed here that can solve complex tasks captured by LTL formulas. Related are also Monte Carlo Tree Search (MCTS) algorithms that are used to find optimal decisions in games or planning problems and they also rely on building search trees \cite{zagoruyko2019monte,best2019dec,browne2012survey}. The main challenge in MCTS is to balance between exploration and exploitation, which can be achieved using, e.g., the Upper Confidence Bound for Trees (UCT)
algorithm \cite{kocsis2006bandit}. MCTS methods involve a \textit{simulation} step that requires to complete one random playout from the child node that is added to the tree. $\text{STyLuS}^*$ completely avoids this step which can be computationally expensive for complex planning problems that require a large horizon until the task is accomplished. Moreover, to the best of our knowledge, MCTS has not been applied to LTL planning problems.

A preliminary version of this work can be found in  \cite{kantaros2018largescale}. In  \cite{kantaros2018largescale} it is shown that the proposed biased sampling process satisfies the assumptions in \cite{kantaros2017Csampling} so that the proposed algorithm inherits the same probabilistic completeness and asymptotic optimality guarantees.
Compared to \cite{kantaros2018largescale}, here we additionally provide exponential convergence rate bounds. To the best of our knowledge, this the first sampling-based motion planning algorithm with temporal logic specifications that also has convergence-rate guarantees.  Moreover, we provide additional extensive simulation studies that show the effect of bias in sampling in the convergence rate of the algorithm, as well as scalability of our method with respect to the number of robots, the size of the transition systems, and the size of the NBA. We also compare our method to relevant state-of-the-art methods. To the best of our knowledge, this the first optimal control synthesis method for global temporal logic specifications with optimality and convergence guarantees that can be applied to large-scale multi-robot systems.


The rest of the paper is organized as follows. In Section \ref{sec:problem} we present the problem formulation. In Section \ref{sec:solution} we describe our proposed sampling-based planning algorithm and in Section \ref{sec:corr} we examine its correctness, optimality, and convergence rate. Numerical experiments are presented in Section \ref{sec:sim}.

\section{Problem Formulation}\label{sec:problem}
Consider $N$ mobile robots that live in a complex workspace $\mathcal{W}\subset\mathbb{R}^d$. 
We assume that there are $W$ disjoint regions of interest in $\mathcal{W}$. The $j$-th region is denoted by $r_j$ and it can be of any arbitrary shape.\footnote{For simplicity of notations we consider disjoint regions $r_j$. However, overlapping regions can also be considered by introducing additional states to  $\text{wTS}_i$ defined in Definition \ref{defn:wTS} that capture the presence of robot $i$ in more than one region. 
} Given the robot dynamics, robot mobility in the workspace $\ccalW$ can be represented by a weighted Transition System (wTS); see also Figure \ref{fig:wts}. The wTS for robot $i$ is defined as follows:

\begin{definition}[wTS]
A \textit{weighted Transition System} (wTS) for robot $i$, denoted by $\text{wTS}_{i}$ is a tuple $\text{wTS}_{i}=\left(\mathcal{Q}_{i}, q_{i}^0,\rightarrow_{i}, w_i, \mathcal{AP}_i,L_{i}\right)$ where: 
(a) $\mathcal{Q}_{i}=\bigcup_{j=1}^{W}\{q_{i}^{r_j}\}$ is the set of states, where a state $q_{i}^{r_j}$ indicates that robot $i$ is at location $r_j$; (b) $q_{i}^0\in\mathcal{Q}_{i}$ is the initial state of robot $i$; (c) $\rightarrow_{i}\subseteq\mathcal{Q}_{i}\times\mathcal{Q}_{i}$ is the transition relation for robot $i$. Given the robot dynamics, if there is a control input ${\bf{u}}_i$ that can drive robot $i$ from location $r_j$ to $r_e$, then there is a transition from state $q_i^{r_j}$ to $q_i^{r_e}$ denoted by $(q_i^{r_j}, q_i^{r_e})\in\rightarrow_i$; (d) $w_{i}:\mathcal{Q}_{i}\times\mathcal{Q}_{i}\rightarrow \mathbb{R}_+$ is a cost function that assigns weights/costs to each possible transition in wTS. For example, such costs can be associated with the distance that needs to be traveled by robot $i$ in order to move from state $q_i^{r_j}$ to state $q_i^{r_k}$; 
(e) $\mathcal{AP}_i=\bigcup_{j=1}^W\{\pi_{i}^{r_j}\}$ is the set of atomic propositions, where $\pi_{i}^{r_j}$ is true if robot $i$ is inside region $r_j$ and false otherwise; and (f) $L_{i}:\mathcal{Q}_{i}\rightarrow {\mathcal{AP}_i}$ is an observation/output function defined as $L_i(q_i^{r_j})=\pi_i^{r_j}$, for all $q_i^{r_j}\in\ccalQ_i$. 
\label{defn:wTS}
\end{definition} 

\begin{figure}[t]
\centering
  \includegraphics[width=1\linewidth]{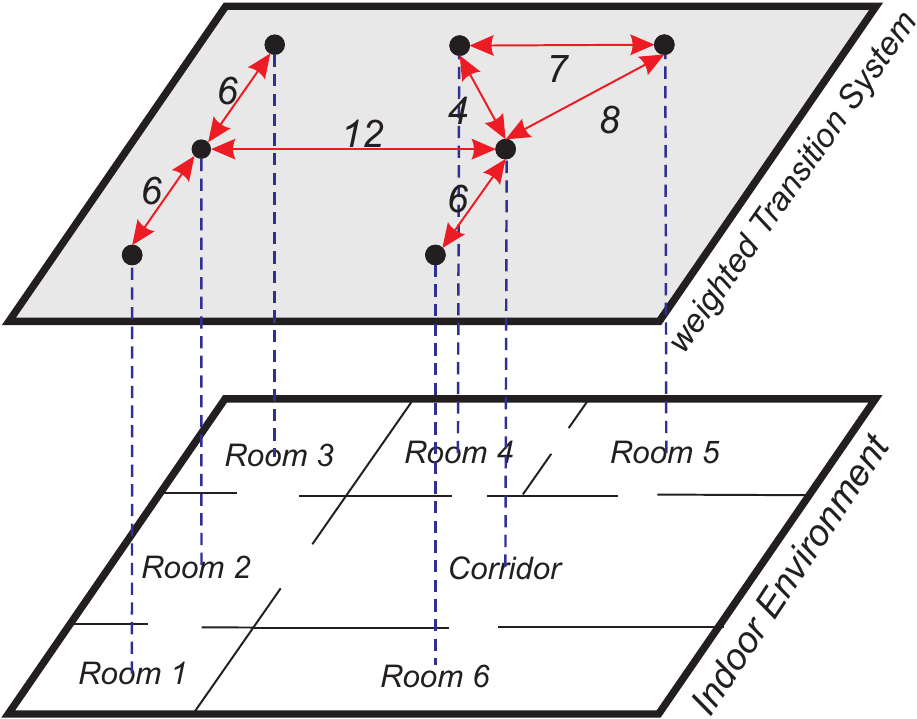}
  \caption{Graphical depiction of a wTS that abstracts robot mobility in an indoor environment. Black disks stand for the states of wTS, red edges capture transitions among states and numbers on these edges represent the cost $w_i$ for traveling from one state to another one.}
 \label{fig:wts}
\end{figure}

Given the definition of the wTS, we can define the synchronous \textit{Product Transition System} (PTS) as follows \cite{ulusoy2013optimality}:

\begin{definition}[PTS]
Given $N$ transition systems $\text{wTS}_i=(\mathcal{Q}_{i}, q_{i}^0, \rightarrow_{i},w_i,\mathcal{AP},L_{i})$, the \textit{product transition system} $\text{PTS}=\text{wTS}_{1}\otimes\text{wTS}_{2}\otimes\dots\otimes\text{wTS}_{N}$ is a tuple $\text{PTS}=(\mathcal{Q}_{\text{PTS}}, q_{\text{PTS}}^0,\longrightarrow_{\text{PTS}},w_{\text{PTS}},\mathcal{AP},L_{\text{PTS}})$ where
(a) $\mathcal{Q}_{\text{PTS}}=\mathcal{Q}_{1}\times\mathcal{Q}_{2}\times\dots\times\mathcal{Q}_{N}$ is the set of states;
(b) $q_{\text{PTS}}^0=(q_{1}^0,q_{2}^0,\dots,q_{N}^0)\in\mathcal{Q}_{\text{PTS}}$ is the initial state,
(c) $\longrightarrow_{\text{PTS}}\subseteq\mathcal{Q}_{\text{PTS}}\times\mathcal{Q}_{\text{PTS}}$ is the transition relation defined by the rule $\frac{\bigwedge _{\forall i}\left(q_{i}\rightarrow_{i}q_{i}'\right)}{q_{\text{PTS}}\rightarrow_{\text{PTS}}q_{\text{PTS}}'}$, where with slight abuse of notation $q_{\text{PTS}}=(q_{1},\dots,q_{N})\in\mathcal{Q}_{\text{PTS}}$, $q_{i}\in\mathcal{Q}_{i}$. The state $q_{\text{PTS}}'$ is defined accordingly. In words, this transition rule says that there exists a transition from $q_{\text{PTS}}$ to $q_{\text{PTS}}'$ if there exists a transition from $q_i$ to $q_i'$ for all $i\in\left\{1,\dots,N\right\}$; (d) $w_{\text{PTS}}:\mathcal{Q}_{\text{PTS}}\times\mathcal{Q}_{\text{PTS}}\rightarrow \mathbb{R}_+$ is a cost function that assigns weights/cost to each possible transition in PTS, defined as $	w_{\text{PTS}}(q_{\text{PTS}},q_{\text{PTS}}')=\sum_{i=1}^N w_i(\Pi|_{\text{wTS}_{i}}q_{\text{PTS}},\Pi|_{\text{wTS}_{i}}q_{\text{PTS}}')$, 
where $q_{\text{PTS}}',q_{\text{PTS}}\in\mathcal{Q}_{\text{PTS}}$, and $\Pi_{\text{wTS}_i}q_{\text{PTS}}$ stands for the projection of state $q_{\text{PTS}}$ onto the state space of $\text{wTS}_i$. The state $\Pi_{\text{wTS}_i}q_{\text{PTS}}\in\mathcal{Q}_i$ is obtained by removing all states in $q_{\text{PTS}}$ that do not belong to $\mathcal{Q}_i$; (e) $\mathcal{AP}=\bigcup_{i=1}^N\mathcal{AP}_i$ is the set of atomic propositions; and, (f) $L_{\text{PTS}}=\bigcup_{\forall i}L_{i}: \mathcal{Q}_{\text{PTS}}\rightarrow{\mathcal{AP}}$ is an observation/output function giving the set of atomic propositions that are satisfied at a state $q_{\text{PTS}}\in\mathcal{Q}_{\text{PTS}}$. 
\label{def:pts}
\end{definition} 

In what follows, we give definitions related to the $\text{PTS}$ that we will use throughout the rest of the paper. An \textit{infinite path} $\tau$ of a $\text{PTS}$ is an infinite sequence of states, $\tau=\tau(1)\tau(2)\tau(3)\dots$ such that $\tau(1)=q_{\text{PTS}}^0$, $\tau(k)\in\mathcal{Q}_{\text{PTS}}$, and $(\tau(k),\tau_{i}(k+1))\in\rightarrow_{\text{PTS}}$, $\forall k\in\mathbb{N}_+$, where $k$ is an index that points to the $k$-th entry of $\tau$ denoted by $\tau(k)$. The \textit{trace} of an infinite path $\tau=\tau(1)\tau(2)\tau(3)\dots$ of a PTS, denoted by $\texttt{trace}(\tau)\in\left(2^{\mathcal{AP}}\right)^{\omega}$, where $\omega$ denotes infinite repetition, is an infinite word that is determined by the sequence of atomic propositions that are true in the states along $\tau$, i.e., $\texttt{trace}(\tau)=L(\tau(1))L(\tau(2))\dots$.  A \textit{finite path} of a $\text{PTS}$ can be defined accordingly. The only difference with the infinite path is that a finite path is defined as a finite sequence of states of a $\text{PTS}$.  Given the definition of the weights $w_{\text{PTS}}$ in Definition \ref{def:pts}, the \textit{cost} of a finite path $\tau$, denoted by $\hat{J}(\tau)\geq0$, can be defined as
\begin{equation}\label{eq:cost}
\hat{J}(\tau)=\sum_{k=1}^{|\tau|-1}w_{\text{PTS}}(\tau(k),\tau(k+1)),
\end{equation} 
where, $|\tau|$ stands for the number of states in $\tau$. In words, the cost \eqref{eq:cost} captures the total cost incurred by all robots during the execution of the finite path $\tau$. 



We assume that the robots have to accomplish a complex collaborative task captured by a global LTL statement $\phi$ defined over the set of atomic propositions $\mathcal{AP}=\bigcup_{i=1}^N\mathcal{AP}_i$. Due to space limitations, we abstain from formally defining the semantics and syntax of LTL. A detailed overview can be found in \cite{baier2008principles}. Given an LTL formula $\phi$, we define the \textit{language} $\texttt{Words}(\phi)=\left\{\sigma\in (2^{\mathcal{AP}})^{\omega}|\sigma\models\phi\right\}$, where $\models\subseteq (2^{\mathcal{AP}})^{\omega}\times\phi$ is the satisfaction relation, as the set of infinite words $\sigma\in (2^{\mathcal{AP}})^{\omega}$ that satisfy the LTL formula $\phi$. Any LTL formula $\phi$ can be translated into a Nondeterministic B$\ddot{\text{u}}$chi Automaton (NBA) over $(2^{\mathcal{AP}})^{\omega}$ denoted by $B$  defined as follows \cite{vardi1986automata}.
\begin{definition}[NBA]
A \textit{Nondeterministic B$\ddot{\text{u}}$chi Automaton} (NBA) $B$ over $2^{\mathcal{AP}}$ is defined as a tuple $B=\left(\ccalQ_{B}, \ccalQ_{B}^0,\Sigma,\rightarrow_B,\mathcal{Q}_B^F\right)$, where $\ccalQ_{B}$ is the set of states, $\ccalQ_{B}^0\subseteq\ccalQ_{B}$ is a set of initial states, $\Sigma=2^{\mathcal{AP}}$ is an alphabet, $\rightarrow_{B}\subseteq\ccalQ_{B}\times \Sigma\times\ccalQ_{B}$ is the transition relation, and $\ccalQ_B^F\subseteq\ccalQ_{B}$ is a set of accepting/final states. 
\end{definition}

Given the $\text{PTS}$ and the NBA $B$ that corresponds to the LTL $\phi$, we can now define the \textit{Product B$\ddot{\text{u}}$chi Automaton} (PBA) $P=\text{PTS}\otimes B$, as follows \cite{baier2008principles}:

\begin{definition}[PBA]\label{defn:pba}
Given the product transition system $\text{PTS}=(\mathcal{Q}_{\text{PTS}},\allowbreak q_{\text{PTS}}^0,\longrightarrow_{\text{PTS}},w_{\text{PTS}},\mathcal{AP},L_{\text{PTS}})$ and the NBA $B=(\mathcal{Q}_{B}, \mathcal{Q}_{B}^0,\Sigma,\rightarrow_B,\mathcal{Q}_{B}^{F})$, we can define the \textit{Product B$\ddot{\text{u}}$chi Automaton} $P=\text{PTS}\otimes B$ as a tuple $P=(\mathcal{Q}_P, \mathcal{Q}_P^0,\longrightarrow_{P},w_P,\mathcal{Q}_P^F)$ where
(a) $\mathcal{Q}_P=\mathcal{Q}_{\text{PTS}}\times\mathcal{Q}_{B}$ is the set of states; (b) $\mathcal{Q}_P^0=q_{\text{PTS}}^0\times\mathcal{Q}_B^0$ is a set of initial states;
(c) $\longrightarrow_{P}\subseteq\mathcal{Q}_P\times 2^{\mathcal{AP}}\times\mathcal{Q}_P$ is the transition relation defined by the rule: $\frac{(q_{\text{PTS}}\rightarrow_{\text{PTS}} q_{\text{PTS}}')\wedge( q_{B}\xrightarrow{L_{\text{PTS}}\left(q_{\text{PTS}}\right)}q_{B}')}{q_{P}=\left(q_{\text{PTS}},q_{B}\right)\longrightarrow_P q_{P}'=\left(q_{\text{PTS}}',q_{B}'\right)}$. Transition from state $q_P\in\mathcal{Q}_P$ to $q_P'\in\mathcal{Q}_P$, is denoted by $(q_P,q_P')\in\longrightarrow_P$, or $q_P\longrightarrow_P q_P'$; (d) $w_P(q_{\text{P}},q_{\text{P}}')=w_{\text{PTS}}(q_{\text{PTS}},q_{\text{PTS}}')$,  where $q_{\text{P}}=(q_{\text{PTS}},q_B)$ and $q_{\text{P}}'=(q_{\text{PTS}}',q_B')$; and
(e) $\mathcal{Q}_P^F=\mathcal{Q}_{\text{PTS}}\times\mathcal{Q}_B^F$ is a set of accepting/final states. 
\end{definition} 

Given $\phi$ and the PBA an infinite path $\tau$ of a $\text{PTS}$ satisfies $\phi$ if and only if $\texttt{trace}(\tau)\in\texttt{Words}(\phi)$, which is equivalently denoted by $\tau\models\phi$. Specifically, if there is
a path satisfying $\phi$, then there exists a path  $\tau\models\phi$ that can be written in a finite representation, called prefix-suffix structure, i.e., $\tau=\tau^{\text{pre}}[\tau^{\text{suf}}]^{\omega}$, where the prefix part $\tau^{\text{pre}}$ is executed only once followed by the indefinite execution of the suffix part $\tau^{\text{suf}}$. The prefix part $\tau^{\text{pre}}$ is the projection of a finite path $p^{\text{pre}}$ that lives in $\ccalQ_P$ onto $\ccalQ_{\text{PTS}}$. The path $p^{\text{pre}}$ starts from an initial state $q_P^0\in\ccalQ_P^0$ and ends at a final state $q_P^F\in\ccalQ_P^F$, i.e., it has the following structure $p^{\text{pre}}=(q_{\text{PTS}}^0,q_B^0)(q_{\text{PTS}}^1,q_B^1)\dots (q_{\text{PTS}}^K,q_B^K)$ with $(q_{\text{PTS}}^K,q_B^K)\in\ccalQ_P^F$. The suffix part $\tau^{\text{suf}}$ is the projection of a finite path $p^{\text{suf}}$ that lives in $\ccalQ_P$ onto $\ccalQ_{\text{PTS}}$. The path $p^{\text{suf}}$ is a cycle around the final state $(q_{\text{PTS}}^K,q_B^K)$, i.e., it has the following structure $p^{\text{suf}}=(q_{\text{PTS}}^{K},q_B^K)(q_{\text{PTS}}^{K+1},q_B^{K+1})\dots (q_{\text{PTS}}^{K+S},q_B^{K+S})(q_{\text{PTS}}^{K+S+1},q_B^{K+S+1})$, where $(q_{\text{PTS}}^{K+S+1},q_B^{K+S+1})=(q_{\text{PTS}}^{K},q_B^{K})$. Then our goal is to compute a plan $\tau=\tau^{\text{pre}}[\tau^{\text{suf}}]^{\omega}=\Pi|_{\text{PTS}}p^{\text{pre}}[\Pi|_{\text{PTS}}p^{\text{pre}}]^{\omega}$, where $\Pi|_{\text{PTS}}$ stands for the projection on the state-space $\ccalQ_{\text{PTS}}$, so that the following objective function is minimized
\begin{align}\label{eq:cost2}
J(\tau)=\beta\hat{J}(\tau^{\text{pre}})+(1-\beta)\hat{J}(\tau^{\text{suf}}),
\end{align}
where $\hat{J}(\tau^{\text{pre}})$ and $\hat{J}(\tau^{\text{suf}})$ stand for the cost of the prefix and suffix part, respectively and $\beta\in[0,1]$ is a user-specified parameter. Specifically, in this paper we address the following problem. 

\begin{problem}\label{pr:problem}
Given a global LTL specification $\phi$, and transition systems $\text{wTS}_i$, for all robots $i$, determine a discrete team plan $\tau$ that satisfies $\phi$, i.e., $\tau\models\phi$, and minimizes the cost function \eqref{eq:cost2}.
\end{problem}

\subsection{A Solution to Problem \ref{pr:problem}}\label{sec:prelim}
Problem \ref{pr:problem} is typically solved by applying graph-search methods to the PBA. Specifically, to generate a motion plan $\tau$ that satisfies $\phi$, the PBA is viewed as a weighted directed graph $\mathcal{G}_P=\{\mathcal{V}_P, \mathcal{E}_P, w_P\}$, where the set of nodes $\mathcal{V}_P$ is indexed by the set of states $\mathcal{Q}_P$, the set of edges $\mathcal{E}_P$ is determined by the transition relation $\longrightarrow_P$, and the weights assigned to each edge are determined by the function $w_P$. Then, to find the optimal plan $\tau\models\phi$, shortest paths towards final states and shortest cycles around them are computed. More details about this approach can be found in \cite{smith2011optimal,ulusoy2013optimality,ulusoy2014optimal,kantaros15asilomar} and the references therein.

\section{Sampling-based Optimal Control Synthesis}\label{sec:solution}

In this section, we build upon our previous work \cite{kantaros2017Csampling} and propose a biased sampling-based optimal control synthesis algorithm that can synthesize optimal motion plans $\tau$ in prefix-suffix structure, i.e., $\tau=\tau^{\text{pre}}[\tau^{\text{suf}}]^{\omega}$, that satisfy a given global LTL specification $\phi$ from PBA with arbitrarily large state-space. The procedure is based on the incremental construction of a directed tree that approximately represents the state-space $\mathcal{Q}_P$ and the transition relation $\rightarrow_P$ of the PBA defined in Definition \ref{defn:pba}. In what follows, we denote by $\mathcal{G}_T=\{\mathcal{V}_T,\mathcal{E}_T,\texttt{Cost}\}$ the tree that approximately represents the PBA $P$. Also, we denote by $q_P^r$ the root of $\ccalG_T$. The set of nodes $\mathcal{V}_T$ contains the states of $\mathcal{Q}_P$ that have already been sampled and added to the tree structure. The set of edges $\mathcal{E}_T$ captures transitions between nodes in $\mathcal{V}_T$, i.e., $(q_P,q_P')\in\mathcal{E}_T$, if there is a transition from state $q_P\in\mathcal{V}_T$ to state $q_P'\in\mathcal{V}_T$. The function $\texttt{Cost}:\ccalV_T:\rightarrow\mathbb{R}_+$ assigns the cost of reaching node $q_P\in\mathcal{V}_T$ from the root $q_P^r$ of the tree. In other words, $\texttt{Cost}(q_P)=\hat{J}(\tau_T)$, where $q_P\in\ccalV_T$ and $\tau_T$ is the path in the tree $\ccalG_T$ that connects the root to $q_P$.

The construction of the prefix and the suffix part is described in Algorithm \ref{alg:plans}. In lines \ref{alg1:line1}-\ref{alg1:GB}, first the LTL formula is translated to an NBA $B$ and then $B$ is pruned by removing transitions that can never happen. Then, in lines \ref{alg1:line2}-\ref{alg1:line6}, the prefix parts $\tau^{\text{pre},a}$ are constructed, followed by the construction of their respective suffix parts $\tau^{\text{suf},a}$ in lines \ref{alg1:line7}-\ref{alg1:line15}. Finally, using the constructed prefix and suffix parts, the optimal plan $\tau=\tau^{\text{pre},a^{*}}[\tau^{\text{suf},a^{*}}]^{\omega}\models\phi$ is synthesized in lines \ref{alg1:line15a}-\ref{alg1:line16}.

\begin{algorithm}[t]
\caption{$\text{STyLuS}^{*}$: large-Scale optimal Temporal Logic Synthesis}
\LinesNumbered
\label{alg:plans}
\KwIn {LTL formula $\phi$, $\{\text{wTS}_i\}_{i=1}^N$, $q_{\text{PTS}}^0\in\ccalQ_{\text{PTS}}$, maximum numbers of iterations $n_{\text{max}}^{\text{pre}}$, $n_{\text{max}}^{\text{suf}}$}
\KwOut {Optimal plans $\tau\models\phi$}
Convert $\phi$ to an NBA $B=\left(\ccalQ_B,\ccalQ_B^0,\rightarrow_B,\ccalQ_B^F\right)$\;\label{alg1:line1} 
$[\{\Sigma_{q_B,q_B'}^{\text{feas}}\}_{\forall q_B, q_B'\in\ccalQ_B}]=\texttt{FeasibleWords}(\{\Sigma_{q_B,q_B'}\}_{\forall q_B, q_B'\in\ccalQ_B})$\;\label{alg1:feas} 
Construct graph $\ccalG_B$ and $d(q_B,q_B')$\;\label{alg1:GB} 
Define goal set: $\mathcal{X}_{\text{goal}}^{\text{pre}}$\;\label{alg1:line2} 
\For{$b_0=1:|\ccalQ_B^0|$}{\label{alg1:line2a} 
Initial NBA state: $q_B^0=\ccalQ_B^0(b_0)$\;\label{alg1:initNBA}
Root of the tree: $q_P^r=(q_\text{PTS}^0,q_B^0)$\;\label{alg1:line3}
\footnotesize{
$\left[\ccalG_T,\ccalP\right]=\texttt{ConstructTree}(\mathcal{X}_{\text{goal}}^{\text{pre}},\{\text{wTS}_i\}_{i=1}^N,B,q_P^r,n_{\text{max}}^{\text{pre}})$}\;\label{alg1:line4}
\normalsize
\For {$a=1:|\ccalP|$}{\label{alg1:line5} 
$\tau^{\text{pre},a}=\texttt{FindPath}(\mathcal{G}_T,q_P^r,\mathcal{P}(a))$\;}\label{alg1:line6} 
%
\For {$a=1:|\ccalP|$ }{\label{alg1:line7}
Root of the tree: $q_P^r=\ccalP(a)$\;\label{alg1:line8} 
Define goal set: $\mathcal{X}_{\text{goal}}^{\text{suf}}(q_P^r)$\;\label{alg1:line9}
\If{$(q_P^r\in\ccalX_{\text{goal}}^{\text{suf}}) ~\wedge~(w_P(q_P^r,q_P^r)=0)$}{\label{alg1:line9a}
$\ccalG_T=(\{q_P^r\},\{q_P^r,q_P^r\},0)$\;\label{alg1:line9b}
$\ccalS_a=\{q_P^r\}$\;}\label{alg1:line9c}
\Else{\label{alg1:line9d} 
\footnotesize{
$\left[\ccalG_T,\ccalS_a\right]=\texttt{ConstructTree}(\mathcal{X}_{\text{goal}}^{\text{suf}},\{\text{wTS}_i\}_{i=1}^N,B,q_P^r,n_{\text{max}}^{\text{suf}})$}\;\label{alg1:line10}}
}
Compute $\tau^{\text{suf},a}$ (see \cite{kantaros2017Csampling})\;\label{alg1:suf}
$a_{q_B^0}=\argmin_a(\hat{J}(\tau^{\text{pre},a})+\hat{J}(\tau^{\text{suf},a}))$\;\label{alg1:line15}} 
$a^*=\argmin_{a_{q_B^0}}(\hat{J}(\tau^{\text{pre}}_{q_B^0})+\hat{J}(\tau^{\text{suf}}_{{q_B^0}}))$\;\label{alg1:line15a} 
Optimal Plan: $\tau=\tau^{\text{pre},a^{*}}[\tau^{\text{suf},a^{*}}]^{\omega}$\;\label{alg1:line16} 
\end{algorithm}

\subsection{Feasible Symbols}

In this section, given the NBA $B$ that corresponds to the assigned LTL formula $\phi$, we define a function $d:\ccalQ_B\times\ccalQ_B\rightarrow \mathbb{N}$ that returns the minimum number of \textit{feasible} NBA transitions that are required to reach a state $q_B'\in\ccalQ_B$ starting from a state $q_B\in\ccalQ_B$ [lines \ref{alg1:line1}- \ref{alg1:GB}, Alg. \ref{alg:plans}]. This function will be used in the construction of the prefix and suffix parts to bias the sampling process. A \textit{feasible} NBA transition is defined as follows; see also Figure \ref{fig:nba}.

\begin{figure}[t]
  \centering
  \includegraphics[width=1\linewidth]{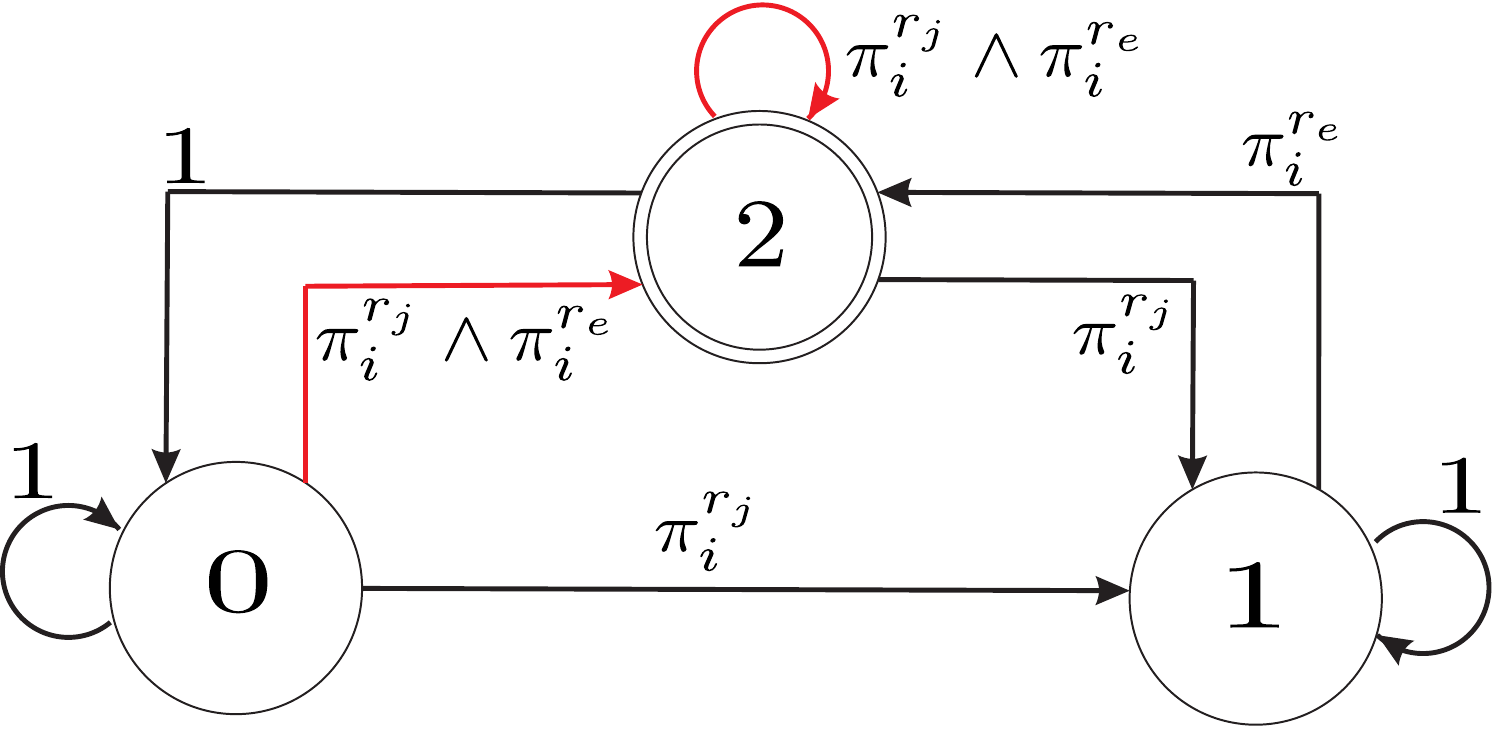}
   \caption{Graphical depiction of infeasible transitions of the NBA that corresponds to the LTL formula $\phi=\square\Diamond(\pi_i^{r_j})\wedge \square\Diamond(\pi_i^{r_e})$. The states `0' and `2' correspond to the initial and final state, respectively. The transition from the initial to the final state and the self-loop around the final state are infeasible transitions, since they are activated only by the symbol $\bar{\sigma}=\pi_i^{r_j}\pi_i^{r_e}$ which can never be generated by $\text{wTS}_i$, assuming disjoint regions $r_j$ and $r_e$.}
  \label{fig:nba}
\end{figure}


\begin{definition}[Feasible NBA transitions]
A transition $(q_B,\bar{\sigma},q_B')\in\rightarrow_B$ is feasible if the finite symbol $\bar{\sigma}\in\Sigma=2^{\mathcal{AP}}$ is a \textit{feasible} symbol, i.e., if $\bar{\sigma}$ can be generated by the PTS defined in Definition \ref{def:pts}.
\end{definition}

To characterize the symbols $\bar{\sigma}\in\Sigma$ that are feasible, we first need to define the symbols $\bar{\sigma}_i\in\Sigma_i=2^{\mathcal{AP}_i}$ that are feasible, i.e, the symbols that can be generated by $\text{wTS}_i$ defined in Definition \ref{defn:wTS}. 
\begin{definition}[Feasible symbols $\bar{\sigma}_i\in\Sigma_i$]\label{defn:infTSi}
A symbol $\bar{\sigma}_i\in\Sigma_i$ is \textit{feasible} if and only if $\bar{\sigma}_i\not\models b_i^{\text{inf}}$, where $b_i^{\text{inf}}$ is a Boolean formula defined as 
\begin{equation}\label{bInfTSi}
b_i^{\text{inf}}=\vee_{\forall r_j}( \vee_{\forall r_e} (\pi_i^{r_j}\wedge\pi_i^{r_e})).
\end{equation}
\end{definition}
Note that the Boolean formula $b_i^{\text{inf}}$ is satisfied by any finite symbol $\bar{\sigma}_i\in\Sigma_i$ that requires robot $i$ to be present in two or more disjoint regions, simultaneously. For instance, the symbol $\bar{\sigma}_i=\pi_i^{r_j}\pi_i^{r_e}$ satisfies $b_i^{\text{inf}}$.\footnote{Note that if we consider regions $r_j$ that are not necessarily disjoint, then the Boolean formula $b_i^{\text{inf}}$ is defined as $b_i^{\text{inf}}=\vee_{\forall r_j}( \vee_{\forall r_e \text{s.t.} r_j \cap r_e=\emptyset} (\pi_i^{r_j}\wedge\pi_i^{r_e}))$. Note also that definition of infeasible symbols depends on the problem at hand, i.e, the definition of atomic propositions included in the sets $\mathcal{AP}_i$.} Next, we define the feasible symbols $\bar{\sigma}\in\Sigma$.

\begin{definition}[Feasible symbols  $\bar{\sigma}\in\Sigma$]\label{defn:infPTS}
A symbol $\bar{\sigma}\in\Sigma$ is \textit{feasible} if and only if $\bar{\sigma}_i\not\models b_i^{\text{inf}}$, for all robots $i$,
where $\bar{\sigma}_i=\Pi|_{\Sigma_i}\bar{\sigma}$, $b_i^{\text{inf}}$ is defined in \eqref{bInfTSi}, and $\Pi|_{\Sigma_i}\bar{\sigma}$ stands for the projection of the symbol $\bar{\sigma}$ onto $\Sigma_i=2^{\mathcal{AP}_i}$.\footnote{For instance, $\Pi|_{\Sigma_i}(\pi_i^{r_e}\pi_m^{r_h})=\pi_i^{r_e}$.}
\end{definition}

To define the proposed function $d:\ccalQ_B\times\ccalQ_B\rightarrow \mathbb{N}$, we first construct sets $\Sigma_{q_B,q_B'}^\text{feas}\subseteq\Sigma$ that collect feasible finite symbols $\bar{\sigma}$ that enable a transition from a state $q_B\in\ccalQ_B$ to $q_B'\in\ccalQ_B$ according to $\rightarrow_B$  [line \ref{alg1:feas}, Alg. \ref{alg:plans}]. 
To construct these sets, sets $\Sigma_{q_B,q_B'}\subseteq\Sigma$ that collect all finite (feasible or infeasible) symbols $\bar{\sigma}\in\Sigma$ that enable a transition from $q_B\in\ccalQ_B$ to $q_B'\in\ccalQ_B$, for all $q_B,q_B'\in\ccalQ_B$, are required.\footnote{Note that the sets $\Sigma_{q_B,q_B'}$ can be computed during the translation of an LTL formula $\phi$ to an NBA; see e.g., the software package \cite{ltl2nbaBelta} that relies on \cite{gastin2001fast} for the construction of the NBA.} Then, the sets $\Sigma_{q_B,q_B'}^\text{feas}\subseteq \Sigma_{q_B,q_B'}$ can be constructed by removing from $\Sigma_{q_B,q_B'}$  all symbols $\bar{\sigma}$ that are not feasible, for all $q_B,q_B'\in\ccalQ_{B}$.

Next, we view the NBA as a directed graph $\ccalG_B=\{\ccalV_B,\ccalE_B\}$, where the set of nodes $\ccalV_B$ is indexed by the states $q_B\in\ccalQ_B$ and the set of edges $\ccalE_B\subseteq\ccalV_B\times\ccalV_B$ collects the edges from nodes/states $q_B$ to $q_B'$ denoted by $(q_B,q_B')$, where $(q_B,q_B')$ exists if $\Sigma_{q_B,q_B'}^{\text{feas}}\neq \emptyset$ [line \ref{alg1:GB}, Alg. \ref{alg:plans}]. 
Assigning weights equal to one to all edges in the set $\ccalE_B$, we define the function $d:\ccalQ_B\times\ccalQ_B\rightarrow \mathbb{N}$ as

\begin{equation}
d(q_B,q_B')=\left\{
                \begin{array}{ll}
                  |SP_{q_B,q_B'}|, \mbox{if $SP_{q_B,q_B'}$ exists,}\\
                  \infty, ~~~~~~~~~\mbox{otherwise}, 
                \end{array}
              \right.
\end{equation}
\normalsize
where $SP_{q_B,q_B'}$ denotes the shortest path in $\ccalG_B$ from $q_B\in\ccalV_B$ to $q_B'\in\ccalV_B$ and $|SP_{q_B,q_B'}|$ stands for its cost., i.e., the number of transitions/edges in $SP_{q_B,q_B'}$. 


\subsection{Construction of Optimal Prefix Parts}\label{sec:prefix}

In this section we describe the construction of the tree $\mathcal{G}_T=\{\mathcal{V}_T,\mathcal{E}_T,\texttt{Cost}\}$ that will be used for the synthesis of the prefix part [lines \ref{alg1:line2}-\ref{alg1:line6}, Alg. \ref{alg:plans}]. Since the prefix part connects an initial state $q_P^0=(q_\text{PTS}^0,q_B^0)\in\ccalQ_P^0$ to an \textit{accepting} state $q_P=(q_{\text{PTS}},~q_B)\in\ccalQ_P^F$, with $q_B\in\mathcal{Q}_B^{F}$, we can define the goal region for the tree $\ccalG_T$, as  [line \ref{alg1:line2}, Alg. \ref{alg:plans}]:
\begin{equation}\label{eq:goalPre}
\mathcal{X}_{\text{goal}}^{\text{pre}}=\{q_P=(q_{\text{PTS}},~q_B)\in\ccalQ_P~|~q_B\in\mathcal{Q}_B^{F}\}.
\end{equation}
The root $q_P^r$ of the tree is an initial state $q_P^0=(q_\text{PTS}^0,q_B^0)$ of the PBA and the following process is repeated for each initial state $q_B^0\in\ccalQ_B^0$, in parallel [line \ref{alg1:line2a}-\ref{alg1:line3}, Alg. \ref{alg:plans}]. The construction of the tree is described in Algorithm \ref{alg:tree} [line \ref{alg1:line4}, Alg. \ref{alg:plans}]. In line \ref{alg1:initNBA} of Algorithm \ref{alg:plans}, $\ccalQ_B^0(b_0)$ stands for the $b_0$-th initial state assuming an arbitrary enumeration of the elements of the set $\ccalQ_B^0$. The set $\mathcal{V}_T$ initially contains only the root $q_P^r$, i.e., an initial state of the PBA [line \ref{tree:line1} , Alg. \ref{alg:tree}] and, therefore, the set of edges is initialized as $\mathcal{E}_T=\emptyset$ [line \ref{tree:line2}, Alg. \ref{alg:tree}]. By convention, we assume that the cost of $q_P^r$ is zero [line \ref{tree:line3}, Alg. \ref{alg:tree}]. 
Given the root $q_P^r$ we select a \textit{feasible} final state $q_B^{F}\in\ccalQ_B^F$, such that (i) $d(q_B^0,q_B^{F})\neq \infty$ and (ii) $d(q_B^{F},q_B^{F})\neq\infty$. Among all final states that satisfy both (i) and (ii), we select one randomly denoted by $q_B^{F,\text{feas}}$ [line \ref{tree:feasF}, Alg. \ref{alg:tree}]. 
If there does not exist such a state $q_B^{F,\text{feas}}$, then this means that there is no prefix-suffix plan associated with the initial state $q_B^0$. In this case, the construction of the tree stops without having detected any final states around which a loop exists [lines \ref{tree:infeasF}-\ref{tree:exit}, Alg. \ref{alg:tree}]. The final state $q_B^{F,\text{feas}}$ will be used in the following subsection in order to bias the exploration of the PBA towards this state. We also define the set $\ccalD_{\text{min}}$ that collects the nodes $q_P=(q_{\text{PTS}},q_B)\in\ccalV_T$ that have the minimum distance $d(q_B,q_B^{F,\text{feas}})$ among all nodes in $\ccalV_T$, i.e., 
\begin{equation}
\ccalD_{\text{min}}=\{q_P=(q_{\text{PTS}},q_B)\in\ccalV_T~|~d(q_B,q_B^{F,\text{feas}})=d_{\text{min}} \},
\end{equation}
 where $d_{\text{min}}=\min \cup\{d(q_B,q_B^{F,\text{feas}})\}_{\forall q_B\in\Pi|_B\ccalV_T}$ and $\Pi|_B\ccalV_T\subseteq\ccalQ_B$ stands for the projection of all states $q_P\in\ccalV_T\subseteq\ccalQ_P$ onto $\ccalQ_B$. The set $\ccalD_{\text{min}}$ initially collects only the root [line \ref{tree:Dmin}, Alg. \ref{alg:tree}].

\begin{algorithm}[t]
\caption{\texttt{Function} $[\mathcal{G}_T,~\ccalZ]=\texttt{ConstructTree}(\ccalX_{\text{goal}},~\{\text{wTS}_i\}_{i=1}^N,~B,~q_P^r,~n_{\text{max}})$}
\LinesNumbered
\label{alg:tree}
$\mathcal{V}_T=\{q_P^r\}$\;\label{tree:line1} 
$\mathcal{E}_T=\emptyset$\; \label{tree:line2} 
$\texttt{Cost}(q_P^r)=0$\;\label{tree:line3}
\If{\texttt{prefix}}{
Select a feasible NBA final state $q_B^{F,\text{feas}}\in\ccalQ_B^F$.\;\label{tree:feasF}
\If{ $q_B^{F,\text{feas}}$ does not exist}{ \label{tree:infeasF}
Exit the function and set $\ccalZ=\emptyset$\;}}\label{tree:exit}
$\ccalD_{\text{min}}=\{q_P^r\}$\;\label{tree:Dmin} 
\For {$n=1:n_{\text{max}}$}{\label{tree:line4} 
\If{\texttt{prefix}}{
$q^{\text{new}}_{\text{PTS}}=\texttt{Sample}(\ccalV_T,\text{wTS}_1,\dots,\text{wTS}_N,q_B^{F,\text{feas}})$\;}\label{tree:samplePre} 
\If{\texttt{suffix}}{
$q^{\text{new}}_{\text{PTS}}=\texttt{Sample}(\ccalV_T,\text{wTS}_1,\dots,\text{wTS}_N,q_B^r)$\;}\label{tree:sampleSuf} 
	\For {$b=1:|\mathcal{Q}_B|$}{\label{tree:line6} 
			 $q_B^{\text{new}}=\mathcal{Q}_B(b)$\;\label{tree:line7} 
			 $q_{P}^{\text{new}}=(q^{\text{new}}_{\text{PTS}},q_B^{\text{new}})$\;\label{tree:line8} 
				\If {$q_{P}^{\text{new}}\notin\mathcal{V}_T$}{\label{tree:line9} 
						$[\mathcal{V}_T,~\mathcal{E}_T,\texttt{Cost}]=\texttt{Extend}(q_{P}^{\text{new}},\rightarrow_P)$\;\label{tree:line10}
						Update $\ccalD_{\text{min}}$\;}\label{tree:updDmin}
				\If{$q_{P}^{\text{new}}\in\mathcal{V}_T$}{\label{tree:line11} 				
						$[\mathcal{E}_T,\texttt{Cost}]=\texttt{Rewire}(q_P^{\text{new}},\mathcal{V}_T,\mathcal{E}_T,\texttt{Cost})$\;} }}\label{tree:line12} 
$\ccalZ=\ccalV_T\cap\ccalX_{\text{goal}}$\;\label{tree:line15} 
\end{algorithm}

\subsubsection{Sampling a state $q_P^{\text{new}}\in\mathcal{Q}_P$}\label{sec:sample}
The first step in the construction of the graph $\mathcal{G}_T$ is to sample a state $q_P^{\text{new}}$ from the state-space of the PBA. This is achieved by a sampling function $\texttt{Sample}$; see Algorithm \ref{alg:sample1}. Specifically, we first create a state $q_{\text{PTS}}^{\text{rand}}=\Pi|_{\text{PTS}}q_P^{\text{rand}}$, where $q_P^{\text{rand}}$ is sampled from a given discrete distribution $f_{\text{rand}}(q_P|\ccalV_T):\ccalV_T\rightarrow[0,1]$ and $\Pi|_{\text{PTS}}q_P^{\text{rand}}$ stands for the projection of  $q_P^{\text{rand}}$ onto the state-space of the PTS [line \ref{s:line2}, Alg. \ref{alg:sample1}]. The probability density function $f_{\text{rand}}(q_P|\ccalV_T)$ defines the probability of selecting the state $q_P\in\ccalV_T$ as the state $q_P^{\text{rand}}$ at iteration $n$ of Algorithm \ref{alg:tree} given the set $\ccalV_T$. The distribution $f_{\text{rand}}$ is defined as follows:

\begin{equation}\label{eq:frand}
f_{\text{rand}}(q_P|\ccalV_T,\ccalD_{\text{min}})=\left\{
                \begin{array}{ll}
                  p_{\text{rand}}\frac{1}{|\ccalD_{\text{min}}|}, ~~~~~~~~~~~~\mbox{if}~q_P\in\ccalD_{\text{min}}\\
                  (1-p_{\text{rand}})\frac{1}{|\ccalV_T\setminus\ccalD_{\text{min}}|},~\mbox{otherwise},
                \end{array}
              \right.
\end{equation}

\normalsize
\noindent where $p_{\text{rand}}\in(0.5,1)$ stands for the probability of selecting \textit{any} node $q_P\in\ccalD_{\text{min}}$ to be $q_P^{\text{rand}}$. Note that $p_{\text{rand}}$ can change with iterations $n$ but it should always satisfy $p_{\text{rand}}\in(0.5,1)$ so that states $q_P=(q_{\text{PTS}},q_B)\in\ccalD_\text{min}\subseteq\ccalV_T$ are selected more often to be $q_P^{\text{rand}}$ [line \ref{s:line2}, Alg. \ref{alg:sample1}]. 
\begin{rem}[Density function $f_{\text{rand}}$]\label{rem:frand}
Observe that the discrete density function $f_{\text{rand}}$ in \eqref{eq:frand} is defined in a uniform-like way, since all states in the set $\ccalD_{\text{min}}$ are selected with probability  $\frac{p_{\text{rand}}}{|\ccalD_{\text{min}}|}$ while the states in $\ccalV_T\setminus\ccalD_{\text{min}}$ are selected with probability $\frac{1-p_{\text{rand}}}{|\ccalV_T\setminus\ccalD_{\text{min}}|}$, where $p_{\text{rand}}\in(0.5,1)$. However, alternative definitions for $f_{\text{rand}}$ are possible as long as $f_{\text{rand}}$ (a) satisfies Assumptions 4.1(i)-(iii) made in \cite{kantaros2017Csampling}; and (b) is biased to generate states that belong to $\ccalD_{\text{min}}$ more often than states that belong to $\ccalV_T\setminus\ccalD_{\text{min}}$. Assumptions 4.1(i) and 4.1(iii) in \cite{kantaros2017Csampling} are required to guarantee that the proposed algorithm is probabilistically complete while Assumptions 4.1(i)-(iii) are required for the asymptotic optimality of the proposed method; see Section \ref{sec:corr}. The fact that the density function \eqref{eq:frand} satisfies Assumption 4.1 in \cite{kantaros2017Csampling} is shown in Section \ref{sec:corr}. Finally, bias in the sampling increases scalability of the proposed control synthesis method; see Section \ref{sec:sim}. Note that, as it will be discussed later, $f_{\text{rand}}$ can also change with iterations $n$ of Algorithm \ref{alg:tree}.
\end{rem}

Given  $q_{\text{PTS}}^{\text{rand}}$, in our previous work \cite{kantaros2017Csampling}, we sample a state $q_{\text{PTS}}^{\text{new}}$ from a discrete distribution $f_{\text{new}}(q_{\text{PTS}}|q_{\text{PTS}}^{\text{rand}})$ so that $q_{\text{PTS}}^{\text{new}}$ is reachable from $q_{\text{PTS}}^{\text{rand}}$. Here, we sample a state $q_{\text{PTS}}^{\text{new}}$ so that it is both reachable from $q_{\text{PTS}}^{\text{rand}}$ and also it can lead to a final state $q_P^F=(q_{\text{PTS}},q_B^{F,\text{feas}})\in\ccalQ_P^F$ by following the shortest path in $\ccalG_B$ that connects $q_B^0$ to $q_B^{F,\text{feas}}$. The proposed biased sampling process is illustrated in Figure \ref{fig:sampling}.

\begin{figure}[t]
  \centering
  \includegraphics[width=1\linewidth]{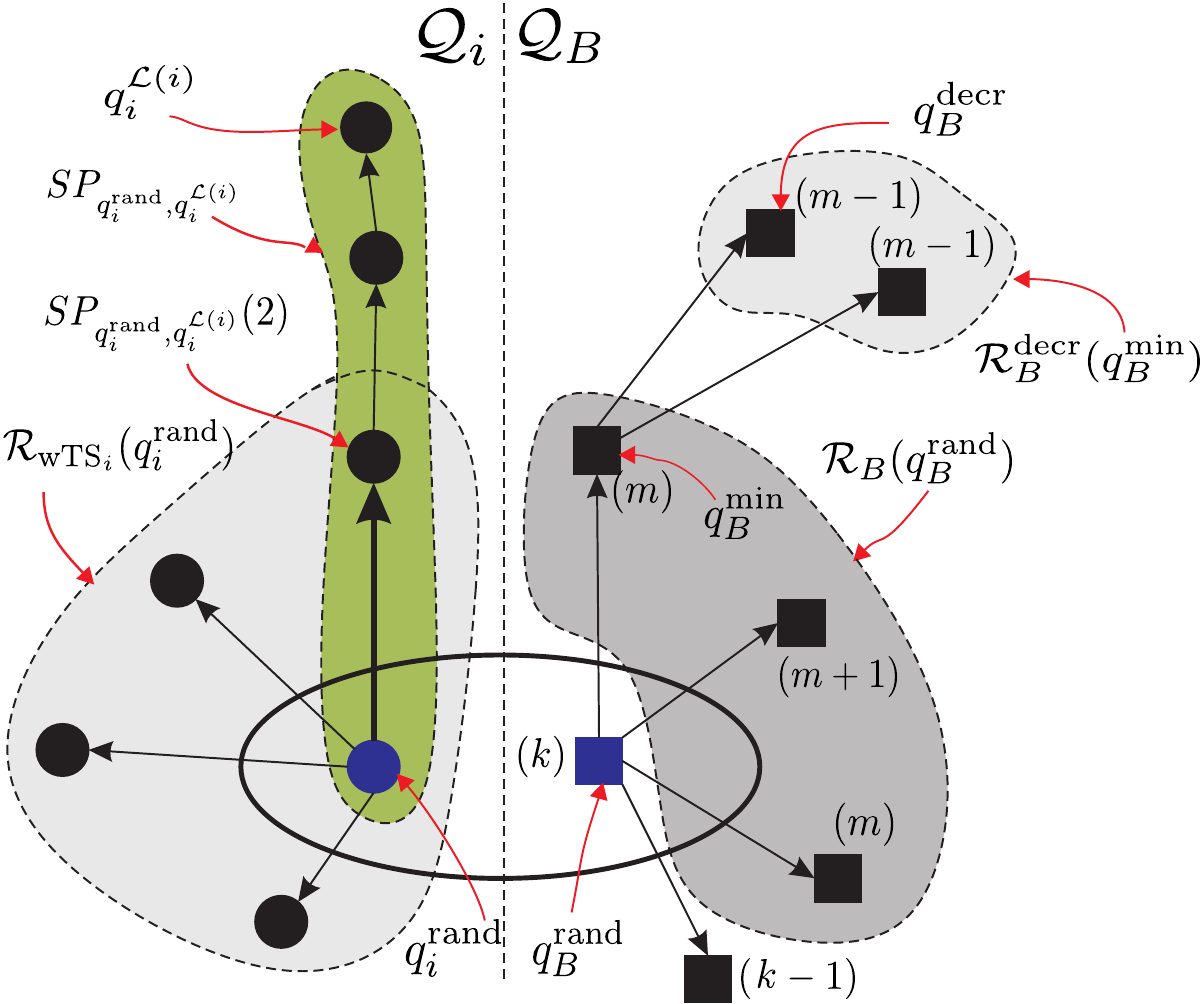}
\caption{{Graphical depiction of the proposed sampling function when $N=1$. The thick black arrow ends in the state that will be selected with probability $p_{\text{new}}$. Note that for $\bar{\sigma}=L_{i}(q_{i}^{\ccalL(i)})$, it holds that $(q_B^{\text{min}},\bar{\sigma}, q_B^{\text{decr}})\in \rightarrow_B$. Next to each state $q_B\in\ccalQ_B$, we also note inside parentheses the value of $d(q_B,q_B^{F,\text{feas}})$.}}
  \label{fig:sampling}
\end{figure}

\paragraph{Selection of $q_B^{\text{min}}\in\ccalQ_B$} First, given  $q_P^{\text{rand}}=(q_{\text{PTS}}^{\text{rand}},q_B^{\text{rand}})$, we first construct the reachable set $\ccalR_B(q_B^{\text{rand}})$ that collects all states $q_B\in\ccalQ_B$ that can be reached in one hop in $B$ from $q_B^{\text{rand}}$ given the observation $L_{\text{PTS}}(q_{\text{PTS}}^{\text{rand}})$ (see also Figure \ref{fig:sampling}) [line \ref{s:RBrand}, Alg. \ref{alg:sample1}], i.e., 
\begin{equation}\label{eq:RBrand}
\ccalR_B(q_B^{\text{rand}})=\{q_B\in\ccalQ_B~|~(q_B^{\text{rand}},L_{\text{PTS}}(q_{\text{PTS}}^{\text{rand}}), q_B)  \in\rightarrow_B\}.
\end{equation}
Given $\ccalR_B(q_B^{\text{rand}})$ we construct the set $\ccalR_B^{\text{min}}(q_B^{\text{rand}})$ that collects the states $q_B\in\ccalR_B(q_B^{\text{rand}})$ that have the minimum distance from $q_B^{F,\text{feas}}$ among all other nodes $q_B\in\ccalR_B(q_B^{\text{rand}})$, i.e., 
\begin{align}\label{eq:RBmin}
&\ccalR_B^{\text{min}}(q_B^{\text{rand}})=\{q_B\in\ccalR_B(q_B^{\text{rand}})~|\\&d(q_B,q_B^{F,\text{feas}})=\min_{q_B'\in\ccalR_B(q_B^{\text{rand}})} d(q_B',q_B^{F,\text{feas}}) \}\subseteq\ccalR_B(q_B^{\text{rand}}),\nonumber
\end{align}

In what follows, we denote by  $q_B^{\text{cand,min}}$ the states that belong to $\ccalR_B^{\text{min}}(q_B^{\text{rand}})$. For every state $q_B^{\text{cand,min}}\in\ccalR_B^{\text{min}}(q_B^{\text{rand}})$, we construct the set $\ccalR_B^{\text{decr}}(q_B^{\text{cand,min}})$ that collects all states $q_B\in\ccalQ_B$, for which (i) there exists a feasible symbol $\bar{\sigma}\in\Sigma_{q_B^{\text{cand,min}},q_B}^{\text{feas}}$ such that $q_B^{\text{cand,min}}\xrightarrow{\bar{\sigma}}q_B$, and (ii) $q_B$ is closer to $q_B^{F,\text{feas}}$ than $q_B^{\text{cand,min}}$ is, i.e., 
\begin{align}\label{eq:RBfeas}
&\ccalR_B^{\text{decr}}(q_B^{\text{cand,min}})=\{q_B\in\ccalQ_B~|~(\Sigma_{q_B^{\text{cand,min}},q_B}^{\text{feas}}\neq\emptyset)\\&\wedge
d(q_B,q_B^{F,\text{feas}})=d(q_B^{\text{cand,min}},q_B^{F,\text{feas}})-1 \}\subseteq\ccalQ_B.\nonumber
\end{align}
We collect all states $q_B^{\text{cand,min}}\in\ccalR_B^{\text{min}}(q_B^{\text{rand}})$ that satisfy $\ccalR_B^{\text{decr}}(q_B^{\text{cand,min}})\neq\emptyset$ in the set $\ccalM(q_B^{\text{rand}})$ [line \ref{s:M}, Alg.\ref{alg:sample1}] defined as
$$\ccalM(q_B^{\text{rand}})=\{q_B^{\text{cand,min}}~|~ \ccalR_B^{\text{decr}}(q_B^{\text{cand,min}})\neq\emptyset \}\subseteq \ccalR_B^{\text{min}}(q_B^{\text{rand}}).$$
Next, given the set $\ccalM(q_B^{\text{rand}})$, we sample a state $q_B^{\text{min}}\in\ccalM(q_B^{\text{rand}})$ from a discrete distribution $f_B^{\text{min}}(q_B|\ccalM(q_B^{\text{rand}})):\ccalM(q_B^{\text{rand}})\rightarrow[0,1]$ (see also Figure \ref{fig:sampling}) [line \ref{s:qBmin}, Alg. \ref{alg:sample1}] defined as $$f_B^{\text{min}}(q_B|\ccalM(q_B^{\text{rand}}))=\frac{1}{|\ccalM(q_B^{\text{rand}})|}.$$
%
Note that if $\ccalM(q_B^{\text{rand}})=\emptyset$, then no sample will be taken at this iteration [line \ref{s:empty2}, Alg. \ref{alg:sample1}].
%
%
\normalsize
\paragraph{Selection of $q_B^{\text{decr}}\in\ccalR_B^{\text{decr}}(q_B^{\text{min}})$} Given the state $q_B^{\text{min}}$ and its respective set $\ccalR_B^{\text{decr}}(q_B^{\text{min}})$, we sample a state $q_B^{\text{decr}}\in\ccalR_B^{\text{decr}}(q_B^{\text{min}})$ from a given discrete distribution $f_B^{\text{decr}}(q_B|\ccalR_B^{\text{decr}}(q_B^{\text{min}})):\ccalR_B^{\text{decr}}(q_B^{\text{min}})\rightarrow[0,1]$ [line \ref{s:qBfeas}, Alg. \ref{alg:sample1}]  defined as
$$f_B^{\text{decr}}(q_B|\ccalR_B^{\text{decr}}(q_B^{\text{min}}))=\frac{1}{|\ccalR_B^{\text{decr}}(q_B^{\text{min}})|},$$ 
for all $q_B\in \ccalR_B^{\text{decr}}(q_B^{\text{min}})$; see also Figure \ref{fig:sampling}.
%
Note that if $\ccalR_B^{\text{decr}}(q_B^{\text{min}})=\emptyset$, then no sample $q_{\text{PTS}}^{\text{new}}$ is generated [line \ref{s:empty1}, alg. \ref{alg:sample1}].

\normalsize{Given $q_B^{\text{min}}$ and $q_B^{\text{decr}}$, we select a symbol $\bar{\sigma}$ from $\Sigma_{q_B^{\text{min}},q_B^{\text{decr}}}\neq\emptyset$ [line \ref{s:sigmamin}, Alg. \ref{alg:sample1}].\footnote{To speed up the detection of final states, we always select the same symbol
$\bar{\sigma}\in\Sigma_{q_B^{\text{min}},q_B^{\text{decr}}}$ for a given pair of states $q_B^{\text{min}}$ and $q_B^{\text{decr}}$. Also, by construction of $q_B^{\text{decr}}$, it holds that $\Sigma_{q_B^{\text{min}},q_B^{\text{decr}}}\neq\emptyset$.} Given the symbol $\bar{\sigma}$, we construct the set $\ccalL$ so that the $i$-th element of $\ccalL$ captures the region where robot $i$ has to be so that the symbol $\bar{\sigma}$ can be generated. 
For instance, the set $\ccalL$ corresponding to the symbol $\bar{\sigma}=\pi_{i}^{r_j}\pi_z^{r_e}$ is constructed so that $\ccalL(i)=r_j$, $\ccalL(z)=r_e$ and $\ccalL(h)=\varnothing$, for all robots $h\neq i, z$. Then, we sample a state $q_{i}^{\text{new}}$, for all robots $i$, from a discrete distribution $f_{\text{new},i}^{\text{pre}}(q_{i}|q_{i}^{\text{rand}}):\ccalR_{\text{wTS}_i}(q_{i}^{\text{rand}})\rightarrow[0,1]$ [line \ref{s:fnewi}, Alg. \ref{alg:sample1}] defined as} 

\footnotesize{
\begin{equation}\label{eq:fnew}
f_{\text{new},i}^{\text{pre}}(q_{i}|q_{i}^{\text{rand}})=\left\{
                \begin{array}{ll}
                \frac{1}{|\ccalR_{\text{wTS}_i}(q_{i}^{\text{rand}})|},~\mbox{if}~\ccalL(i)=\varnothing,\\
                 p_{\text{new}},\mbox{if}~(L(i)\neq\varnothing)\wedge
                   (q_{i}=SP_{q_{i}^{\text{rand}},q_{i}^{\ccalL(i)}}(2))\\
                  (1-p_{\text{new}})\frac{1}{|\ccalR_{\text{wTS}_i}(q_{i}^{\text{rand}})|-1},\mbox{if}~(\ccalL(i)\neq\varnothing)\wedge\\ \indent(q_{i}\in \ccalR_{\text{wTS}_i}(q_{i}^{\text{rand}})\setminus SP_{q_{i}^{\text{rand}},q_{i}^{\ccalL(i)}}(2)),
                \end{array}
              \right.
\end{equation}}
\normalsize where $\ccalR_{\text{wTS}_i}(q_{i}^{\text{rand}})$ collects all states in $\ccalQ_i$ that can be reached in one hop from $q_{i}^{\text{rand}}=\Pi|_{\text{wTS}_i}q_{\text{PTS}}^{\text{rand}}$, i.e., 
\begin{equation}
\ccalR_{\text{wTS}_i}(q_{i}^{\text{rand}})=\{q_i\in\ccalQ_i~|~q_{i}^{\text{rand}}\rightarrow_i q_i\}, \nonumber
\end{equation}
where $\Pi|_{\text{wTS}_i}q_{\text{PTS}}^{\text{rand}}$ stands for the projection of $q_{\text{PTS}}^{\text{rand}}$ onto the state-space of $\text{wTS}_i$. Also, in \eqref{eq:fnew}, viewing $\text{wTS}_i$ as a graph, $SP_{q_{i}^{\text{rand}},q_{i}^{\ccalL(i)}}$ stands for the shortest path from the node/state $q_{i}^{\text{rand}}$ to the state $q_{i}^{\ccalL(i)}$, i.e., $SP_{q_{i}^{\text{rand}},q_{i}^{\ccalL(i)}}$ is a finite sequence of states in $\text{wTS}_i$ that start from  $q_{i}^{\text{rand}}$ and end at the state $q_{i}^{\ccalL(i)}$; see also Figure \ref{fig:sampling}. Also, $SP_{q_{i}^{\text{rand}},q_{i}^{\ccalL(i)}}(2)$ stands for the second state in this sequence. Moreover, in \eqref{eq:fnew}, $p_{\text{new}}$ stands for the probability of selecting the state $SP_{q_{i}^{\text{rand}},q_{i}^{\ccalL(i)}}(2)$ to be $q_{i}^{\text{new}}$ if $\ccalL(i)\neq\varnothing$. Note that $p_{\text{new}}$ can change with iterations $n$ but it should always satisfy $p_{\text{new}}\in(0.5,1)$ so that the state $SP_{q_{i}^{\text{rand}},q_{i}^{\ccalL(i)}}(2)$ is selected more often to be $q_{i}^{\text{new}}$, as it is closer to the state $q_i^{\ccalL(i)}$. Finally, given the states $q_i^{\text{new}}$, we construct $q_{\text{PTS}}^{\text{new}}=(q_{1}^{\text{new}},\dots,q_{N}^{\text{new}})\in\ccalQ_{\text{PTS}}$. Observe that by construction of $f_{\text{new},i}^{\text{pre}}$ the state $q_i^{\text{new}}$ always lies in $\ccalR_{\text{wTS}_i}(q_{i}^{\text{rand}})$. As a result, $q_{\text{PTS}}^{\text{new}}$ is reachable from $q_{\text{PTS}}^{\text{rand}}$.



\begin{algorithm}[t]
\caption{Function $q_{\text{PTS}}^{\text{new}}=\texttt{Sample}(\ccalV_T, \ccalD_{\text{min}},\{\text{wTS}_i\}_{i=1}^N,q_B^{\text{goal}})$}
\label{alg:sample1}
Pick a state $q_{P}^{\text{rand}}=(q_{\text{PTS}}^{\text{rand}},q_B^{\text{rand}})\in\ccalV_T$ from  $f_{\text{rand}}$\;\label{s:line2}
Compute $\ccalR_B(q_B^{\text{rand}})$\;\label{s:RBrand}
\If{$\ccalR_B(q_B^{\text{rand}})\neq\emptyset$}{
Compute $\ccalM(q_B^{\text{rand}})\subseteq \ccalR_B(q_B^{\text{rand}})$\;\label{s:M}
\If{$\ccalM(q_B^{\text{rand}})\neq\emptyset$}{
Sample $q_B^{\text{min}}\in\ccalM(q_B^{\text{rand}})$ from $f_B^{\text{min}}$\;\label{s:qBmin}
Select $q_B^{\text{decr}}\in\ccalR_B^{\text{decr}}(q_B^{\text{min}})$\;\label{s:qBfeas}
Pick $\bar{\sigma}\in\Sigma_{q_B^{\text{min}},q_B^{\text{decr}}}$\;\label{s:sigmamin}
}
\Else{$q_{\text{PTS}}^{\text{new}}=\varnothing$\;}}\label{s:empty1}
\Else{$q_{\text{PTS}}^{\text{new}}=\varnothing$\;}\label{s:empty2}
\If{$(\ccalR_B(q_B^{\text{rand}})\neq\emptyset)\wedge(\ccalM(q_B^{\text{rand}})\neq\emptyset)$}{
Pick a state $q_{i}^{\text{new}}$ from a given probability distribution $f_{\text{new},i}$, for all robots $i$\;\label{s:fnewi}
Construct $q_{\text{PTS}}^{\text{new}}=(q_1^{\text{new}},\dots,q_N^{\text{new}})$\;}\label{s:line5}
\end{algorithm}


In order to build incrementally a graph whose set of nodes approximates the state-space $\mathcal{Q}_P$ we need to append to $q_{\text{PTS}}^{\text{new}}$ a state from the state-space $\mathcal{Q}_B$ of the NBA $B$. Let $q_B^{\text{new}}=\mathcal{Q}_B(b)$ [line \ref{tree:line7}, Alg. \ref{alg:tree}] be the candidate B$\ddot{\text{u}}$chi state that will be attached to $q_{\text{PTS}}^{\text{new}}$, where $\mathcal{Q}_B(b)$ stands for the $b$-th state in the set $\ccalQ_B$ assuming an arbitrary enumeration of the elements of the set $\ccalQ_B$. The following procedure is repeated for all $q_B^{\text{new}}=\ccalQ_B(b)$ with $b\in\{1,\dots,|\mathcal{Q}_B|\}$. First, we construct the state $q_{P}^{\text{new}}=(q^{\text{new}}_{\text{PTS}},q_B^{\text{new}})\in\mathcal{Q}_P$ [line \ref{tree:line8}, Alg. \ref{alg:tree}] and then we check if this state can be added to the tree $\mathcal{G}_T$ [lines \ref{tree:line9}-\ref{tree:line10}, Alg. \ref{alg:tree}]. If the state $q_{P}^{\text{new}}$ does not already belong to the tree from a previous iteration of Algorithm \ref{alg:tree}, i.e, if $q_{P}^{\text{new}}\notin\mathcal{V}_T$ [line \ref{alg1:line9}, Alg. \ref{alg:tree}], we check which node in $\mathcal{V}_T$ (if there is any) can be the parent of $q_{P}^{\text{new}}$ in the tree $\mathcal{G}_T$. If there exist candidate parents for  $q_{P}^{\text{new}}$ then the tree is \textit{extended} towards  $q_{P}^{\text{new}}$ and the set $\ccalD_{\text{min}}$ is updated [line \ref{tree:updDmin}, Alg. \ref{alg:tree}]. If $q_P^{\text{new}}\in\ccalV_T$, then the \textit{rewiring} step follows [lines \ref{tree:line11}-\ref{tree:line12}, Alg. \ref{alg:tree}] that aims to reduce the cost of nodes $q_P\in\ccalV_T$. A detailed description of the `extend' and `rewire' steps can be found in \cite{kantaros2017Csampling}.

\begin{rem}[Density function $f_{\text{new},i}^{\text{pre}}$]\label{rem:fnew}
Observe that the discrete density function $f_{\text{new},i}^{\text{pre}}$ in \eqref{eq:fnew} is defined in a uniform-like way, similar to $f_{\text{rand}}$ in \eqref{eq:frand}. However, alternative definitions for $f_{\text{new},i}^{\text{pre}}$ are possible as long as $f_{\text{new},i}^{\text{pre}}(q_i|q_i^{\text{rand}})$ (a) satisfies Assumptions 4.2(i)-(iii) in \cite{kantaros2017Csampling}; and (b) is biased to generate the state $SP_{q_{i}^{\text{rand}},q_{i}^{\ccalL(i)}}(2)$ (if defined) more often than the other states that belong to $\ccalR_{\text{wTS}_i}(q_i^{\text{rand}})$. Assumptions 4.2(i) and 4.2(iii) in \cite{kantaros2017Csampling} are required to guarantee that the proposed algorithm is probabilistically complete while Assumptions 4.2(i)-(iii) are required to prove the asymptotic optimality of the proposed method; see Section \ref{sec:corr}. Finally, bias in the sampling increases scalability of the proposed control synthesis method; see Section \ref{sec:sim}. The fact that the density function \eqref{eq:fnew} satisfies Assumption 4.2 of \cite{kantaros2017Csampling} is shown in Section \ref{sec:corr}. Note that, as it will be discussed later, $f_{\text{new}}^{\text{pre}}$ can also change with iterations $n$ of Algorithm \ref{alg:tree}. The same remark also holds for the density functions $f_{\text{new},i}^{\text{suf}}$ that will be introduced in Section \ref{sec:suffix} for the construction of the suffix parts.
\end{rem}

\subsubsection{Construction of Paths}
The construction of the tree $\ccalG_T$ ends after $n_{\text{max}}^{\text{pre}}$ iterations, where $n_{\text{max}}^{\text{pre}}$ is user specified [line \ref{tree:line4}, Alg. \ref{alg:tree}]. Then, we construct the set $\ccalP=\ccalV_T\cap\ccalX_{\text{goal}}^{\text{pre}}$ [line \ref{tree:line15}, Alg. \ref{alg:tree}] that collects all the states $q_P\in\ccalV_T$ that belong to the goal region $\ccalX_{\text{goal}}^{\text{pre}}$. Given the tree $\ccalG_T$ and the set $\ccalP$ [line \ref{alg1:line4}, Alg. \ref{alg:plans}] that collects all states $q_P\in\ccalX_{\text{goal}}^{\text{pre}}\cap\ccalV_T$, we can compute the prefix plans [lines \ref{alg1:line5}-\ref{alg1:line6}, Alg. \ref{alg:plans}]. In particular, the path that connects the $a$-th state in the set $\mathcal{P}$, denoted by $\ccalP(a)$, to the root $q_P^r$ constitutes the $\alpha$-th prefix plan and is denoted by $\tau^{\text{pre},a}$ [line \ref{alg1:line6}, Algorithm \ref{alg:plans}]. Specifically, the prefix part $\tau^{\text{pre},a}$ is constructed by tracing the sequence of parent nodes starting from the node that represents the accepting state $\mathcal{P}(a)$ and ending at the root of the tree.

\begin{rem}[Bias in Prefix Parts]\label{rem:biasPre}
During the construction of the prefix parts, the sampling density functions $f_{\text{rand}}$ and $f_{\text{new},i}^{\text{pre}}$ are biased towards detecting states $q_P=(q_{\text{PTS}},q_B^{F,\text{feas}})\in\ccalQ_P^F$, where $q_{\text{PTS}}$ can be any state in $\ccalQ_{\text{PTS}}$ and $q_B^{F,\text{feas}}$ is given feasible final state of the NBA. Once such a state $q_P=(q_{\text{PTS}},q_B^{F,\text{feas}})$ is detected, we can switch the bias towards a different feasible final state. Also, once all such feasible final states $q_B^{F,\text{feas}}$ are detected or after a pre-determined number of iterations $n$, we can switch from biased density functions to unbiased (uniform) density functions $f_{\text{rand}}$ and $f_{\text{new},i}^{\text{pre}}$, that favor exploration of $\ccalQ_P$ towards all directions, by selecting $p_{\text{rand}}=|\ccalD_{\text{min}}|/|\ccalV_T|$ and $p_{\text{new}}=1/|\ccalR_{\text{wTS}_i}(q_i^{\text{rand}})|$, where recall that $q_i^{\text{rand}}=\Pi|_{\text{wTS}_i}(q_P^{\text{rand}})$.
\end{rem} 

\subsection{Construction of Optimal Suffix Parts}\label{sec:suffix}
Once the prefix plans $\tau^{\text{pre},a}$ for all $a\in\{1,\dots,|\ccalP|\}$ are constructed, the corresponding suffix plans $\tau^{\text{suf},a}$ are constructed [lines \ref{alg1:line7}-\ref{alg1:suf}, Alg. \ref{alg:plans}]. Specifically, every suffix part $\tau^{\text{suf},a}$ is a sequence of states in $\ccalQ_P$ that starts from the state $\mathcal{P}(a)$ and ends at the same state $\mathcal{P}(a)$.
To construct the suffix part $\tau_i^{\text{suf},a}$ we build a tree $\mathcal{G}_T=\{\mathcal{V}_T, \mathcal{E}_T, \texttt{Cost}\}$ that approximates the PBA $P$, in a similar way as in Section \ref{sec:prefix}, and implement a cycle-detection mechanism to identify cycles around the state $P(a)$. The only differences are that: (i) the root of the tree is now $q_P^{r}=\ccalP(a)$ [line \ref{alg1:line8}, Alg. \ref{alg:plans}] detected during the construction of the prefix plans, (ii) the goal region corresponding to the root $q_P^{r}=\ccalP(a)$, is defined as 
\begin{align}\label{eq:goalSuf}
\mathcal{X}_{\text{goal}}^{\text{suf}}(q_P^r)=&\{q_P=(q_{\text{PTS}},~q_B)\in\ccalQ_P~|\nonumber\\&(q_P,L(q_{\text{PTS}}),q_P^r)\in\rightarrow_P\},
\end{align}
(iii) we first check if $q_P^r\in\ccalX_{\text{goal}}^{\text{suf}}$ [line \ref{alg1:line9a}, Alg. \ref{alg:plans}], and (iv) a probability density function $f_{\text{new},i}^{\text{suf}}$ that is different from \eqref{eq:fnew} is employed. As for (iii), if $q_P^r\in\ccalX_{\text{goal}}^{\text{suf}}$, the construction of the tree is trivial, as it consists of only the root, and a loop around it with zero cost [line \ref{alg1:line9b}, Alg. \ref{alg:plans}]. If $q_P^r\notin\ccalX_{\text{goal}}^{\text{suf}}$, then the tree $\ccalG_T$ is constructed by Algorithm \ref{alg:tree} [line \ref{alg1:line10}, Alg. \ref{alg:plans}]. As for (iv), $f_{\text{new},i}^{\text{suf}}(q_{i}|q_{i}^{\text{rand}}):\ccalR_{\text{wTS}_i}(q_{i}^{\text{rand}})\rightarrow[0,1]$ [line \ref{s:fnewi}, Alg. \ref{alg:sample1}] that generates $q_{i}^{\text{new}}$ is defined as
\begin{align}\label{eq:fnewsuf1}
f_{\text{new},i}^{\text{suf}}(q_{i}|q_{i}^{\text{rand}})=f_{\text{new},i}^{\text{pre}}(q_{i}|q_{i}^{\text{rand}}),
\end{align}
if $q_B^r\notin\ccalR_B^{\text{decr}}$, where $q_B^r=\Pi|_{B}q_P^r\in\ccalQ_B^F$. If $q_B^r\in\ccalR_B^{\text{decr}}$, then this means that the NBA part of the root (final state), i.e., $q_B^r\in\ccalQ_B^F$, can be reached in one hop from $q_{B}^{\text{min}}$. Then in this case, to steer the tree towards the root $q_P^r=(q_{\text{PTS}}^r,q_B^r)$, we select the following probability density function $f_{\text{new},i}^{\text{suf}}(q_{i}|q_{i}^{\text{rand}})$ to generate the samples $q_{\text{PTS}}^{\text{new}}$.

\footnotesize{
\begin{align}\label{eq:fnewsuf2}
f_{\text{new},i}^{\text{suf}}(q_{i}|q_{i}^{\text{rand}})=\left\{
                \begin{array}{ll}
                 \frac{1-p_{\text{new}}}{|\ccalR_{\text{wTS}_i}(q_{i}^{\text{rand}})|},\mbox{if}~(\ccalL(i)=\varnothing)\wedge
                   (q_{i}\neq SP_{q_{i}^{\text{rand}},q_{i}^{r}}(2)),\\
                p_{\text{new}}, \mbox{if}~(\ccalL(i)=\varnothing) \wedge
                   (q_{i}= SP_{q_{i}^{\text{rand}},q_{i}^{r}}(2))\\
                 p_{\text{new}},
                  \mbox{if}~(L(i)\neq\varnothing)\wedge
                   (q_{i}=SP_{q_{i}^{\text{rand}},q_{i}^{\ccalL(i)}}(2))\\
                 \frac{1-p_{\text{new}}}{|\ccalR_{\text{wTS}_i}(q_{i}^{\text{rand}})|-1},
                    \mbox{if}~(\ccalL(i)\neq\varnothing)\wedge\\ \indent
                   (q_{i}\in \ccalR_{\text{wTS}_i}(q_{i}^{\text{rand}})\setminus SP_{q_{i}^{\text{rand}},q_{i}^{\ccalL(i)}}(2))
                \end{array}
              \right.
\end{align}}
\normalsize
where $q_i^r=\Pi|_{\text{wTS}_i}q_{\text{PTS}}^r$. Note that when \eqref{eq:fnewsuf1} is employed, the sampling process is biased towards any state $q_P=(q_{\text{PTS}},q_B^r)$ (and not $q_P=(q_{\text{PTS}},q_B^{F, \text{feas}})$ as in the prefix part); see also [line \ref{tree:sampleSuf}, Alg. \ref{alg:tree}]. On the other hand, when \eqref{eq:fnewsuf2} is employed the sampling process is biased towards the root $q_P^r=(q_{\text{PTS}}^r,q_B^r)$.
Once a tree rooted at $q_P^r=\ccalP(a)$ is constructed, a set $\ccalS_a\subseteq\ccalV_T$ is formed that collects all states $q_P\in\ccalV_T\cap\ccalX_{\text{goal}}^{\text{suf}}(q_P^r)$ [lines \ref{alg1:line9c}, \ref{alg1:line10}, Alg. \ref{alg:plans}]. 
Given the set $\ccalS_a$, we compute the best suffix plan $\tau^{\text{suf},a}$ associated with $q_P^r=\ccalP(a)\in\ccalQ_P^F$, as in \cite{kantaros2017Csampling}, [line \ref{alg1:suf}, Alg. \ref{alg:plans}]. This process is repeated for all $a\in\{1,\dots,|\ccalP|\}$, in parallel [line \ref{alg1:line7}, Alg. \ref{alg:plans}]. In this way, for each prefix plan $\tau^{\text{pre},a}$ we construct its corresponding suffix plan $\tau^{\text{suf},a}$, if it~exists.

\begin{rem}[Bias in Suffix Parts]
During the construction of the suffix parts, the sampling density functions $f_{\text{rand}}$ and $f_{\text{new},i}^{\text{suf}}$ are biased towards detecting the root $q_P^r$, so that a loop around $q_P^r$, i.e., a suffix part is detected. Once this happens, we can switch from biased density functions to unbiased (uniform) density functions $f_{\text{rand}}$ and $f_{\text{new},i}^{\text{suf}}$, that favor exploration of $\ccalQ_P$ towards all directions, by selecting $p_{\text{rand}}=|\ccalD_{\text{min}}|/|\ccalV_T|$ and $p_{\text{new}}=1/|\ccalR_{\text{wTS}_i}(q_i^{\text{rand}})|$, where recall that $q_i^{\text{rand}}=\Pi|_{\text{wTS}_i}(q_P^{\text{rand}})$.
\end{rem}

\subsection{Construction of Optimal Discrete Plans}\label{sec:plan}
By construction, any motion plan $\tau^a=\tau^{\text{pre},a}[\tau^{\text{suf},a}]^{\omega}$, with $\mathcal{S}_a\neq\emptyset$, and $a\in\{1,\dots,|\mathcal{P}|\}$ satisfies the LTL task $\phi$. The cost $J(\tau^a)$ of each plan $\tau^a$  is defined in \eqref{eq:cost2}. Given an initial state $q_B^0\in\ccalQ_B^0$, among all the motion plans $\tau^a\models\phi$, we select the one with the smallest cost $J(\tau^a)$ [line \ref{alg1:line15}, Alg. \ref{alg:plans}]. The plan with the smallest cost given an initial state $q_B^0$ is denoted by $\tau_{q_B^0}$. Then, among all plans $\tau_{q_B^0}$, we select again the one with smallest cost $J(\tau_{q_B^0})$, i.e., $\tau=\tau^{a_*}$, where $a_*=\text{argmin}_{a_{q_B^0}}{J(\tau_{q_B^0})}$ [lines \ref{alg1:line15a}-\ref{alg1:line16}, Alg. \ref{alg:plans}].

\subsection{Complexity Analysis}\label{sec:complexity}
The memory resources needed to store the PBA as a graph structure $\ccalG_P=\{\ccalV_P, \ccalE_P, w_P\}$, defined in Section \ref{sec:problem}, using its adjacency list is $O(|\ccalV_P|+|\ccalE_P|)$ \cite{sedgewick2011algorithms}. On the other hand, the memory needed to store a tree, constructed by Algorithm \ref{alg:tree}, that approximates the PBA is $O(|\ccalV_T|)$, since $|\ccalE_T|=|\ccalV_T|-1$. Due to the incremental construction of the tree we get that $|\ccalV_T|\leq|\ccalV_P|<|\ccalV_P|+|\ccalE_P|$ which shows that our proposed algorithm requires fewer memory resources compared to existing optimal control synthesis algorithms that rely on the construction of the PBA \cite{ulusoy2013optimality,ulusoy2014optimal}.

Moreover, the time complexity of sampling the state $q_\text{PTS}^{\text{new}}$ in Algorithm \ref{alg:sample1} is $O(|\ccalQ_B|+\max_i\{|\ccalQ_i|\} + \max_i\{(|\rightarrow_i|+|\ccalQ_i|\log(|\ccalQ_i|)\})$, where $|\rightarrow_i|$ denotes the total number of edges in the $\text{wTS}_i$ viewing it as a graph. The terms $|\ccalQ_B|$ and $\max_i\{|\ccalQ_i|\}$ capture the computational time to construct the reachable sets $\ccalR_B(q_B^{\text{rand}})$ and $\ccalR_{\text{wTS}_i}(q_i^{\text{rand}})$ for the largest state-space $\ccalQ_i$ among all robots, respectively; note that the sets $\ccalR_{\text{wTS}_i}(q_i^{\text{rand}})$ can be constructed simultaneously, i.e., in parallel, across the robots, which explains the $\max$ operator. Also, these reachable sets do not need to be computed on-the-fly; instead they can be pre-computed making the sampling process more computationally efficient.
Similarly, the term $ \max_i\{(|\rightarrow_i|+|\ccalQ_i|\log(|\ccalQ_i|)\})$ is due to the computation of the  shortest paths $SP_{q_i^{\text{rand}},q_i^{\ccalL(i)}}$ using the Djikstra algorithm. 
Moreover, the time complexity of `extend' and the `rewire' step is $O(|\ccalV_T|(N+1))$, as shown in Section IV-D in \cite{kantaros2017sampling}.

\begin{rem}[Implementation]
Note that in practice the prefix and suffix parts can be constructed in parallel. Specifically, as soon as a new final state  is detected during the construction of the prefix parts, the construction of a new tree rooted at this final state can be triggered immediately so that the respective suffix part is detected. Moreover, observe that the for-loop over all the initial states of the PBA in line \ref{alg1:line2a} of Algorithm \ref{alg:plans} can also be executed in parallel. A distributed implementation of the `sample', `extend', and `rewire' step is also presented in \cite{kantaros2017Dsampling}.
\end{rem}

\section{Convergence Analysis} \label{sec:corr}
In this section, we examine, the correctness, optimality, and convergence rate of $\text{STyLuS}^{*}$ described in Algorithm \ref{alg:plans}. Specifically, in Section \ref{sec:complOpt}, we show that $\text{STyLuS}^{*}$ is probabilistically complete and asymptotically optimal. Then, in Section \ref{sec:rate}, we show that Algorithm \ref{alg:tree} converges exponentially fast to the optimal solution of Problem \ref{pr:problem}. In what follows, we denote by $\mathcal{G}_T^n=\{\mathcal{V}_T^{n}, \mathcal{E}_T^{n}, \texttt{Cost}\}$ the tree that has been built by Algorithm \ref{alg:tree} at the $n$-th iteration for the construction of either a prefix or suffix part. We also denote the nodes $q_P^{\text{rand}}$ and $q_P^{\text{new}}$ at iteration $n$ by $q_P^{\text{rand},n}$ and $q_P^{\text{new},n}$, respectively. The same notation also extends to $f_{\text{new},i}$, $f_{\text{rand}}$, $p_{\text{rand}}$, and $p_{\text{new}}$. Also, in what follows we define the reachable set of state $q_P\in\ccalQ_P$ as follows
\begin{equation}\label{eq:reach}
\ccalR_P^{\rightarrow}(q_P)=\{q_P'\in\ccalQ_P~|~q_P'\rightarrow_P q_P\}\subseteq\ccalQ_P,
\end{equation}
i.e., $\ccalR_P^{\rightarrow}(q_P)$ collects all states $q_P'\in\ccalQ_P$ that can be reached in one hop from $q_P\in\ccalQ_P$

\subsection{Completeness and Optimality}\label{sec:complOpt}

In this section, we show that $\text{STyLuS}^{*}$ is probabilistically complete and asymptotically optimal. To show this, we first show the following results. The proofs of the following propositions can be found in Appendix \ref{sec:prop}.
\begin{prop}[Biased Probability Density $f_{\text{rand}}^n$]\label{prop:frand}
Assume that $p_{\text{rand}}^n>\epsilon$, for all $n\geq 1$ and for any $\epsilon>0$. Then, the probability density function $f_{\text{rand}}^n(q_P|\ccalV_T^n):\ccalV_T^n\rightarrow[0,1]$ defined in \eqref{eq:frand} is (i) bounded away from zero on $\ccalV_T^n$ (ii) bounded below by a sequence $g^{n}(q_P|\ccalV_T^n)$, such that $\sum_{n=1}^{\infty}g^n(q_P|\ccalV_T^n)=\infty$, for any given node $q_P\in\ccalV_T^n$, and 
(iii) can generate independent samples $q_P^{\text{rand},n}$.
\end{prop}

Proposition \ref{prop:frand} implies that $f_{\text{rand}}^n$ defined in \eqref{eq:frand} satisfies Assumption 4.1 (i) and (iii) in \cite{kantaros2017Csampling} and the relaxed version of Assumption 4.1 (ii) discussed in Remark A.1 in  \cite{kantaros2017Csampling}. Note that to ensure that the sampling process is biased towards a final state, $\epsilon$ should satisfy $\epsilon\geq0.5$. Note that, since $p_{\text{rand}}^n$ can change with $n$, as long as $p_{\text{rand}}^n>\epsilon$ for all $n\geq 1$, we can switch from biased to unbiased (uniform) sampling process by selecting $p_{\text{rand}}^n=|\ccalD_{\text{min}}^n|/|\ccalV_T^n|$, where in this case $\epsilon=1/|\ccalQ_P|$. Finally, due to Proposition \ref{prop:frand}, we can show that there exists an infinite sequence $\ccalK$, so that for all $k\in\ccalK$ it holds $q_{\text{PTS}}^{\text{rand},n+k}=q_\text{PTS}^{\text{rand},n}$; see Lemma 5.4 in \cite{kantaros2017Csampling}. 

\begin{prop}[Biased Probability Density $f_{\text{new}}^n$]\label{prop:fnew}
Assume that  $p_{\text{new}}^n>\epsilon$, for all $n\geq 1$ and for any $\epsilon>0$. Then,
the probability density functions $f_{\text{new}}^n(q_{\text{PTS}}|q_{\text{PTS}}^{\text{rand}})=\Pi_{i=1}^{N}f_{\text{new},i}^n(q_{i}|q_{i}^{\text{rand},n})$ defined in \eqref{eq:fnew}, \eqref{eq:fnewsuf1}, and \eqref{eq:fnewsuf2} (i) are bounded away from zero on $\ccalR_{\text{PTS}}(q_{\text{PTS}}^{\text{rand},n})$; (ii) for any fixed and given node  $q_{\text{PTS}}^{\text{rand},n}\in\ccalV_T^n$, there exists an infinite sequence $h^{n+k}(q_{\text{PTS}}|q_{\text{PTS}}^{\text{rand},n+k})$ so that the distributions $f_{\text{new}}^{n+k}(q_{\text{PTS}}|q_{\text{PTS}}^{\text{rand},n})$, for all $k\in\ccalK$, satisfy
$f_{\text{new}}^{n+k}(q_{\text{PTS}}|q_{\text{PTS}}^{\text{rand},n})\geq h^{n+k}(q_{\text{PTS}}|q_{\text{PTS}}^{\text{rand},n+k})$; and
(iii) given a state $q_{\text{PTS}}^{\text{rand},n}$, independent samples $q_{\text{PTS}}^{\text{new},n}$ can be drawn from the probability density function $f_{\text{new}}^n$. 
\end{prop}

Proposition \ref{prop:fnew} implies that the functions $f_{\text{new}}^n$ defined in \eqref{eq:fnew}, \eqref{eq:fnewsuf1}, and \eqref{eq:fnewsuf2} satisfy Assumption 4.2 (i) and (iii) in \cite{kantaros2017Csampling} and the relaxed version of Assumption 4.2 (ii) discussed in Remark A.2 in  \cite{kantaros2017Csampling}. Note that to ensure that the sampling process is biased towards a final state, $\epsilon$ should satisfy $\epsilon\geq 0.5$. Moreover, note that Proposition \ref{prop:fnew} holds even if the state towards which $f_{\text{new}}^n$ is biased at iteration $n$ changes. This means that during the construction of the prefix part we can switch the bias towards a different state $q_B^{F,\text{feas}}$ at any iteration $n$. As before, the fact that $p_{\text{new}}^n$ can change with $n$, as long as $p_{\text{new}}^n>\epsilon$ for all $n\geq 1$, allows us to switch from biased to unbiased (uniform) sampling process by selecting $p_{\text{new}}^n=1/|\ccalR_{\text{wTS}_i}(q_i^{\text{rand},n})|$, where in this case $\epsilon=1/|\ccalQ_i|$.

Since the probability density functions $f_{\text{rand}}^n$ and $f_{\text{new}}^n$ satisfy Assumptions 4.1 and 4.2 in \cite{kantaros2017Csampling}, we can show, following the same steps as in \cite{kantaros2017Csampling}, that $\text{STyLuS}^{*}$ is probabilistically complete and asymptotically optimal.
\begin{theorem}[Completeness and Optimality]\label{thm:alg1}
Assume that there exists a solution to Problem \ref{pr:problem}. Then, $\text{STyLuS}^{*}$ is probabilistically complete, i.e., it will find with probability $1$ a motion plan $\tau\models\phi$, as $n_{\text{max}}^{\text{pre}}\to\infty$ and $n_{\text{max}}^{\text{suf}}\to\infty$ and asymptotically optimal, i.e.,
it will find with probability $1$ the optimal motion plan $\tau\models\phi$,  as $n_{\text{max}}^{\text{pre}}\to\infty$ and $n_{\text{max}}^{\text{suf}}\to\infty$, that minimizes the cost function $J(\tau)$ defined in \eqref{eq:cost2}.
\end{theorem}

\subsection{Rate of Convergence}\label{sec:rate}
In this section, we show that $\text{STyLuS}^{*}$ converges exponentially fast to the optimal solution of Problem \ref{pr:problem}. To show this, we first show that it converges exponentially fast to a feasible solution of Problem \ref{pr:problem}.
\begin{theorem}[Convergence Rate to a Feasible Solution]\label{thm:conv}
Assume that there exists a solution to Problem \ref{pr:problem} in prefix-suffix form.
%
Let $\texttt{p}$ denote either the prefix or suffix part of this solution in the PBA defined as
 \begin{equation}\label{eq:path}
\texttt{p}=q_P^1,q_P^2,\dots,q_P^{K-1},q_P^{K},
\end{equation}
that is of length (number of states) $K$ and connects  the root $q_P^r=q_P^1$ of the tree to a state $q_P^K\in\ccalX_{\text{goal}}$, where $q_P^{k}\rightarrow_P q_P^{k+1}$, for all $k\in\{1,2,\dots,K-1\}$. Then, there exist parameters $\alpha_n(\texttt{p})\in(0,1]$ for every iteration $n$ of Algorithm \ref{alg:tree}, such that the probability $\Pi_{\text{suc}}(q_P^K)$ that Algorithm \ref{alg:tree} detects the state $q_P^K\in\ccalX_{\text{goal}}$ within $n_{\text{max}}$ iterations satisfies
\begin{equation}\label{eq:bnd1}
1 \geq \Pi_{\text{suc}}(q_P^K)\geq 1-e^{-\frac{\sum_{n=1}^{n_{\text{max}}}\alpha_n(\texttt{p})}{2}n_{\text{max}}+K}, \mbox{if}~ n_{\text{max}}>K.
\end{equation}
%
\end{theorem}

\begin{proof}
%
To prove this result, we model the sampling process discussed in Section \ref{sec:sample} as a Poisson binomial process. Specifically, we define Bernoulli random variables $Y_n$ at every iteration $n$ of Algorithm \ref{alg:tree} so that $Y_n=1$ only if the state $q_P^k$ is sampled and added to the tree at iteration $n$, where $k$ is the smallest element of the set $\{1,\dots\}$ that satisfies  $q_P^k\notin\ccalV_T^{n-1}$. Then, using the random variables $Y_n$, we define the random variable $Y=\sum_{n=1}^{n_{\text{max}}}Y_n$ which captures the total number of successes of the random variables $Y_n$ and we show that it follows a Poisson binomial distribution. Finally, we show that $\Pi_{\text{suc}}(q_P^K)\geq\mathbb{P}(Y\geq K)$ which yields \eqref{eq:bnd1} by applying the Chernoff bounds  \cite{motwani2010randomized} to $Y$.

Let $X_k^n$, for all $k\in\{1,\dots,K\}$ denote a Bernoulli random variable associated with iteration $n$ of Algorithm \ref{alg:tree}, that is equal to $1$ if the state $q_P^{k}=(q_\text{PTS}^k,q_B^k)$ in \eqref{eq:path} is added to the tree at iteration $n$ or has already been added to the tree at a previous iteration $m<n$, and is $0$ otherwise. Observe that $X_1^n$ is equal to  $1$ for all iterations $n\geq 1$ of Algorithm \ref{alg:tree}, since the tree is rooted at $q_P^1$ and, therefore, $q_P^1\in\ccalV_T^n$, for all $n\geq 1$. By construction of the sampling step in Section \ref{sec:sample} the probability that $X_k^n=1$ is defined as follows


\begin{equation}\label{eq:padd}
\mathbb{P}_n^{\text{add}}(q_P^k)=\left\{
                \begin{array}{ll}
\sum_{\substack{q_P\in\ccalV_T^n\cap\ccalR_P^{\rightarrow}(q_P^k)}}[
f_{\text{rand}}^n(q_P)f_{\text{new}}^n(q_\text{PTS}^k|q_{\text{PTS}})], \\~~~~~~\mbox{if}~q_P^k\notin\ccalV_T^n,\\
1,~~~\mbox{if}~q_P^k\in\ccalV_T^n,
                \end{array}
              \right.
\end{equation}
where $\ccalR_P^{\rightarrow}(q_P^k)$ is defined in \eqref{eq:reach}, and $q_i^k=\Pi|_{\text{wTS}_i}q_\text{PTS}^k$. Observe that if $\ccalV_T^n\cap\ccalR_P^{\rightarrow}(q_P^k)= \emptyset$ then $\mathbb{P}_n^{\text{add}}(q_P^k)=0$, if $q_P^k\notin\ccalV_T^n$. Moreover, note that if $q_P^k$ already belongs to $\ccalV_T^n$ from a previous iteration $m<n$ of Algorithm \ref{alg:tree}, then it trivially holds that $\mathbb{P}_n^{\text{add}}(q_P^k)=1$.


Given the random variables $X_k^n$, we define the discrete random variable $Y_n$ 
initialized as $Y_1=X_1^1$ and for every subsequent iteration $n>1$ defined as 

\begin{equation} \label{eq:Yn}
Y_n=\left\{
                \begin{array}{ll}
                 X_k^n, ~~\mbox{if}~(Y_{n-1}=X_k^{n-1})\wedge(X_k^{n-1}=0)\\
                 X_{k+1}^n, \mbox{if}~(Y_{n-1}=X_k^{n-1})\wedge(X_k^{n-1}=1)\\
                 ~~~~~~~~~~~~~~~~~~~~~~~~~~~~~~\wedge (k+1\leq K)\\
                  X_K^n, ~~\mbox{if}~(Y_{n-1}=X_k^{n-1})\wedge(X_k^{n-1}=1)\\
                 ~~~~~~~~~~~~~~~~~~~~~~~~~~~~~~ \wedge (k+1> K) \\
                \end{array}.
              \right.
\end{equation}

%

In words,  $Y_n$ is defined exactly as $Y_{n-1}$, i.e, $Y_n=Y_{n-1}=X_k^{n-1}=X_k^n$, if $Y_{n-1}=X_{k}^{n-1}=0$, i.e., if the state $q_P^k$ associated with the random variable $Y_{n-1}=X_k^{n-1}$ does not exist in the tree at iteration $n-1$. Thus, in this case, $Y_n=X_{k}^{n}=1$ if the state $q_P^{k}$ in \eqref{eq:path} is added to the tree at iteration $n$; see the first case in \eqref{eq:Yn}.
Also, $Y_n=X_{k+1}^n$, if $Y_{n-1}=X_{k}^{n-1}=1$, i.e., if the state $q_P^{k}$ was added to the tree at the previous iteration  $n-1$. In this case, $Y_n=X_{k+1}^n=1$, if the next state $q_P^{k+1}$ in \eqref{eq:path} is added to the tree at iteration $n$ (or has already been added at a previous iteration $m<n$); see the second case in \eqref{eq:Yn}. If $k+1>K$ and $X_k^{n-1}=1$, then we define $Y_n=X_K^n$; see the last case in \eqref{eq:Yn}. Note that in this case, $Y_n$ can be defined arbitrarily, i.e., $Y_n=X_{\bar{k}}^n$, for any $\bar{k}\in\{1,\dots,K\}$, since if $k+1>K$ and $X_k^{n-1}=1$, then this means that all states that appear in \eqref{eq:path} have been added to  $\ccalV_T^n$. By convention, in this case we define  $Y_n=X_K^n$.\footnote{Note that in general the states $q_P^k$ in \eqref{eq:path} will not be sampled and added to the tree in the order they appear in \eqref{eq:path}.} 
Since $Y_n$ is equal to $X_n^k$ for some $k\in\{1,\dots,K\}$, as per \eqref{eq:Yn}, for all $n\geq 1$, we get that $Y_n$ also follows a Bernoulli distribution with parameter (probability of success)  $p^{\text{suc}}_n$ equal to the probability of success of $X_n^k$ defined in \eqref{eq:padd}, i.e.,
$$p^{\text{suc}}_n=\mathbb{P}_n^{\text{add}}(q_P^k),$$ where the index $k$ is determined as per \eqref{eq:Yn}.


Given the random variables $Y_n$, $n\in\{1,\dots,n_{\text{max}}\}$, we define the discrete random variable $Y$ as
\begin{equation}\label{eq:Y}
Y=\sum_{n=1}^{n_{\text{max}}}Y_n.
\end{equation}
Observe that $Y$ captures the total number of successes of the random variables $Y_n$ after $n_{\text{max}}>0$ iterations, i.e., if $Y=y$, $y\in\{1,\dots,n_{\text{max}}\}$, then $Y_n=1$ for exactly $y$ random variables $Y_n$.
Therefore, if $Y\geq K$, then all states that appear in the path $\texttt{p}$ of length $K$ given in \eqref{eq:path} have been added to the tree, by definition of the random variables $Y_n$ and $Y$ in \eqref{eq:Yn} and \eqref{eq:Y}, respectively. 
Since there may be more than one path in the PBA that connect the goal state $q_P^K$ to the root $q_P^1$, we conclude that the probability $\Pi_{\text{suc}}(q_P^K)$ of finding the goal state $q_P^K$ within $n_{\text{max}}$ iterations is at least $\mathbb{P}(Y\geq K)$, i.e.,\footnote{Recall that $Y_n=X_k^n=1$ implies that the state $q_P^k$ is added (if not already added from a previous iteration) to the tree at iteration $n$. Nevertheless, this does not necessarily mean that the parent of $q_P^k$ will be the node $q_P^{k-1}$. As a result, $Y\geq K$ (or $Y_n=X_K^n=1$) implies that the goal state $q_P^K$ along with all other states that belong to $\texttt{p}$ have been added to the tree but the path $\texttt{p}$ may not exist in the tree $\ccalG_T^n$. Also, the equality in \eqref{eq:Psuc} holds if \eqref{eq:path} is the only path in the PBA that connects $q_P^K$ to the root $q_P^1$.} 
\begin{equation}\label{eq:Psuc}
\Pi_{\text{suc}}(q_P^K)\geq  \mathbb{P}(Y\geq K).
\end{equation}
%
%

In what follows, our goal is to compute the probability $\mathbb{P}(Y\geq K)$. Observe that $Y$ is defined as a sum of Bernoulli random variables $Y_n$ that are not identically distributed as their probabilities of success $p^{\text{suc}}_n$ are not fixed across the iterations $n$, since the definition of $Y_n$ changes at every iteration $n$ as per \eqref{eq:Yn}.
Therefore, $Y$ follows a Poisson Binomial distribution which has a rather complicated pmf which is valid for small $n$ and numerically unstable for large $n$; see e.g., \cite{choi2002approximating,hong2011computing}. Therefore, instead of computing $\mathbb{P}(Y\geq K)$, we compute a lower bound for $\mathbb{P}(Y\geq K)$ by applying the Chernoff bound to $Y$; see e.g., \cite{motwani2010randomized}. 
%
Specifically, we have that
\begin{align}\label{eq:prbfailk}
&\mathbb{P}(Y<K)<\mathbb{P}(Y\leq K)=\mathbb{P}(Y\leq K\frac{\mu}{\mu})\\&
=\mathbb{P}\left(Y\leq(1-\underbrace{(1-\frac{K}{\mu})}_{=\delta})\mu\right)=\mathbb{P}(Y\leq(1-\delta)\mu)\leq e^{-\frac{\mu\delta^2}{2}},\nonumber
\end{align}
where $\mu$ is the mean value of $Y$ defined as $\mu=\sum_{n=1}^{n_{\text{max}}} p^{\text{suc}}_n$. Also, the last inequality in \eqref{eq:prbfailk} is due to the Chernoff bound in the lower tail of $Y$ and holds for  any $\delta\in(0,1)$. Observe that the Chernoff bound can be applied to $Y$, as it is defined as the sum of \textit{independent} Bernoulli random variables $Y_n$. Specifically, the random variables $Y_n$ are independent since independent samples $q_P^{\text{new}}$ can be generated by the proposed sampling process, described in Section \ref{sec:sample}, by construction of the density functions $f_{\text{rand}}$ and $f_{\text{new},i}$. Substituting $\delta=1-\frac{K}{\mu}$ in \eqref{eq:prbfailk}, we get 

\begin{equation}\label{eq:bound}
\mathbb{P}(Y<K)\leq e^{-\frac{\mu}{2}+K-\frac{K^2}{\mu}}\leq e^{-\frac{\mu}{2}+K}=e^{-\frac{\sum_{n=1}^{n_{\text{max}}} p^{\text{suc}}_n}{2}+K},
\end{equation}
where the last inequality is due to $e^{-\frac{K^2}{\mu}}\leq 1$. Recall that \eqref{eq:bound} holds for any $\delta=1-\frac{K}{\mu}\in(0,1)$, i.e, for any $n_{\text{max}}$ that satisfies 
\begin{equation}\label{eq:cond1}
0<\delta<1 \implies K<\mu=\sum_{n=1}^{n_{\text{max}}} p^{\text{suc}}_n\implies K<n_{\text{max}},
\end{equation}
where the last inequality in \eqref{eq:cond1} is due to $ p^{\text{suc}}_n\leq 1$. Therefore, \eqref{eq:bound} holds as long as $n_{\text{max}}>K$. 

Note also that the inequality $0<K<\sum_{n=1}^{n_{\text{max}}} p^{\text{suc}}_n$ in \eqref{eq:cond1} is well defined, since $p^{\text{suc}}_n=\mathbb{P}_n^{\text{add}}(q_P^k)$ is strictly positive for all $n\geq 1$ by definition of $Y_n$. To show that, observe that if $Y_n=X_k^n$, for some $k\in\{1,\dots,K\}$, then it holds that $q_P^{k-1}\in\ccalV_T^n$, by construction of $Y_n$ in \eqref{eq:Yn}. Then, by definition of the feasible path \eqref{eq:path}, we have that $q_P^{k-1}\in \ccalR_P^{\rightarrow}(q_P^k)$. Thus, we get that  $q_P^{k-1}\in\ccalV_T^n\cap \ccalR_P^{\rightarrow}(q_P^k)$. 
Then, $p^{\text{suc}}_n=\mathbb{P}_n^{\text{add}}(q_P^k)>0$ (see also \eqref{eq:padd}) holds, since (i) $\ccalV_T^n\cap \ccalR_P^{\rightarrow}(q_P^k)\neq \emptyset$, as $q_P^{k-1}\in\ccalV_T^n\cap \ccalR_P^{\rightarrow}(q_P^k)$; (ii) $f_{\text{rand}}^n$ is bounded away from zero $\ccalV_T^n$ by construction in \eqref{eq:frand} (as long as $p_{\text{rand}}^n>0$), for all $n\geq 1$, i.e., $f_{\text{rand}}^n(q_P^{k-1})>0$, $q_P^{k-1}\in\ccalV_T^n$; and (iii) $f_{\text{new},i}^n$ is bounded away from zero on $\ccalR_{\text{wTS}_i}(q_i^{k-1})$, for all $n\geq 1$ and for all $i\in\{1,\dots,N\}$ by construction in \eqref{eq:fnew}, \eqref{eq:fnewsuf1}, and \eqref{eq:fnewsuf2}  (as long as $p_{\text{new}}^n>0$), i.e.,  $f_{\text{new},i}^n(q_i^k|q_i^{k-1})>0$, where $q_i^k=\Pi|_{\text{wTS}_i}q_P^k$.

Thus, we proved that there exist parameters 
\begin{equation}
\alpha_n(\texttt{p})=p^{\text{suc}}_n\in(0,1],
\end{equation}
associated with every iteration $n$ of Algorithm \ref{alg:tree} such that the probability $\Pi_{\text{suc}}(q_P^K)$ of finding the goal state $q_P^K$ within $n_{\text{max}}>K$ iterations satisfies
\begin{equation}
1\geq\Pi_{\text{suc}}(q_P^K)\geq1-\mathbb{P}(Y<K)\geq 1-e^{-\frac{\sum_{n=1}^{n_{\text{max}}} \alpha_n(\texttt{p})}{2}+K},\nonumber 
\end{equation}
completing the proof. \end{proof}

Observe that Theorem \ref{thm:conv} holds for any (biased or unbiased) density functions $f_{\text{rand}}^n$ and $f_{\text{new},i}^n$, that can possibly change with iterations $n$, as long as they (i) are bounded away from zero on $\ccalV_T^n$ and $\ccalR_{\text{wTS}_i}(q_i)$,  respectively; 
 and (ii) can generate independent samples, which are essentially the Assumptions 4.1(i), 4.1(iii), 4.2(i), and 4.2(iii), respectively. Therefore, the probability of finding a feasible plan converges exponentially to $1$ even if we switch to an unbiased density function by selecting $p_{\text{rand}}^n=|\ccalD_{\text{min}}^n|/|\ccalV_T^n|$ and $p_{\text{new}}^n=1/|\ccalR_{\text{wTS}_i}(q_i^{\text{rand},n})|$.

\begin{rem}[Effect of bias]\label{rem:conv}
The parameters $\alpha_n(\texttt{p})$ in Theorem \ref{thm:conv} capture the probability of sampling a state that belongs to a path $\texttt{p}$ in \eqref{eq:path} and adding it to the tree $\ccalG_T^n$ at iteration $n$; see also \eqref{eq:padd}. The larger the parameters $\alpha_n(\texttt{p})$ are, the larger the probability $\Pi_{\text{suc}}(q_P^K)$ is of detecting the goal state $q_P^K\in\ccalX_{\text{goal}}$ given any feasible path $\texttt{p}$ of length $K$ that connects $q_P^K$ to the root $q_P^1$ within $n_{\text{max}}>K$ iterations; see also \eqref{eq:bnd1}. 
This also implies that for a given $q_B^{F,\text{feas}}$ if the density functions $f_{\text{rand}}^n$ and $f_{\text{new},i}^n$ are biased towards feasible states, i.e., towards states $q_P^{\text{new},n}$ with $\ccalR_P^{\rightarrow}(q_P^{\text{new},n})\neq\emptyset$, then there exists at least one feasible (either prefix or suffix) path $\texttt{p}$ associated with $q_B^{F,\text{feas}}$ so that its parameters $\alpha_n(\texttt{p})$ are larger than the parameters of any other feasible path, which can be associated with any NBA final state in case of the prefix part. As a result, this path will be detected faster than any other feasible path in the PBA. If this path  $\texttt{p}$ is a prefix part, then it ends at a PBA final state that is associated with $q_B^{F,\text{feas}}$ while if it is a suffix part, its end state is the root.
%
Note that using biased density functions does not mean that \textit{all}, e.g., prefix paths associated with $q_B^{F,\text{feas}}$ will be detected faster than if unbiased density functions were used. Instead, what this means is that if the sampling process is biased towards feasible states $q_P^{\text{new},n}$, then there exists \textit{at least one} prefix path associated with $q_B^{F,\text{feas}}$ that will be detected much faster.  Also, paths which the sampling process is not biased to are expected to be detected slower than if unbiased (uniform) density functions are employed. 
Finally, recall that the sampling process is towards shortest paths, in terms of numbers of hops, in the state-space of both the wTSs and the NBA. Therefore, the paths that the sampling functions are biased to are expected to have a small  $K$, i.e., a small number of transitions/states. The latter is also verified through numerical experiments in Section \ref{sec:sim}.
\end{rem}

\begin{rem}[Biasing towards infeasible states]
Note that if the density functions $f_\text{rand}^n$ and $f_{\text{new},i}^n$ are biased towards states that violate the LTL formula, i.e., states $q_P^{\text{new},n}$ with $\ccalR_P^{\rightarrow}(q_P^{\text{new},n})=\emptyset$, then the parameters $\alpha_n(\texttt{p})$, associated with a \textit{feasible} path $\texttt{p}$, will have low values, by construction of  $f_\text{rand}^n$ and $f_{\text{new},i}^n$. As a result, a large number of iterations is expected until $\texttt{p}$ is detected due to Theorem \ref{thm:conv}. Also, note that in general, it is not possible to know \textit{a priori} whether the states in the PBA which the sampling process is biased to are feasible or not; therefore, in practice, values for $p_{\text{rand}}^n$ and $p_{\text{new}}^n$ that are arbitrarily close to 1 should be avoided, as they might significantly slow down the synthesis part. 
\end{rem}

Using Theorem \ref{thm:conv}, and Propositions \ref{prop:frand}-\ref{prop:fnew}, we establish the exponential convergence rate to the optimal, either prefix or suffix, path that connects a desired goal state to the root of the tree.

\begin{theorem}[Covergence Rate to the Optimal Path]\label{thm:convOpt}
Assume that there exists a solution to Problem \ref{pr:problem} in prefix-suffix form. 
%
Let $\texttt{p}^*$ denote either the prefix or suffix part of the optimal solution in the PBA defined as
 \begin{equation}\label{eq:optPath}
\texttt{p}^*=q_P^1,q_P^2,\dots,q_P^{K-1},q_P^{K},
\end{equation}
that is of length $K$ (number of states) and connects  the root $q_P^r=q_P^1$ of the tree to a state $q_P^K\in\ccalX_{\text{goal}}$, where $q_P^{k}\rightarrow_P q_P^{k+1}$, for all $k\in\{1,2,\dots,K-1\}$. Then, there exist parameters $\gamma_n(q_P^k)\in(0,1]$ for every state $q_P^k$ in \eqref{eq:optPath} and every iteration of Algorithm \ref{alg:tree}, as well as an iteration $n_k$ of Algorithm \ref{alg:tree} for every state $q_P^k$  in \eqref{eq:optPath}, so that the probability $\Pi_{\text{opt}}(\texttt{p}^*)$ that Algorithm \ref{alg:tree} detects the optimal path \eqref{eq:optPath} within $n_{\text{max}}$ iterations satisfies
\begin{align}\label{eq:bndOpt}
\Pi_{\text{opt}}(\texttt{p}^*)\geq&\prod_{k=1}^{K-1}(1-e^{-\frac{\sum_{n=n_{k-1}}^{n_{\text{max}}} \gamma_n(q_P^k)}{2}+1})\nonumber\\&(1-e^{-\frac{\sum_{n=1}^{\bar{n}} \alpha_n(\texttt{p}^*)}{2}+K}),
\end{align}
if $n_{\text{max}}>2K$. In \eqref{eq:bndOpt}, the parameters $\alpha_n(\texttt{p}^*)$ are defined as in Theorem \ref{thm:conv} and $\bar{n}$ is any fixed iteratition of Algorithm \ref{alg:tree} that satisfies $\bar{n}\leq n_{\text{max}}-K$.
\end{theorem}

\begin{proof}
To show this result, first we show that if all states in \eqref{eq:optPath} exist in the tree $\ccalG_T^{\bar{n}}$ within $\bar{n}<n_{\text{max}}$ iterations of Algorithm \ref{alg:tree}, and if all these states are eventually selected (after iteration $\bar{n}$) to be the nodes $q_P^{\text{new},n}$ in the order that they appear in \eqref{eq:optPath}, then the optimal path \eqref{eq:optPath} is detected by Algorithm \ref{alg:tree}. 
Second, we compute the probability that the above statement is satisfied within $n_\text{max}$ iterations. Finally, we show that this probability is a lower bound to the probability of finding the optimal path \eqref{eq:optPath}, which yields  \eqref{eq:bndOpt}.

%
First, assume that (i) all states $q_P^k$ exist in the tree at iteration $\bar{n}<n_\text{max}$, and (ii)  there exists a finite sequence $\ccalN:=\{n_1,n_2,\dots,n_{K-1}\}$ of iterations of Algorithm \ref{alg:tree}, where $n_{k+1}>n_{k}>\bar{n}$ so that $q_P^k=q_P^{\text{new},n_k}$. In what follows, we use induction to show that the optimal path is detected at iteration $n_{K-1}$. Specifically, by assumptions (i) and (ii), we have that $q_P^1$ will be selected to be the state $q_P^{\text{new},n_1}$ at iteration $n_1$ and that $q_P^2\in\ccalV_T^{n_1}$. 
Observe that by definition of the optimal path \eqref{eq:optPath}, the optimal cost of $q_P^2$ is obtained when $q_P^2$ is directly connected to the root $q_P^1$. Therefore, by construction of the rewiring step, $q_P^2\in\ccalV_T^{n_1}$ will get rewired to $q_P^1\in\ccalV_T^{n_1}$ at iteration $n_1$. 
Next, the inductive step follows. By assumptions (i) and (ii), it holds that $q_P^{k-1}$ is selected to be the node $q_P^{\text{new},n_{k-1}}$ at iteration $n_{k-1}>n_{k-2}$ and that $q_P^k\in\ccalV_T^{n_{k-1}}$. Assume that $q_P^{k}\in\ccalV_T^{n_{k-1}}$ gets rewired to $q_P^{k-1}\in\ccalV_T^{n_{k-1}}$, which is connected to the root as in \eqref{eq:optPath} (inductive assumption). 
We will show that at the subsequent iteration $n_k>n_{k-1}$, $q_P^{k+1}$ will get rewired to $q_P^{k}$, which is connected to the root as in \eqref{eq:optPath}.
In particular, by assumptions (i) and (ii), we have that there exists an iteration $n_k>n_{k-1}$, such that $q_P^{k}\in\ccalV_T^{n_{k}}$ is selected to be the node $q_P^{\text{new},n_{k}}$ while $q_P^{k+1}\in\ccalV_T^{n_{k}}$. Note that by definition of the optimal path \eqref{eq:optPath}, the optimal cost of $q_P^{k+1}$ is obtained when $q_P^{k+1}$ is connected to the root $q_P^1$ as in \eqref{eq:optPath}. Therefore, by construction of the rewiring step, we conclude that $q_P^{k+1}$ will get rewired to $q_P^{k}\in\ccalV_T^{n_{k}}$ at iteration $n_k$.
%
Then, by induction, we conclude that at iteration $n_{K-1}$, the state $q_P^K$ will get rewired to $q_P^{K-1}$ which is connected to the root as in \eqref{eq:optPath}, i.e., the optimal path \eqref{eq:optPath} exists in the tree at iteration $n_{K-1}$.

The probability that both assumptions (i) and (ii) will become true within $\bar{n}$ and $n_{\text{max}}$ iterations, respectively,
is $\mathbb{P}(R(n_{\text{max}})\cap A(\bar{n}))=\mathbb{P}(R(n_{\text{max}})|A(\bar{n}))\mathbb{P}(A(\bar{n}))$, where $A(\bar{n})$ is the event that all states in \eqref{eq:optPath} have been added to the tree at iteration $\bar{n}<n_{\text{max}}$ and $R(n_{\text{max}})$ denotes the event that all the states that appear in the optimal path \eqref{eq:optPath} have been selected to be the nodes $q_P^{\text{new}}$ in the order they appear in \eqref{eq:optPath} after the iteration $\bar{n}$ and before the iteration $n_\text{max}$. 
Then, the probability that the optimal path is detected within $n_{\text{max}}$ iterations satisfies
\begin{align}\label{eq:optIneq}
\Pi_{\text{opt}}(\texttt{p}^*)\geq\mathbb{P}(R(n_{\text{max}})|A(\bar{n}))\mathbb{P}(A(\bar{n})),
\end{align}
where the inequality in \eqref{eq:optIneq} is due to the fact that satisfying assumptions (i) and (ii) is a sufficient, and not necessary, condition to find the optimal path.

In what follows, we compute a lower bound for the probabilities $\mathbb{P}(R(n_{\text{max}})|A(\bar{n}))$ and $\mathbb{P}(A(\bar{n}))$ which we will then substitute in \eqref{eq:optIneq} to get \eqref{eq:bndOpt}. 
Recall that  $\mathbb{P}(R(n_{\text{max}})|A(\bar{n}))$ captures the probability that Algorithm \ref{alg:tree} will reach all iterations $n_1<n_2<\dots<n_{K-1}$, after the iteration $\bar{n}$ (i.e., $\bar{n}<n_1$)  and before the iteration $n_{\text{max}}$, given the event $A(\bar{n})$. 
Therefore, we first need to show that the sequence $\ccalN=\{n_1,\dots,n_{K-1}\}$ exists  since, otherwise, it trivially holds that $\mathbb{P}(R(n_{\text{max}})|A(\bar{n}))=0$ for any $n_{\text{max}}$.\footnote{Note that here we prove that such iterations $n_k$ exist and not that they necessarily satisfy $n_k\leq n_{\text{max}}$, for all $k\in\{1,\dots,K-1\}$. Recall that the probability that all iterations $n_k$ satisfy $\bar{n}<n_k\leq n_{\text{max}}$ is $\mathbb{P}(R(n_{\text{max}})|A(\bar{n}))$.} Assume that such a sequence $\ccalN$ does not exist. 
Then, this means that there exists an iteration $n_{\bar{K}}\in\ccalN$ , where  $1\leq\bar{K}<K$, such that after the iteration $n_{\bar{K}}$ the state $q_P^{\bar{K}+1}$ will never be selected to be the node $q_P^{\text{new}}$. The latter implies that after the iteration $n_{\bar{K}}$ the state
$q_P^{\bar{K}+1}$ that is reachable from $q_P^{\bar{K}}\in\ccalV_T^{n_{\bar{K}}}$ as per \eqref{eq:optPath}, i.e., $q_P^{\bar{K}}\rightarrow_P q_P^{\bar{K}+1}$, will not be selected infinitely often to be the node $q_P^{\text{new}}$, which contradicts  Corollary 5.6 of \cite{kantaros2017Csampling}.\footnote{Recall that Corollary 5.6 of \cite{kantaros2017Csampling} holds since the employed sampling density functions satisfy Assumptions 4.1 and 4.2 of \cite{kantaros2017Csampling}, as shown in Propositions \ref{prop:frand}-\ref{prop:fnew}.} Therefore, we conclude that the sequence $\ccalN$ exists. Note that given $K$ and $n_{\text{max}}$, the sequence $\ccalN$ can only exist if
\begin{equation}\label{eq:barn}
\bar{n}<n_\text{max}-K,
\end{equation}
since otherwise, there will not be enough iterations left after $\bar{n}$ so that the states $q_P^k$, $k\in\{1,\dots,K-1\}$ are selected to be the nodes $q_P^{\text{new},n}$.

Next, we compute a lower bound for the probability $\mathbb{P}(R(n_{\text{max}})|A(\bar{n}))$. 
%
To do this, we first define Bernoulli random variables $Z_k^n$ so that $Z_k^n=1$
if $q_P^k$ is selected to be the node $q_P^{\text{new},n}$ at iteration $n$ of Algorithm \ref{alg:tree}, and $0$ otherwise. By construction of the sampling step, $Z_k^n=1$ occurs with probability
\begin{equation}\label{eq:pnew1}
\mathbb{P}_n^{\text{new}}(q_P^k)=
\sum_{\substack{q_P\in\ccalV_T^n\cap\ccalR_P^{\rightarrow}(q_P^k)}}[
f_{\text{rand}}^n(q_P)f_{\text{new}}^n(q_\text{PTS}^k|q_{\text{PTS}})],
\end{equation}
where $\ccalR_P^{\rightarrow}(q_P^k)$ is defined in \eqref{eq:reach} and $f_{\text{new}}^n(q_\text{PTS}^k|q_{\text{PTS}})=\prod_{i=1}^N f_{\text{new},i}^n(q_i^k|q_i)$, and $q_i^k=\Pi|_{\text{wTS}_i}q_\text{PTS}^k$.

Given $Z_k^n$, we define the following discrete random variables for all states $q_P^k$, $k\in\{1,\dots,K-1\}$
\begin{equation}\label{eq:zk}
Z_k=\sum_{n=n_{k-1}}^{n_\text{max}}Z_k^n,
\end{equation}
where $n_0=\bar{n}$. Notice that the sum in \eqref{eq:zk} is well-defined as long as Algorithm \ref{alg:tree} has reached iteration $n_{k-1}$ within $n_{\text{max}}$ iterations.
Also, note that $Z_k$ follows a Poisson Binomial distribution, since it is defined as the sum of Bernoulli random variables that are not identically distributed. Moreover, observe  that $Z_k$ captures the total number of times $q_P^k$ was selected to be the node $q_P^{\text{new},n}$ when $n\in [n_{k-1},n_\text{max}]$ i.e., after $q_P^{k-1}\in\ccalV_T^n$ was selected to be the node $q_P^{\text{new},n_{k-1}}$. Therefore, $Z_k\geq1$ means that $q_P^k$ was selected to be $q_P^{\text{new},n}$ at least once in the interval $[n_{k-1},n_\text{max}]$. Thus, we can rewrite the event $R(n_{\text{max}})$ as 
\begin{equation}\label{eq:Rmax}
R(n_{\text{max}})=\cap_{k=1}^{K-1}(Z_k\geq 1).
\end{equation}
Given \eqref{eq:Rmax} and using the Bayes rule, $\mathbb{P}(R(n_{\text{max}})|A(\bar{n}))$ can be written as follows
\begin{equation}\label{eq:result0}
\mathbb{P}(R(n_{\text{max}})|A(\bar{n}))=\prod_{k=1}^{K-1}\mathbb{P}((Z_k\geq 1)|B_{k},A(\bar{n})),
\end{equation}
where $B_{k}=\cap_{b=1}^{k-1}(Z_{b}\geq 1)$, for all $k>1$ and $B_0=\varnothing$. In words, $\mathbb{P}((Z_k\geq 1)|B_{k},A(\bar{n}))$ captures the probability of reaching iteration $n_k$ within $n_{\text{max}}$ iterations, given that iteration $n_{k-1}$ has already been reached and all states $q_P^k$ have been added to the tree by the iteration $\bar{n} <n_1$. Also, observe that $Z_k$ is well-defined in \eqref{eq:result0}. Specifically, recall that $Z_k$, defined in \eqref{eq:zk}, is well defined as long as Algorithm \ref{alg:tree} reaches iteration $n_{k-1}$ within $n_{\text{max}}$ iterations. This requirement is satisfied in \eqref{eq:result0} since the probability of $Z_k\geq 1$ is computed given the events $B_{k}$ and $A(\bar{n})$.

Following the same logic as in the proof of Theorem \ref{thm:conv}, we can show that 
\begin{equation}\label{eq:result1}
\mathbb{P}((Z_k\geq 1)|B_{k},A(n_{\text{max}}))\geq 1-e^{-\frac{\sum_{n=n_{k-1}}^{n_{\text{max}}} \gamma_n(q_P^k)}{2}+1},
\end{equation}
as long as $n_{\text{max}}> n_{k-1}$. In \eqref{eq:result1}, $\gamma_n(q_P^k)=\mathbb{P}_n^{\text{new}}(q_P^k)$ defined in \eqref{eq:pnew1}. Note that $\gamma_n(q_P^k)>0$ since $\ccalV_T^n\cap\ccalR_P^{\rightarrow}(q_P^k)\neq\emptyset$, for all $n\in [n_{k-1},n_{\text{max}}]$, since in \eqref{eq:result1}, we assume that the event $B_{k}$ is true. As a result, $q_P^{k-1}\in\ccalV_T^n$, for all $n\in[n_{k-1},n_{\text{max}}]$. Also, by construction of the path \eqref{eq:optPath}, it holds that $q_P^{k-1}\in\ccalR_P^{\rightarrow}(q_P^k)$ which means that $\ccalV_T^n\cap\ccalR_P^{\rightarrow}(q_P^k)\neq\emptyset$.
%

Moreover, recall from the proof of Theorem \ref{thm:conv} that 
\begin{equation}\label{eq:a}
\mathbb{P}(A(\bar{n}))=\Pi_{\text{suc}}(q_P^K)\geq 1-e^{-\frac{\sum_{n=1}^{\bar{n}} \alpha_n(\texttt{p}^*)}{2}+K},
\end{equation}
which holds for $\bar{n}>K$, where $K$ is the length of any path, including the optimal path $\texttt{p}^*$, that connects $q_P^K$ to the root of the tree. 

Substituting equations  \eqref{eq:result0}, \eqref{eq:result1}, and \eqref{eq:a} into \eqref{eq:optIneq} yields \eqref{eq:bndOpt}
%
%
which holds for $\bar{n}<n_{\text{max}}-K$  due to \eqref{eq:barn} and $\bar{n}>K$ due to \eqref{eq:a} or, equivalently, for $\bar{n}<n_{\text{max}}-K$ and $n_{\text{max}}>2K$ completing the proof.
\end{proof}

Note that Theorem \ref{thm:convOpt} holds for any (biased or unbiased) density functions $f_{\text{rand}}^n$ and $f_{\text{new},i}^n$, that can possibly change with iterations $n$, as long as they satisfy Assumptions 4.1 and 4.2 of \cite{kantaros2017Csampling}, respectively. Therefore, the probability of finding the optimal plan converges exponentially to $1$ even if at any iteration $n$ we switch bias to a different state or switch to an unbiased density function by selecting $p_{\text{rand}}^n=|\ccalD_{\text{min}}^n|/|\ccalV_T^n|$ and $p_{\text{new}}^n=1/|\ccalR_{\text{wTS}_i}(q_i^{\text{rand},n})|$.



\section{Numerical Experiments}\label{sec:sim}

\begin{table*}[t]
\caption{\footnotesize{Feasibility and scalability analysis: $|\ccalQ_B|=21$}}\label{tab:table1}
\begin{center}
\begin{tabular}{c|c|c|c|c|c|c|c|}
\cline{1-7}
		\multicolumn{1}{|c|}{$N$}		&
						  $|\ccalQ_i|$ & $|\ccalQ_{\text{P}}|$ & $n_{\text{Pre1}} + n_{\text{Suf1}}$ & $|\ccalV_T^{\text{Pre1}}| + |\ccalV_T^{\text{Suf1}}|$ & \texttt{Pre1+Suf1}  & NuSMV/nuXmv \\ \hline		
	\multicolumn{1}{|c|}{$1$}		& $100$ &  $10^{3}$ & 28 + 28 & 180 + 54 & 0.5 + 0.2 (secs)   & $<$ 1 sec\\ \hline			
	\multicolumn{1}{|c|}{$1$}		& $1000$ & $10^{3}$ & 42 +31  & 338 + 119 & 0.9 + 0.7 (secs) & $<$ 1 sec \\ \hline				
	\multicolumn{1}{|c|}{$1$}		& $10000$ & $10^{4}$  &71 + 43 & 512 + 131& 19.2 + 11.2 (secs) & M/M \\ \hline
	\multicolumn{1}{|c|}{$9$}		& $9$ &  $10^{10}$& 36 + 37 & 373 + 83 &0.78 + 0.29 (secs) & $<1$ sec \\ \hline
	\multicolumn{1}{|c|}{$10$}		& $100$ & $10^{21}$  &31+ 31 & 289 + 101 & 0.7 + 0.4 (secs) & $\approx$ 1.5 secs \\ \hline			
	\multicolumn{1}{|c|}{$10$}		& $1000$ & $10^{31}$  & 34 + 27 &309 + 82& 2.1 + 1.7 (secs) & $\approx$50/40 secs \\ \hline
    \multicolumn{1}{|c|}{$10$}		& $2500$ & $10^{35}$  & 41 + 32 & 367 + 142 & 5.94 + 6.4 (secs) &  M/$\approx 1860$ secs \\ \hline
	\multicolumn{1}{|c|}{$10$}		& $10000$ & $10^{41}$  & 40 + 23  & 357+123&48.44 + 38.1 (secs) & M/M \\ \hline
	\multicolumn{1}{|c|}{$100$}	&  $100$ & $10^{200}$  &49 + 39 & 421 + 81 & 2.5 + 0.8 (secs) & F/F \\ \hline
	\multicolumn{1}{|c|}{$100$}	& $1000$  & $10^{300}$ &30 + 38 & 254 + 110 &12.6 + 24.1 (secs) & M/M  \\ \hline
     \multicolumn{1}{|c|}{$100$}  & $10000$  & $10^{400}$  & 24 + 49 & 241 + 55 &252.1 + 312.1 (secs) &M/M\\  \hline
     \multicolumn{1}{|c|}{$150$}  & $10000$  & $10^{600}$  & 29 + 87 & 382+530 & 365.91 + 1662.6 (secs) &M/M\\  \hline
     \multicolumn{1}{|c|}{$200$}  & $10000$  & $10^{800}$   & 42 + 49 & 453 + 276 &678.1 + 1410.1 (secs) &M/M\\  \hline
\end{tabular}
\end{center}
\end{table*}

In this section, we present case studies, implemented using MATLAB R2015b on a computer with Intel Core i7 2.2GHz and 4Gb RAM, that illustrate the scalability of our proposed algorithm compared to state-of-the-art methods. Simulation studies that show the advantage of the employed reachability-based sampling process, discussed in Section \ref{sec:sample}, compared to sampling approaches that generate random states $q_P^{\text{new}}$ that belong to $\ccalQ_P$, as e.g., in \cite{kantaros2017sampling}, can be found in Section VI in \cite{kantaros2017Csampling}; see also Proposition 5.1 in \cite{kantaros2017Csampling}.


Given an NBA that corresponds to the assigned LTL specification, the first step in the proposed algorithm is to construct the
sets of feasible symbols $\Sigma_{q_B,q_B'}^{\text{feas}}$, $q_B,q_B'\in\ccalQ_B$. Note that the construction of the sets $\Sigma_{q_B,q_B'}^{\text{feas}}\subseteq \Sigma_{q_B,q_B'}$ can become computationally expensive as $N$ and $W$ increase since the size of $\mathcal{AP}=\cup_{i=1}^N\cup_{e=1}^W\{\pi_i^{r_j}\}$ also increases and so does the size of the alphabet $\Sigma=2^{\mathcal{AP}}$ and the size of $\Sigma_{q_B,q_B'}$. To mitigate this issue, instead of defining the set of atomic propositions as $\mathcal{AP}=\cup_{i=1}^N\cup_{e=1}^W\{\pi_i^{r_j}\}$, we define the set $\mathcal{AP}$ as $\mathcal{AP}=\{\cup_{e}^E\{\xi_e\}\cup\varnothing\}$, where $\varnothing$ stands for the empty word and $E$ denotes the total number of Boolean formulas $\xi_e$ in conjunctive or disjunctive normal form that appear in the LTL formula and are defined over the atomic propositions $\pi_i^{r_j}$. This allows us to construct smaller LTL formulas defined over $2^{\{\cup_{e}^E\{\xi_e\}\cup\varnothing\}}$. 
As a result, the computational cost of creating the NBA and the sets $\Sigma_{q_B,q_B'}^{\text{feas}}$ decreases. In the following case studies, the sets $\Sigma_{q_B,q_B'}$ are computed using the software package \cite{ltl2nbaBelta} that relies on \cite{gastin2001fast} for the construction of the NBA.


 
\begin{figure}[t]
\centering
  \includegraphics[width=1\linewidth]{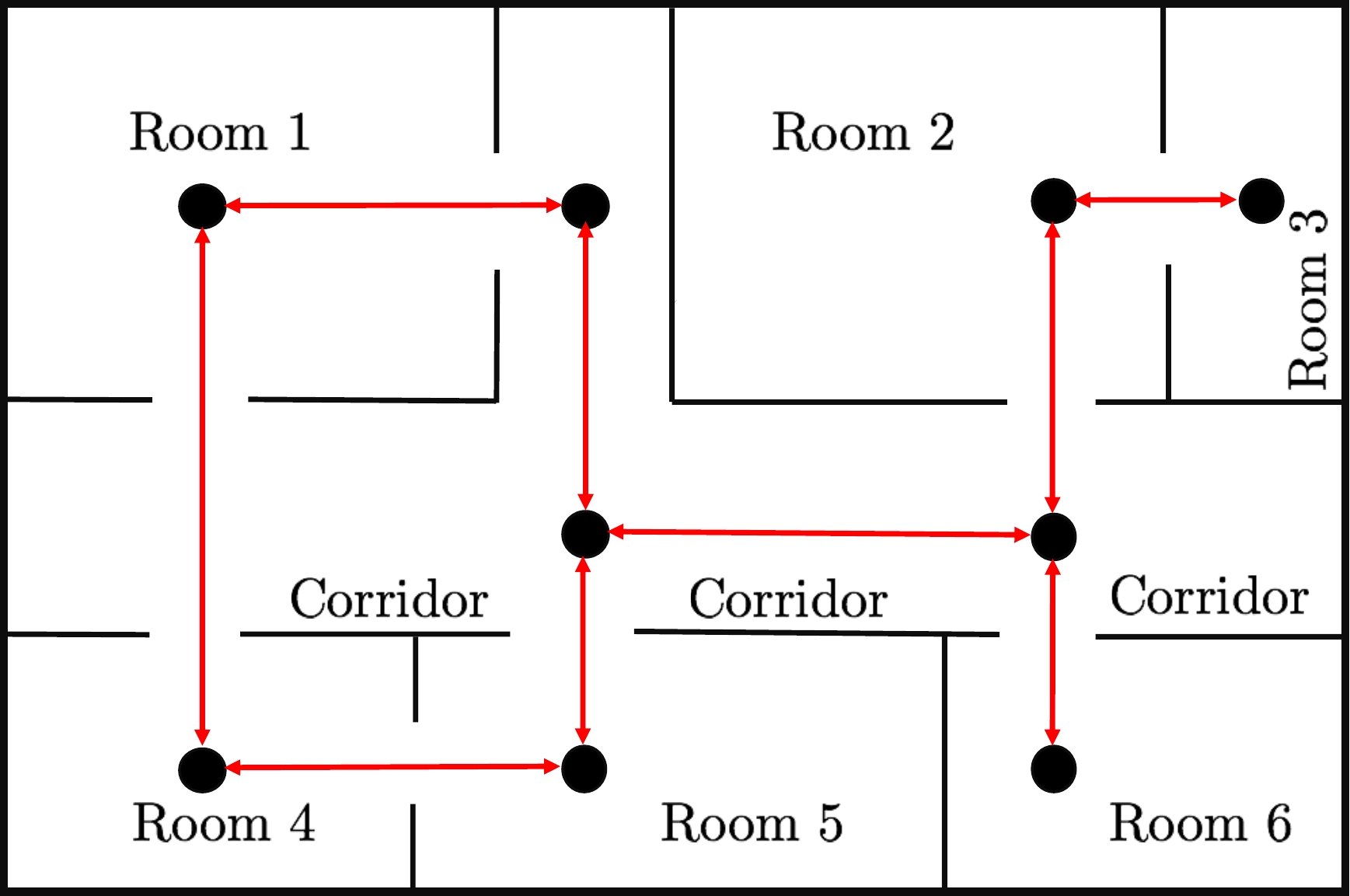}
  \caption{Graphical depiction of the wTS that abstracts robot mobility in an indoor environment for the case study $N=9$ and $|\ccalQ_i|=9$. Black disks stand for the states of wTS and red edges capture transitions among states. The wTS can be viewed as a graph with average degree per node equal to $3$. Self-loops are not shown.}
 \label{fig:constrwts}
\end{figure}

%
\subsection{Completeness, Optimality, and Scalability}\label{sec:cos}
\paragraph{Probabilistic Completeness \& Scalability}In what follows, we examine the performance of $\text{STyLuS}^{*}$ with respect to the number of robots and the size of the wTSs. The results for all considered planning problems are reported in Table \ref{tab:table1}. In Table \ref{tab:table1}, the first, second, and third column show the number $N$ of robots, the size of the state-space $\ccalQ_i$, which is the same for all wTSs, and the size of the state-space of the PBA, defined as $|\ccalQ_P|=|\ccalQ_B|\prod_{i=1}^N|\ccalQ_i|$, respectively. Viewing the wTSs as graphs, the average degree of the nodes of the examined wTSs with $|\ccalQ_i|=9$, $|\ccalQ_i|=10^2$, $|\ccalQ_i|=10^3$,  $|\ccalQ_i|=2.5\cdot 10^3$, and $|\ccalQ_i|=10^4$ is $3$, $12$, $30$, $20$, and $42$, respectively, for all robots $i$ and for each case study in Table \ref{tab:table1}. 
%
The wTSs with $|\ccalQ_i|=9$ is depicted in Figure \ref{fig:constrwts}. All other transition systems are generated randomly.
The fourth and fifth column of Table \ref{tab:table1} show the total number of iterations $n$ of Algorithm \ref{alg:tree} required to find the first prefix and suffix part and the corresponding size of the constructed trees, respectively. The sixth column shows  the time required until the first prefix and suffix part are detected, without executing the rewiring step. The runtime of NuSMV \cite{cimatti2002nusmv} or nuXmv \cite{cavada2014nuxmv} are reported in the last column. The `M', in the last column, means that NuSMV and nuXmv failed to build the model of the multi-robot system and, therefore, we could not proceed to the model-checking process. On the other hand, `F' means that the model was built but the LTL formula is too long, with thousands of logical operators and atomic propositions $\pi_{i}^{r_j}$, to input it to NuSMV or nuXmv in a user-friendly way. Particularly, neither of NuSMV and nuXmv allow to define the LTL specifications in the more compact and symbolic form, discussed before, using the Boolean formulas $\xi_e$. Instead, the LTL formulas have to be expressed explicitly using the atomic propositions $\pi_i^{r_j}$.

For all case studies shown in Table \ref{tab:table1}, we consider the following LTL task 
\begin{align}\label{eq:phi1}
\phi=&G(\xi_1\rightarrow(\bigcirc\neg\xi_1\ccalU \xi_2)) \wedge (\square\Diamond\xi_1) \wedge (\square\Diamond\xi_3) \wedge(\square\Diamond\xi_4) \wedge\nonumber\\& (\neg\xi_1 \ccalU \xi_5) \wedge  (\square\Diamond \xi_5) \wedge (\square\neg\xi_6) \wedge (\Diamond(\xi_7 \vee \xi_8)), 
\end{align}
The LTL formula \eqref{eq:phi1} is satisfied if (i) always when $\xi_1$ is true, then at the next step $\xi_1$ should be false until $\xi_2$ becomes true; (ii) $\xi_1$ is true infinitely often; (iii) $\xi_3$ is true infinitely often; (iv) $\xi_4$ is true infinitely often; (v) $\xi_1$ is false until $\xi_5$ becomes true; (vi) $\xi_6$ is always false; (vii) $\xi_5$ is true infinitely often; and (viii) eventually either $\xi_7$ or $\xi_8$ are true. Recall that, in \eqref{eq:phi1}, $\xi_e$ is a Boolean formula in conjunctive or disjunctive normal form that is true depending on what atomic propositions $\pi_i^{r_j}$ that are true. For instance, for the planning problem with order $10^{21}$ in Table \ref{tab:table1}, $\xi_1$ is defined as 
\begin{align}
\xi_1=&(\pi_1^{r_{100}})\wedge(\pi_8^{r_{20}}\vee\pi_8^{r_{90}}\vee\pi_8^{r_{11}})\nonumber\\&\wedge(\pi_9^{r_1}\vee\pi_9^{r_{10}}\vee\pi_9^{r_{25}}\vee\pi_9^{r_{35}}),\nonumber
\end{align}
which is true if (i) robot 1 is in region $r_{100}$; (ii) robot 8 is in one of the regions  $r_{20}$, $r_{90}$, and $r_{11}$; and (iii) robot 9 is in one of the regions  $r_{1}$, $r_{10}$, $r_{25}$, and $r_{35}$. All other Boolean formulas $\xi_e$ are defined similarly. Similar LTL formulas are widely considered in the relevant literature as they can capture a wide range of robot tasks; see e.g., \cite{guo2015multi}. For instance, the subformula $(\square\Diamond\xi_1) \wedge (\square\Diamond\xi_3) \wedge(\square\Diamond\xi_4) \wedge (\Diamond(\xi_7 \vee \xi_8))$ can capture surveillance and data gathering tasks whereby the robots need to visit certain locations of interest infinitely often to collect data \cite{guo2017distributed,smith2011optimal}. Also, it can capture intermittent connectivity tasks where the robots need to meet infinitely often at communication points to exchange their collected information  \cite{kantaros2016distributedInterm,kantaros2019temporal}. Specifically, in this task, the robots are divided into overlapping robot teams and all robots of each team need to meet infinitely often at communication points that are optimally selected to exchange  the information. Also, the term $(\square\neg\xi_6)$ can capture obstacle avoidance (or collision avoidance) constraints that require the robots to always avoid obstacles. Moreover, $G(\xi_1\rightarrow(\bigcirc\neg\xi_1\ccalU \xi_2))$ can be used to assign priorities to different subtasks. For instance, if $\xi_1$  requires the robots to visit food/water resources and  $\xi_2$ requires the robots to deliver those resources to the end users. Then, the formula $G(\xi_1\rightarrow(\bigcirc\neg\xi_1\ccalU \xi_2))$ requires the robots to never re-gather/visit the resources without first delivering supplies to the users. The subformula $\neg\xi_1 \ccalU \xi_5$ can be interpreted similarly.

The LTL formula \eqref{eq:phi1} corresponds to an NBA with $|\ccalQ_B|=21$, $|\ccalQ_B^0|=1$, $|\ccalQ_B^F|=2$, and $125$ transitions. Note that both final states of the NBA are feasible. Given the NBA, we construct the sets $\Sigma_{q_B,q_B'}^{\text{feas}}\subseteq 2^{\mathcal{AP}}$ in 0.1 seconds approximately for all case studies in Table \ref{tab:table1}. 
To speed up the detection of the first feasible plan, we do not execute the rewiring step until the first final state and a loop around it are detected. Due to this modification, $\text{STyLuS}^{*}$ can only generate feasible plans; see the proof of~Theorem~\ref{thm:convOpt}.
%
 
Observe in Table \ref{tab:table1}, that $\text{STyLuS}^{*}$, NuSMV, and nuXmv have comparable performances when both a small number of robots and small enough wTSs are considered; see e.g., the planning problems with order $10^3$, $10^{10}$, and $10^{21}$. However, as either the number of robots or the size of the wTSs increases, $\text{STyLuS}^{*}$ outperforms both NuSMV and nuXmv in terms of both runtime and the size of problems that it can handle. For instance, both NuSMV and nuXmv fail to build the model when $N=1$ and $|\ccalQ_i|=10^4$, with average degree of the wTS equal to $42$. As a result, it was impossible to synthesize a feasible motion plan with NuSMV or nuXmv using either symbolic techniques (Binary Decision Diagrams) or bounded model checking options. In fact, for this single-robot planning problem, NuSMV could build the model only for much sparser wTSs with average degree equal to $4$. In this case, the model was built in 15 minutes and the verification process finished in almost 1 second. Moreover, SPIN, an explicit state model checker, also failed to build this model due to excessive memory requirements \cite{holzmann2004spin}; comparisons and trade-offs between NuSMV an SPIN can be found in \cite{fraser2009evaluation,choi2007nusmv}. 
On the other hand, for such sparse wTSs, $\text{STyLuS}^{*}$ synthesized a feasible prefix and suffix part in 16 and 14 seconds, respectively.
%
%
%

%
%
Moreover, observe in Table \ref{tab:table1} that $\text{STyLuS}^{*}$ can synthesize feasible motion plans for large-scale multi-robot systems quite fast. For instance, it can synthesize in almost 35 minutes a feasible team plan $\tau$ for a team of $N=200$ robots where each robot is modeled as a transition system with $|\ccalQ_i|=10^4$ states and the PBA has $|\ccalQ_P|=21\times(10^4)^{200}\approx10^{801}$ states. Finally, note that the runtimes reported in Table \ref{tab:table1} can be significantly improved if $\text{STyLuS}^{*}$ is executed in parallel, as in \cite{kantaros2017Dsampling}, and if the shortest paths required by each robot in the sampling step are constructed simultaneously (and not sequentially) across the robots, as also discussed in Section \ref{sec:complexity}. 
Note also that the LTL formulas for the problems that involve more than $100$ robots include thousands of logical operators and atomic propositions $\pi_i^{r_j}$ and, therefore, they are too long to input them to NuSMV and nuXmv in a user-friendly way. 
Also, recall that NuSMV and nuXmv can only generate feasible, and not optimal, plans; see also Section VI in \cite{kantaros2017Csampling} for relevant comparative simulation studies.

\begin{figure}[t]
  \centering
     \subfigure[]{
    \label{fig:iterspars1}
  \includegraphics[width=0.48\linewidth]{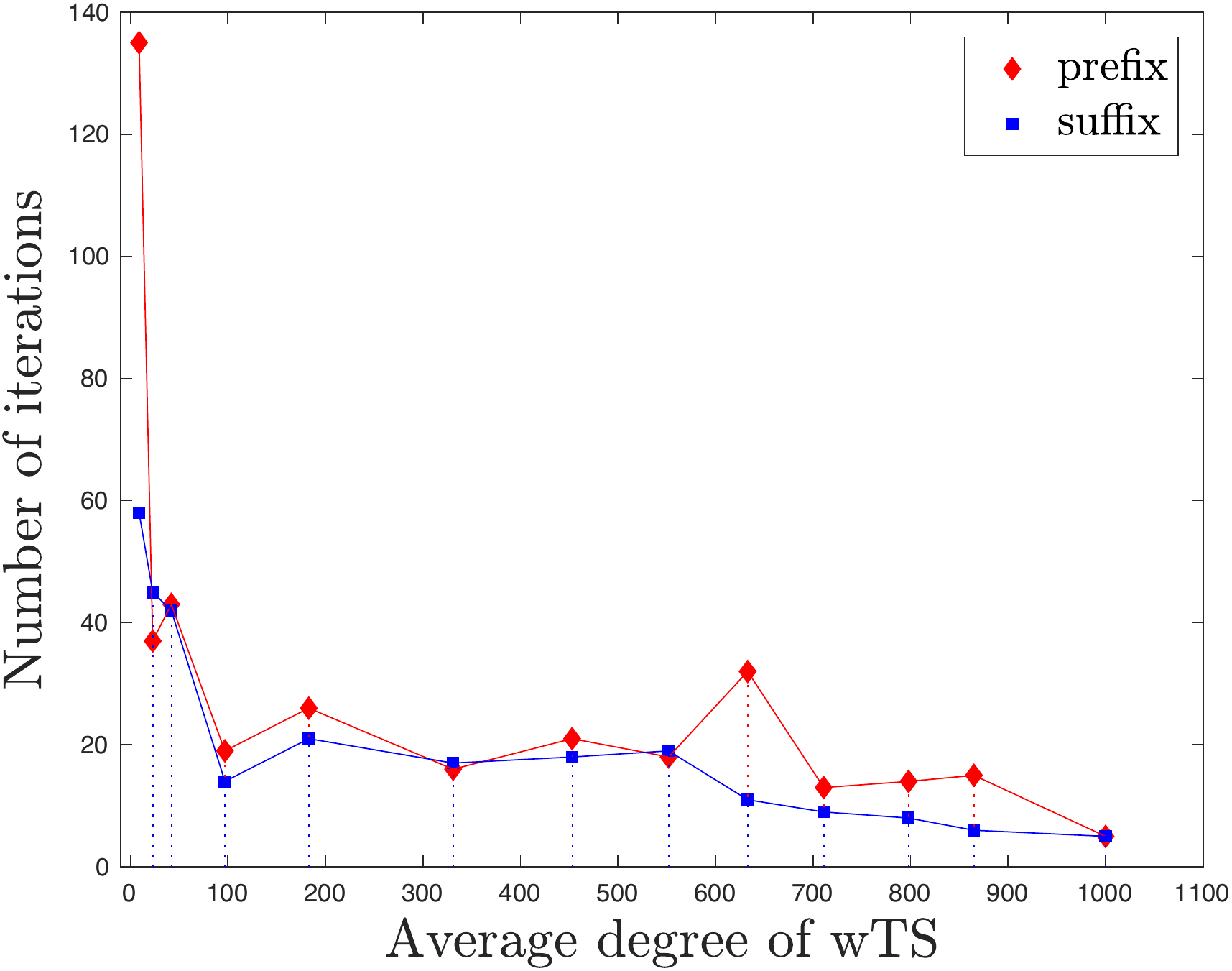}}
  \subfigure[]{
    \label{fig:runtimespars}
  \includegraphics[width=0.48\linewidth]{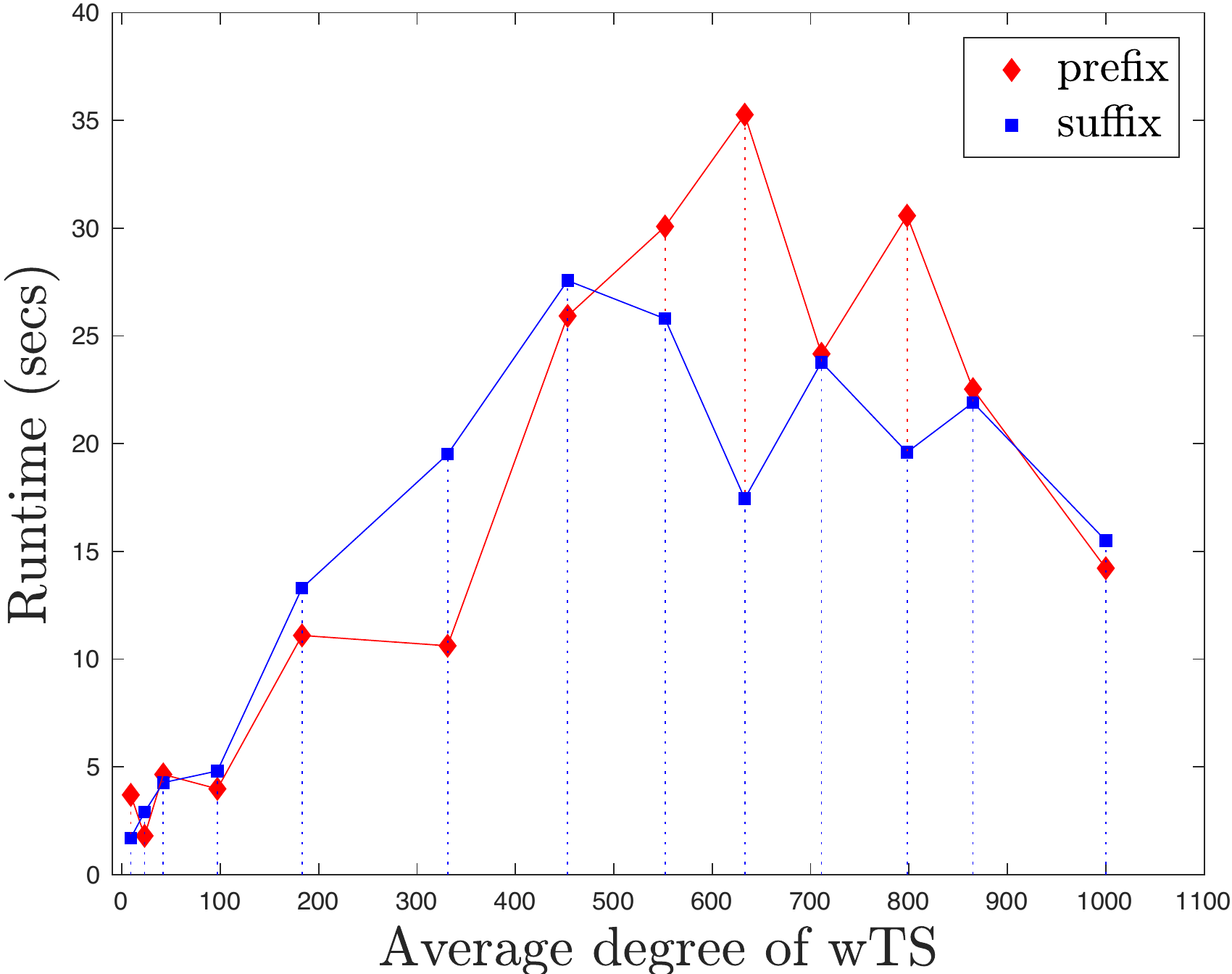}}
  \caption{Case Study I [$N=10$ robots, $|\ccalQ_i|=1000$ states]: Figures \ref{fig:iterspars1} and \ref{fig:runtimespars} show the required number of iterations and total runtime required to construct the first feasible prefix-suffix plan, respectively.}
  \label{fig:sparsity}
\end{figure}

\begin{figure}[t]
  \centering
     \subfigure[]{
    \label{fig:iterspars2}
  \includegraphics[width=0.48\linewidth]{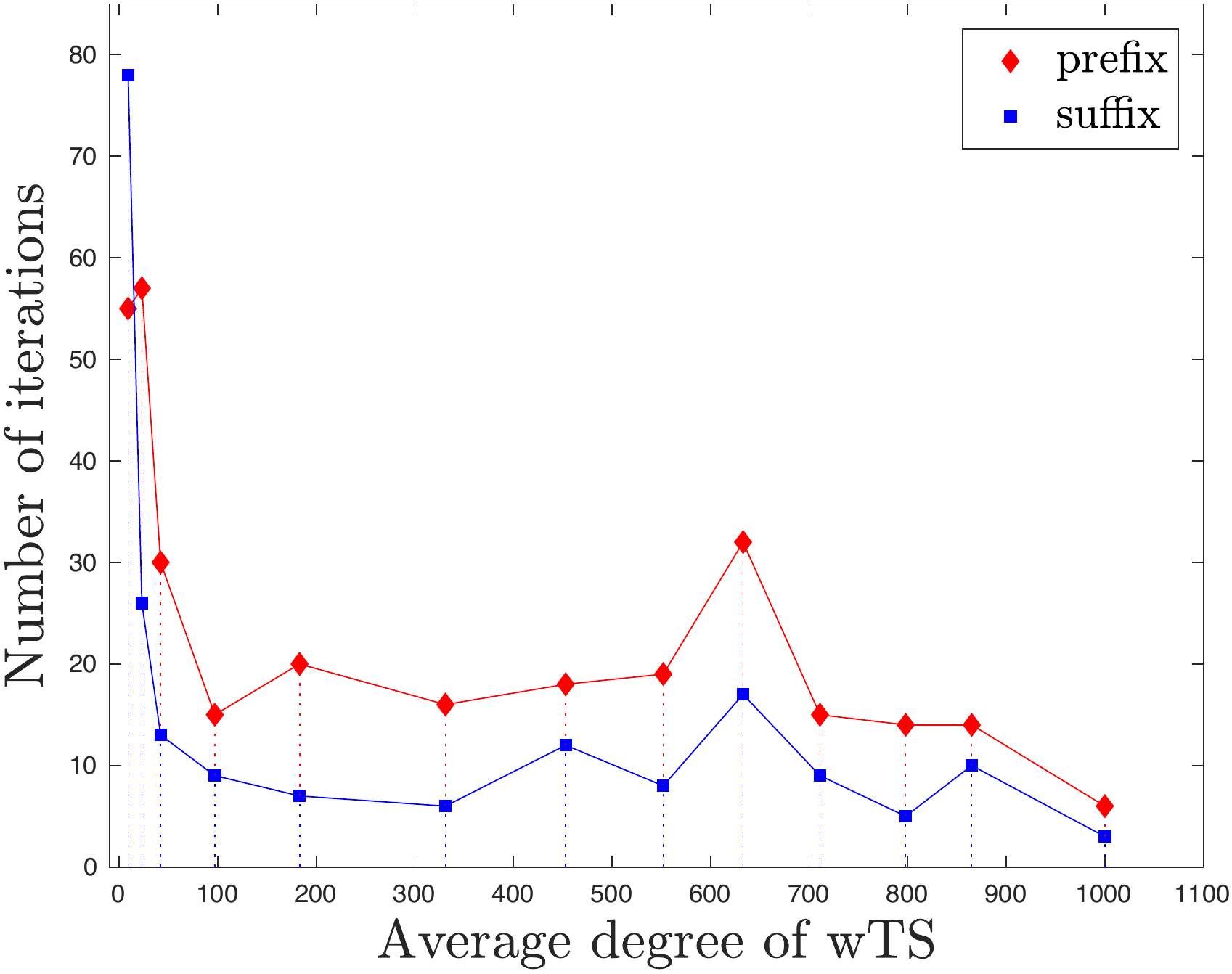}}
  \subfigure[]{
    \label{fig:runtimespars2}
  \includegraphics[width=0.48\linewidth]{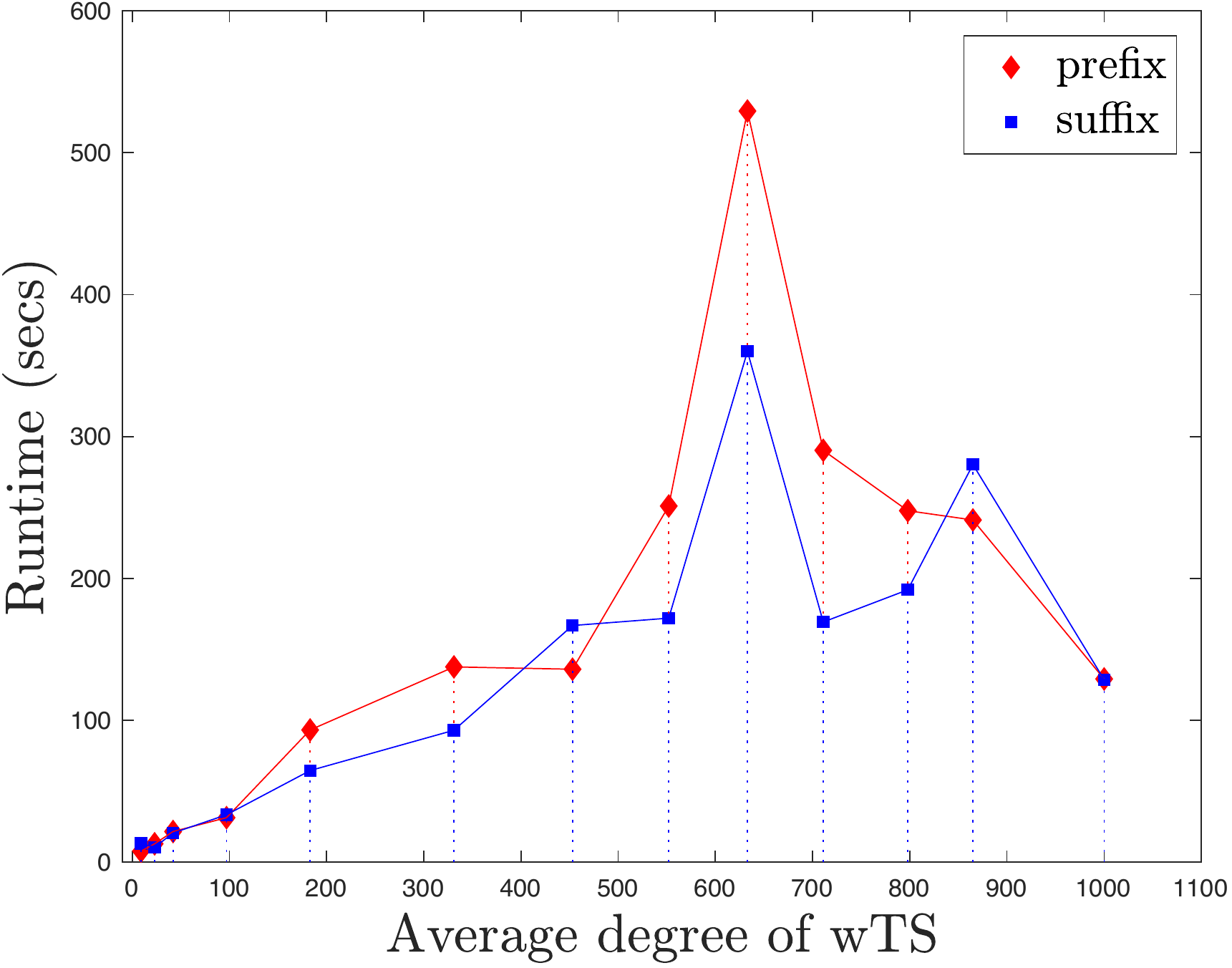}}
  \caption{Case Study I [$N=100$ robots, $|\ccalQ_i|=1000$ states]: Figures \ref{fig:iterspars2} and \ref{fig:runtimespars2} show the required number of iterations and total runtime required to construct the first feasible prefix-suffix plan, respectively. }
  \label{fig:sparsity2}
\end{figure} 

In Figures \ref{fig:sparsity} and \ref{fig:sparsity2}, we examine the performance of $\text{STyLuS}^{*}$ with respect to the sparsity of the transition systems $\text{wTS}_i$, viewing them as graphs, for the planning problems of Table \ref{tab:table1} with order $10^{31}$ and $10^{300}$, respectively. Specifically, in Figure \ref{fig:iterspars1}, we observe that as the wTSs become denser, the required number of iterations to find the first feasible prefix-suffix plan decreases, since the length of the shortest paths in the wTSs, that are used in the sampling process, decreases. We also observe in Figure \ref{fig:runtimespars}  that as the wTSs become denser, the total runtime to find the first feasible prefix-suffix plan initially increases and eventually starts to decrease. The reason is that as the wTSs become denser, the required shortest paths are more difficult to find and, therefore, the computational cost per iteration $n$ of generating the states $q_P^{\text{new}}$ increases; see also Section \ref{sec:complexity}. However, as the wTSs becomes fully connected the total runtime starts to decrease, since only few iterations are required to find a feasible plan, as discussed before. The same observations apply to Figure \ref{fig:sparsity2}, as well, that pertains to the planning problem with order $10^{301}$.

\paragraph{Switching between sampling densities} 
Since optimality depends on the discovery of prefix plans to all feasible final states, bias can be used sequentially to every feasible final state $q_B^{\text{F,feas}}$ to discover such plans, as also discussed in Remark \ref{rem:biasPre}.
Once all feasible final states $q_B^{\text{F,feas}}$ have been detected, or after a pre-determined number of iterations $n$, the rewiring step is activated and uniform sampling is used to better explore the PBA towards all directions in $\ccalQ_P$. Instead, biased sampling favors exploration towards the shortest paths that lead to the final states. 

Note also that the proposed method finds initial feasible plans with small enough cost $J(\tau)$, even though the rewiring step is not activated, due to the biased sampling that favors wTS transitions that correspond to the shortest NBA paths towards $q_B^{\text{F,feas}}$. For instance, observe in Figure \ref{fig:costComp}, which concerns the planning problem with order $10^{10}$, that the cost of the best prefix part decreased from $415.48$ to $314.58$ meters within 5.5 hours. In total, $156$ final states were detected and the first one was found after $0.42$ seconds. On the other hand, using uniform sampling distributions from the beginning, the cost of the best prefix part decreased from $823.96$ to $706.18$ meters within 5.5 hours. In this case, 4 final states were detected and the first one was found after $1.34$ hours.

\paragraph{Asymptotic Optimality} 
Next, we validate the asymptotic optimality of of $\text{STyLuS}^{*}$ discussed in Theorem \ref{thm:alg1}.
Specifically, we consider a team of $N=2$ robots modeled by wTSs with $|\mathcal{Q}_i|=16$ states and $70$ transitions, including self-loops around each state; see Figure \ref{fig:optimWTS}. The assigned task is expressed in the following temporal logic formula. 

\begin{figure}[t]
\centering
  \includegraphics[width=0.8\linewidth]{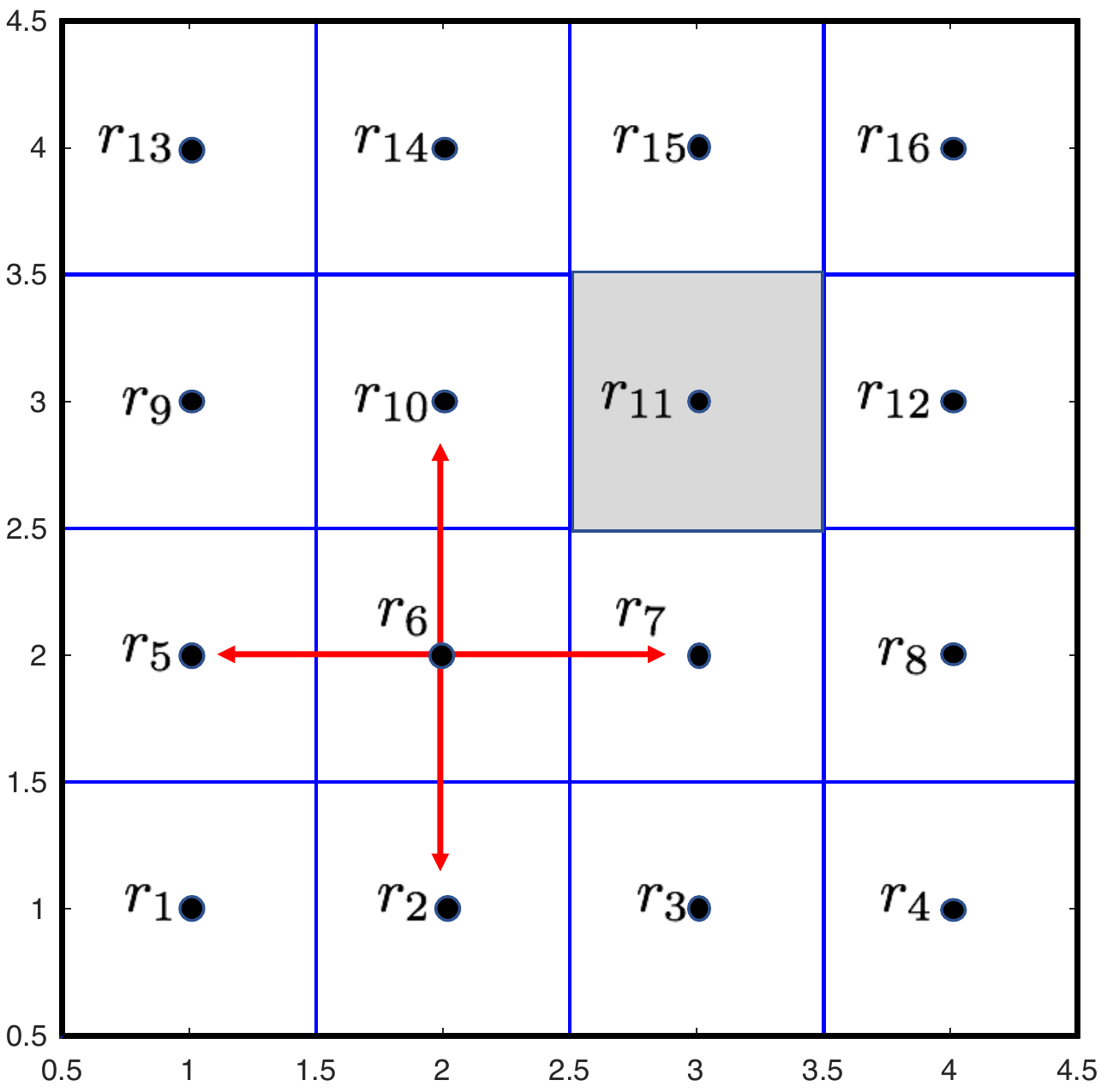}
  \caption{Graphical depiction of the wTS considered for the task defined in \eqref{eq:task2}. The gray cell corresponds to an obstacle that should always be avoided. In each state, a root can take four actions as illustrated by the arrows.}
 \label{fig:optimWTS}
\end{figure}

\begin{align}\label{eq:task2}
\phi&= \square\Diamond(\pi_1^{r_6}\wedge\Diamond(\pi_2^{r_{14}}))\wedge\square(\neg\pi_1^{r_9})\wedge\square (\neg\pi_1^{r_{11}})\wedge\square (\neg\pi_2^{r_{11}})\nonumber\\&\wedge\square(\pi_2^{r_{14}}\rightarrow \bigcirc(\neg\pi_2^{r_{14}}\ccalU\pi_1^{r_{4}}))\wedge (\Diamond\pi_2^{r_{12}}) \wedge (\square\Diamond\pi_2^{r_{10}})
\end{align}
In words, this LTL-based task requires (a) robot 1 to visit location $r_6$, (b) once (a) is true robot 2 to visit location $r_{14}$, (c) conditions (a) and (b) to occur infinitely often, (d) robot 1 to always avoid location $r_9$, (e) once robot 2 visits location $r_{14}$, it should avoid this area until robot 1 visits location $r_{4}$, (f) robot 2 to visit location $r_{12}$ eventually, (g) robot 2 to visit location $r_{10}$ infinitely often and (h) both robots should always avoid the obstacle in region $r_9$. The NBA that corresponds to the LTL specification \eqref{eq:task2} has $|\mathcal{Q}_B|=24$ states with $|\ccalQ_B^0|=1$, $|\ccalQ_B^F|=4$, and $163$ transitions. Specifically, the set of final states is defined as $\ccalB_F=\{6,9,11,13\}$ and all of them are feasible. Also, the state space of the corresponding PBA consists of $\Pi_{i=1}^N|\mathcal{Q}_i||\mathcal{Q}_B|=6,144$ states, which is small enough so that the existing optimal control synthesis methods discussed in Section \ref{sec:prelim} can be used to find the optimal plan.  Figure \ref{fig:asOpt} shows the evolution of the cost of the best prefix part constructed by $\text{STyLuS}^{*}$ with respect to iterations $n$ and time. Observe in Figure \ref{fig:asOpt2} that as $n$ increases, $\text{STyLuS}^{*}$ finds the optimal prefix part which has cost $5$ meters. The required runtime to find the optimal prefix part is illustrated in Figure \ref{fig:asOpt1}. On the other hand, the cost of the prefix part constructed by NuSMV and nuXmv is  $8.8824$ meters.
\begin{figure}[t]
\centering
  \subfigure[]{
    \label{fig:asOpt2}
  \includegraphics[width=0.48\linewidth]{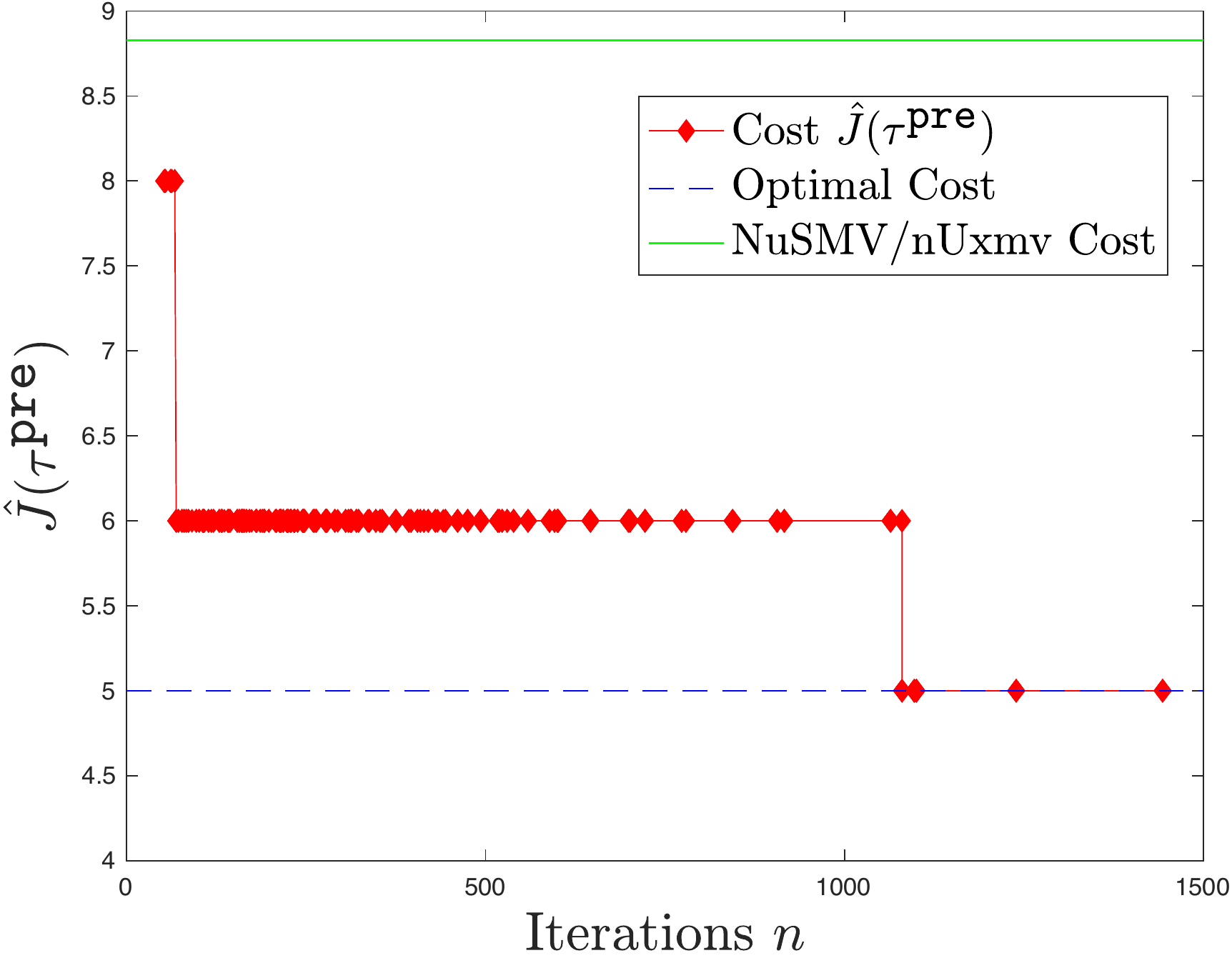}}
     \subfigure[]{
    \label{fig:asOpt1}
  \includegraphics[width=0.48\linewidth]{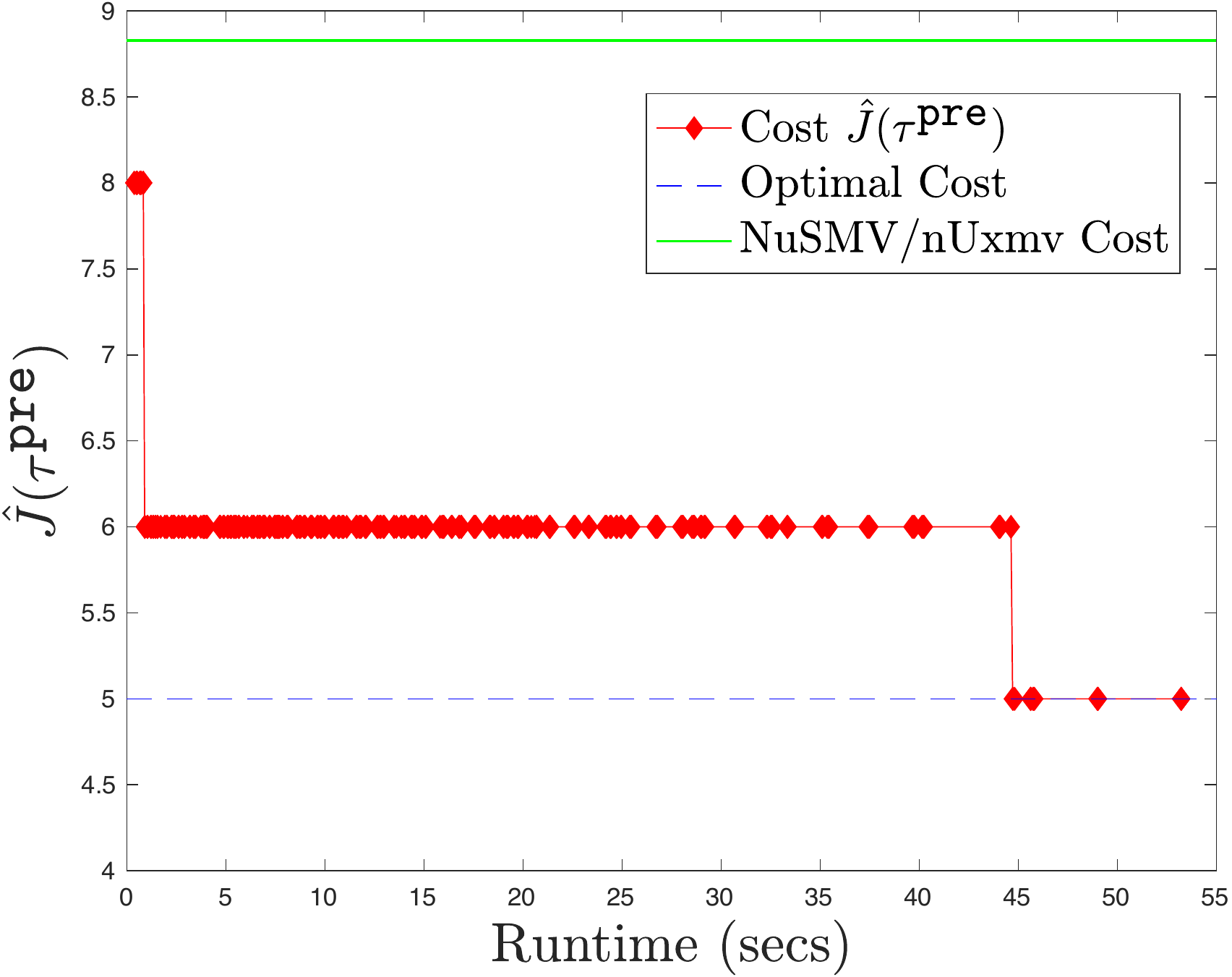}}
  \caption{Case Study III: Figures \ref{fig:asOpt2} and \ref{fig:asOpt1} show the evolution the cost of the best prefix part constructed by $\text{STyLuS}^{*}$ with respect to iterations $n$ and runtime, respectively. Red diamonds denote a new final state detected by $\text{STyLuS}^{*}$. To generate the results of Figure \ref{fig:asOpt} the rewiring step is activated.}
  \label{fig:asOpt}
  \end{figure}

\begin{figure}[t]
  \centering
    \label{fig:iterspars}
  \includegraphics[width=0.8\linewidth]{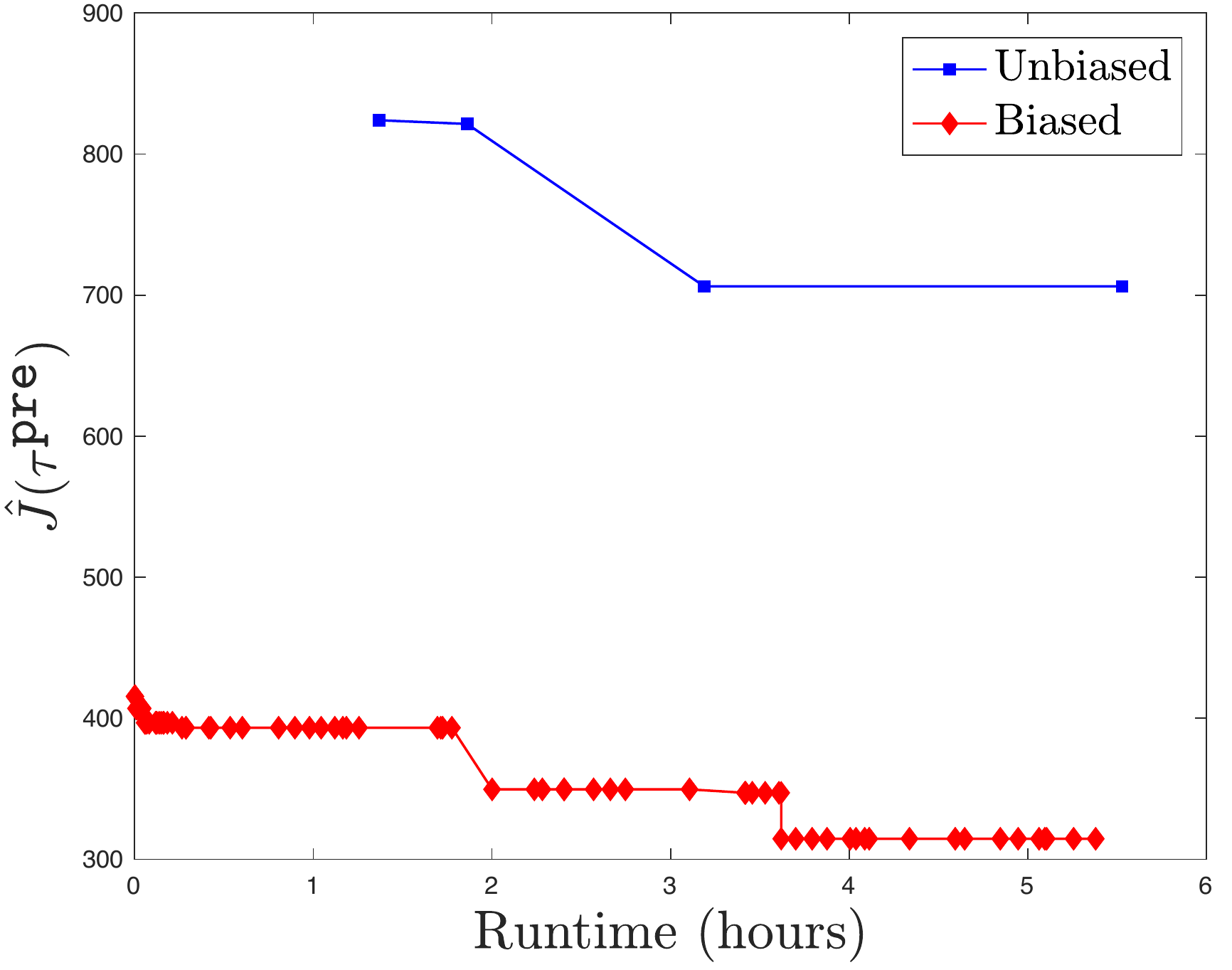}
  \caption{Case Study I [$N=9$ robots, $|\ccalQ_i|=9$ states]: Comparison of the cost of the best prefix part constructed when biased (red line) and unbiased (blue line) from the beginning of $\text{STyLuS}^{*}$ are employed with respect to time. The cost of the best prefix part is reported every time a new final state is detected. Red diamonds and blue squares denote a new final state detected by $\text{STyLuS}^{*}$.}
  \label{fig:costComp}
\end{figure}

\subsection{Scalability for larger and denser NBA}  \label{sec:largeNBA}

\begin{table*}[t]
\caption{\footnotesize{Feasibility and scalability analysis: $|\ccalQ_B|=59$}}\label{tab:table2}
\begin{center}
\begin{tabular}{c|c|c|c|c|c|c|c|}
\cline{1-7}
		\multicolumn{1}{|c|}{$N$}		&
						  $|\ccalQ_i|$ & $|\ccalQ_{\text{P}}|$ & $n_{\text{Pre1}} + n_{\text{Suf1}}$ & $|\ccalV_T^{\text{Pre1}}| + |\ccalV_T^{\text{Suf1}}|$ & \texttt{Pre1+Suf1}  & NuSMV/nuXmv \\ \hline		
	\multicolumn{1}{|c|}{$1$}		& $100$ &  $10^{3}$ & 54 + 92 & 533 + 274 & 2.18 + 1.55 (secs)   & $<1$ sec\\ \hline			
	\multicolumn{1}{|c|}{$1$}		& $1000$ & $10^{3}$ & 78 +51  & 326 + 252 & 1.84 + 1.37 (secs) & $<1$ sec \\ \hline				
	\multicolumn{1}{|c|}{$1$}		& $10000$ & $10^{4}$  &150 + 107 & 769 + 364& 19.2 + 11.2 (secs) & M/M \\ \hline
	\multicolumn{1}{|c|}{$9$}		& $9$ &  $10^{10}$& 93 + 27 & 400 + 168 &20.7 + 18.9 (secs) & $<1$ sec\\ \hline
	\multicolumn{1}{|c|}{$10$}		& $100$ & $10^{21}$  &51+ 39 & 650 + 239 & 2.1 + 0.74 (secs) & $\approx$ 3/2 secs \\ \hline			
	\multicolumn{1}{|c|}{$10$}		& $1000$ & $10^{31}$  & 36 + 154 & 450 + 404 & 3.9 + 6.1 (secs) & $\approx$ 80/65 secs \\ \hline
		\multicolumn{1}{|c|}{$10$}		& $2500$ & $10^{35}$  & 61 + 98 & 710 + 516 &10.4 + 11.9 (secs) & M/$\approx 1920$ secs \\ \hline
	\multicolumn{1}{|c|}{$10$}		& $10000$ & $10^{41}$  & 47 + 164  & 722+604 &56.6 + 98.1(secs) & M/M \\ \hline
	\multicolumn{1}{|c|}{$100$}	&  $100$ & $10^{200}$  &21 + 117 & 154 + 1431 & 1.6 + 18.5 (secs) & F/F \\ \hline
	\multicolumn{1}{|c|}{$100$}	& $1000$  & $10^{300}$ &52 + 74 & 401 + 856 &19.8 + 53.32 (secs) & M/M \\ \hline
     \multicolumn{1}{|c|}{$100$}  & $10000$  & $10^{400}$  & 39 + 89 & 398 + 1621 &306.51 + 1698.12 (secs) & M/M\\  \hline
     \multicolumn{1}{|c|}{$150$}  & $10000$  & $10^{600}$  & 39 + 112 & 526+1864 & 498.12 + 3606.61 (secs) & M/M\\  \hline
     \multicolumn{1}{|c|}{$200$}  & $10000$  & $10^{800}$   & 48 + 103 & 588 + 1926 &702.11 + 3954.21 (secs) & M/M\\  \hline
\end{tabular}
\end{center}
\end{table*}


In this case study, we examine the performance of $\text{STyLuS}^{*}$ with respect to the number of the robots and the size of the wTSs, as in the previous case study, but also for a larger and denser NBA. The results are reported in Table \ref{tab:table2} which has the same structure as Table \ref{tab:table1}. For all case studies shown in Table \ref{tab:table2}, we consider the following LTL task 
\begin{align}\label{eq:phi1}
\phi=&G(\xi_1\rightarrow(\bigcirc\neg\xi_1\ccalU \xi_2)) \wedge (\square\Diamond\xi_1) \wedge (\square\Diamond\xi_3) \wedge(\square\Diamond\xi_4) \wedge\nonumber\\& (\neg\xi_1 \ccalU \xi_5) \wedge (\square\neg\xi_6) \wedge (\square\Diamond(\xi_7 \wedge (\Diamond\xi_8\wedge(\Diamond\xi_5)))), 
\end{align}
The LTL formula \eqref{eq:phi1} is satisfied if (i) always when $\xi_1$ is true, then at the next step $\xi_1$ should be false until $\xi_2$ becomes true; (ii) $\xi_1$ is true infinitely often; (iii) $\xi_3$ is true infinitely often; (iv) $\xi_4$ is true infinitely often; (v) $\xi_1$ is false until $\xi_5$ becomes true; (vi) $\xi_6$ is always false; and (vii) $\xi_7, \xi_8$ and $\xi_5$ are true in this order infinitely often. Also, the LTL formula \eqref{eq:phi1} corresponds to an NBA with $|\ccalQ_B|=59$, $|\ccalQ_B^0|=1$, $|\ccalQ_B^F|=8$, among which 6 final states are feasible, and $884$ transitions. Given the NBA, we construct the sets $\Sigma_{q_B,q_B'}^{\text{feas}}\subseteq 2^{\mathcal{AP}}$ in $0.14$ seconds approximately for all case studies in Table \ref{tab:table2}. 

Observe in Table \ref{tab:table2} that the total number of iterations and the total runtime to detect the first feasible plan have increased compared to the corresponding planning problems in Table \ref{tab:table1}, due to the larger size of the NBA. Moreover, observe that both NuSMV and nuXmv can find the first feasible plan faster than $\text{STyLuS}^{*}$ when planning problems with few robots and small wTSs are considered; see, e.g., the problems with order $10^3$, $10^{10}$, $10^{21}$ in Table \ref{tab:table2}. However, similar to the previous case study, when the number $N$ of robots or the size of the wTSs increases, $\text{STyLuS}^{*}$ outperforms both model checkers both in terms of runtime and  the size of the state-space that they can handle.

 \begin{rem}[Self-loops in wTSs]
Note that all wTSs considered in Tables \ref{tab:table1}-\ref{tab:table2} are constructed so that there are self-loops around each state, modeling in this way waiting actions of robots at all regions $r_j$. If self-loops are removed, the number of required iterations and respective runtime increase, as then the length of the shortest loop around every state increases (from its previous value of $1$). For instance, for the planning problem with order $10^{21}$ in Table \ref{tab:table1}, after removing all self-loops, the total number of iterations required to detect the first feasible prefix and suffix part have increased from $n_{\text{Pre1}}=33$ and $n_{\text{Suf1}}=33$ to $n_{\text{Pre1}}=533$ and $n_{\text{Suf1}}=991$, respectively. Similarly, the corresponding runtimes have increased from $0.7$ and $0.4$ seconds to $7.1$ and $10.6$ seconds. On the other hand, NuSMV and nuXmv can synthesize feasible motion plans in almost $3$ seconds. Adding or removing self-loops
does not seem to affect much the runtime/scalability of NuSMV and nuXmv.
\end{rem}
%

\begin{figure*}[t]
\centering
  \subfigure[$q_P^K=(q_\text{PTS}^K,q_B^K)=((13,10),6),~p_1,~K=13$]{
    \label{fig:proba1}
  \includegraphics[width=0.33\linewidth]{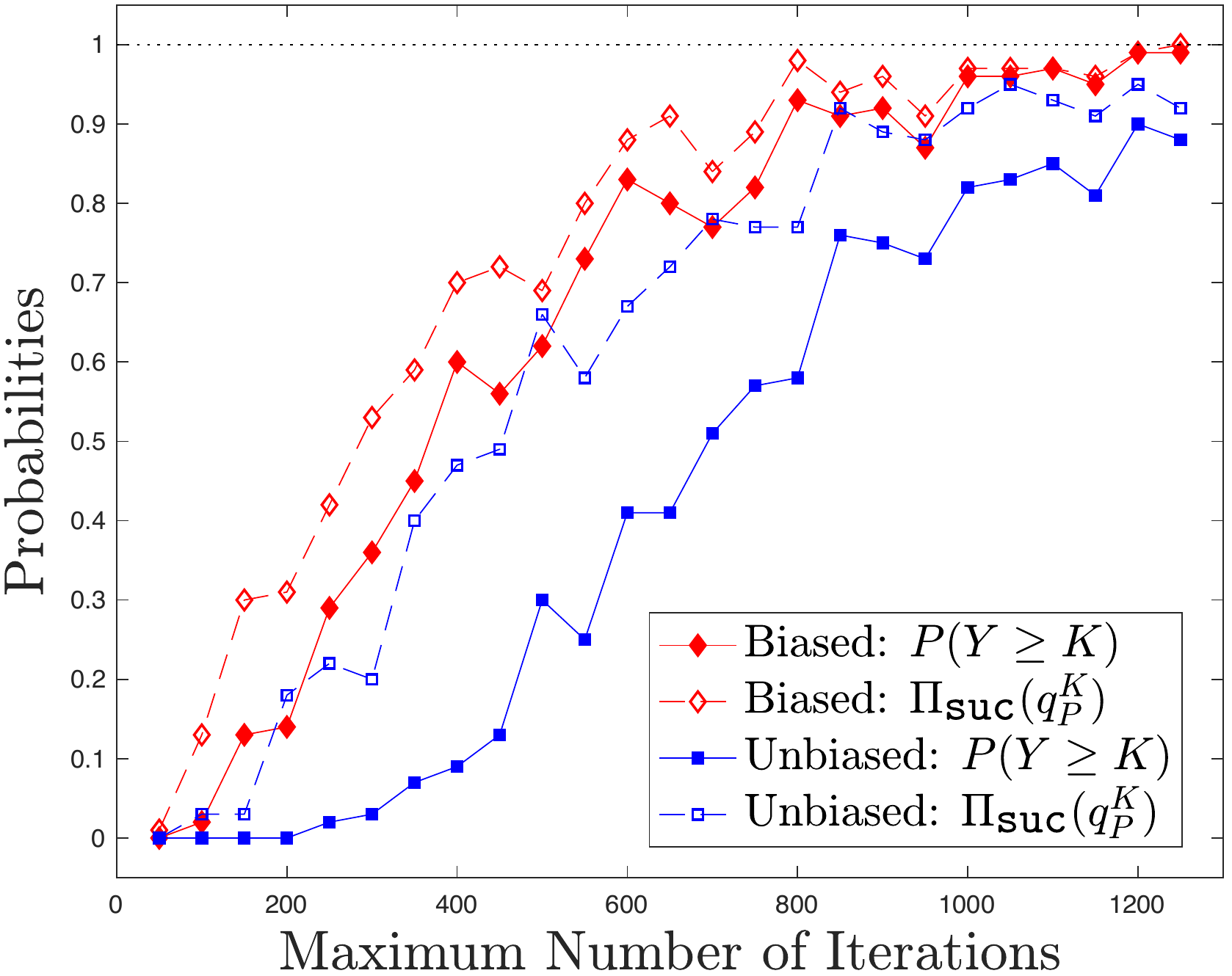}}
     \subfigure[$q_P^K=(q_\text{PTS}^K,q_B^K)=((6,10),9),~p_2, K=6$]{
    \label{fig:proba2}
  \includegraphics[width=0.33\linewidth]{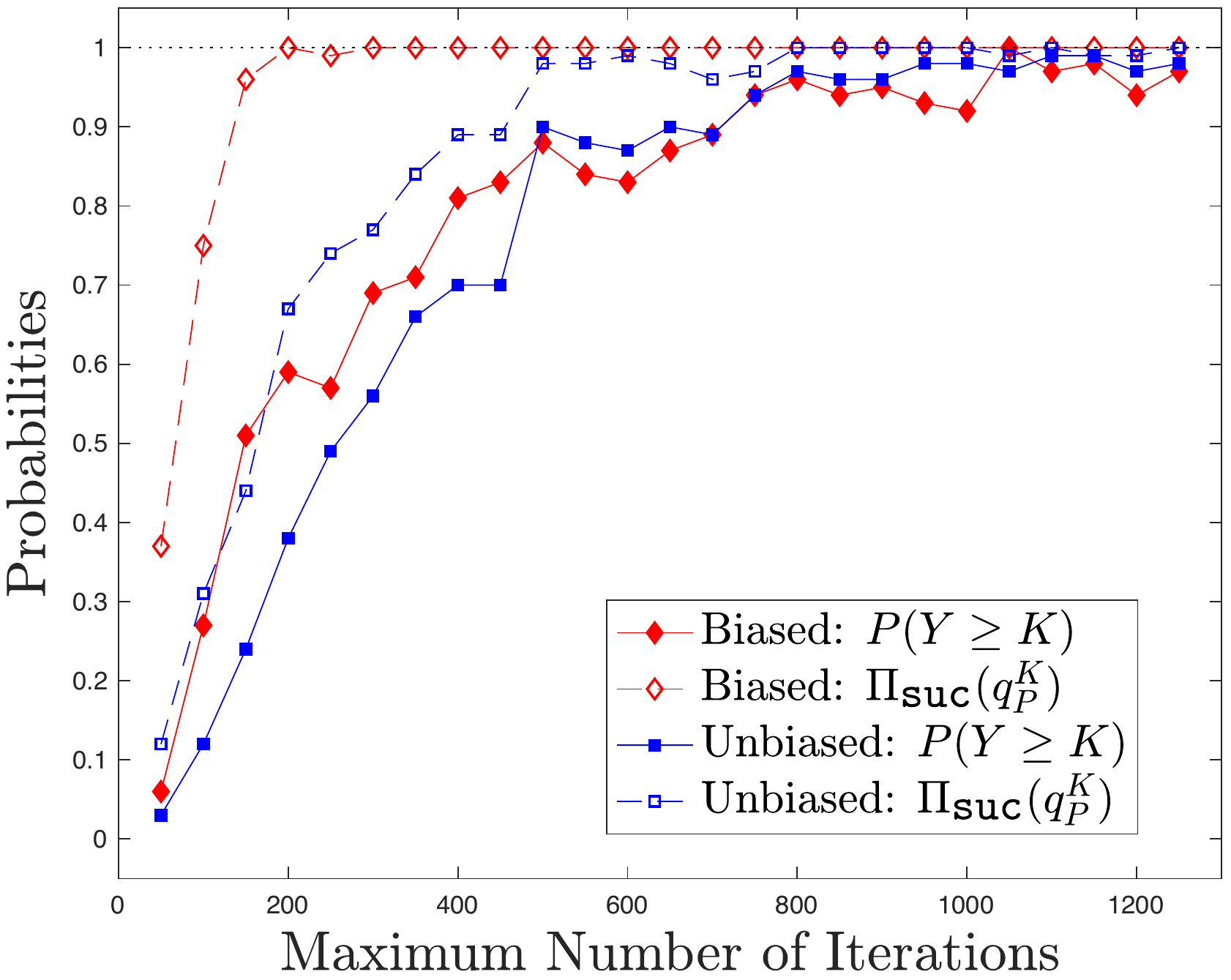}}
       \subfigure[$q_P^K=(q_\text{PTS}^K,q_B^K)=((13,15),13),~p_3, K=7$]{
    \label{fig:proba3}
  \includegraphics[width=0.33\linewidth]{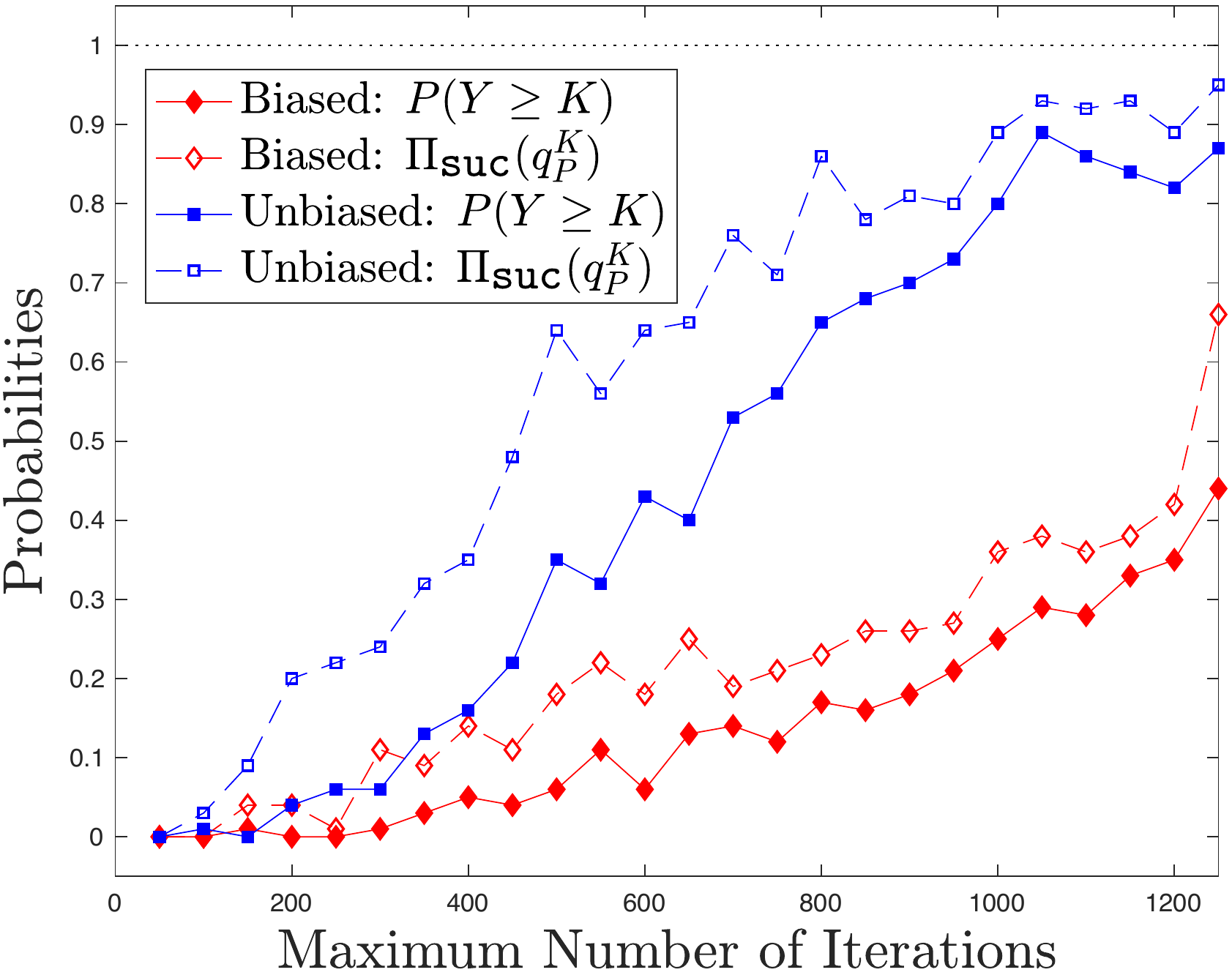}}\\
  \subfigure[$q_P^K=(q_\text{PTS}^K,q_B^K)=((13,10),6),~p_1, K=13$]{
    \label{fig:probb1}
  \includegraphics[width=0.33\linewidth]{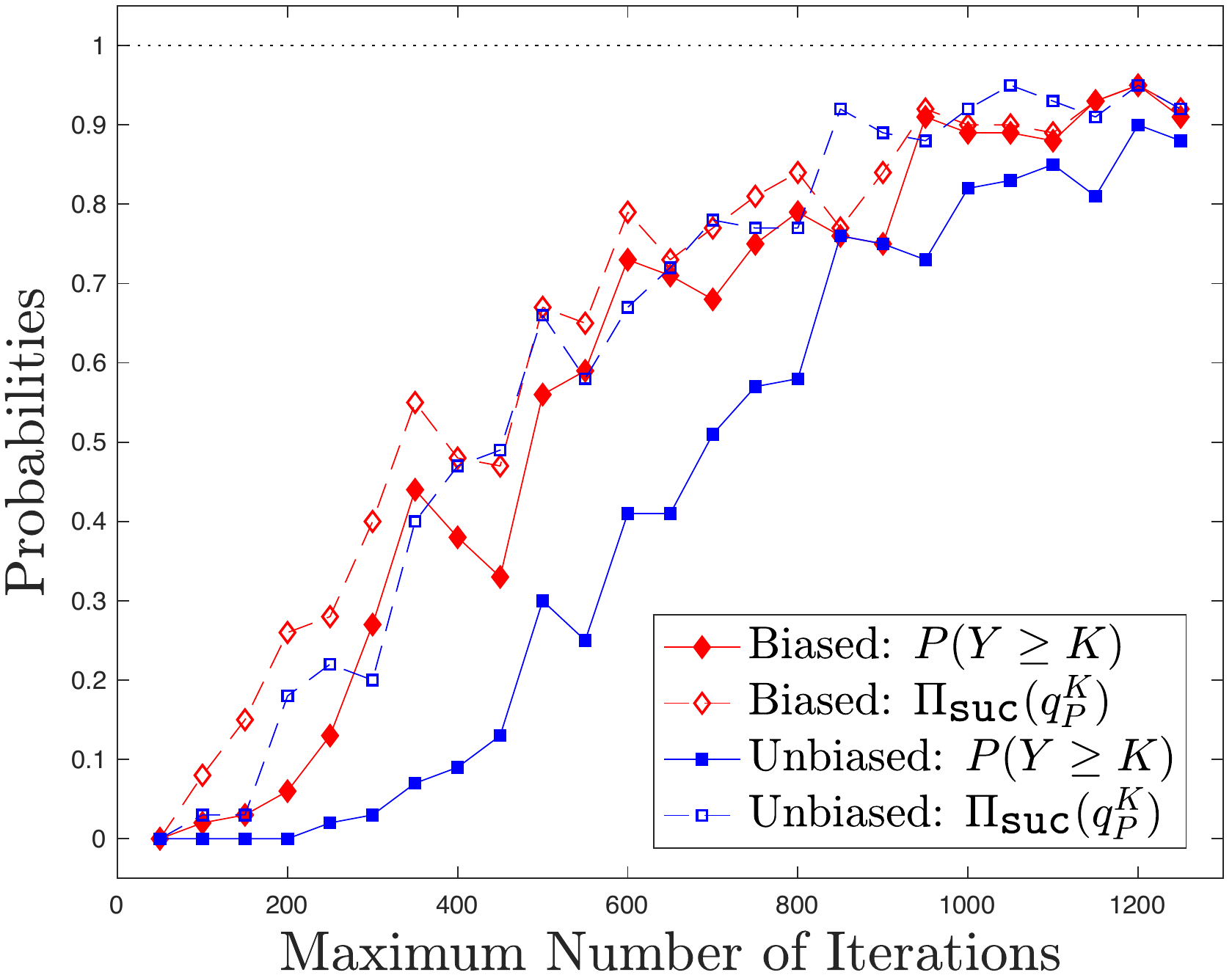}}
     \subfigure[$q_P^K=(q_\text{PTS}^K,q_B^K)=((6,10),9),~p_2, K=6$]{
    \label{fig:probb2}
  \includegraphics[width=0.33\linewidth]{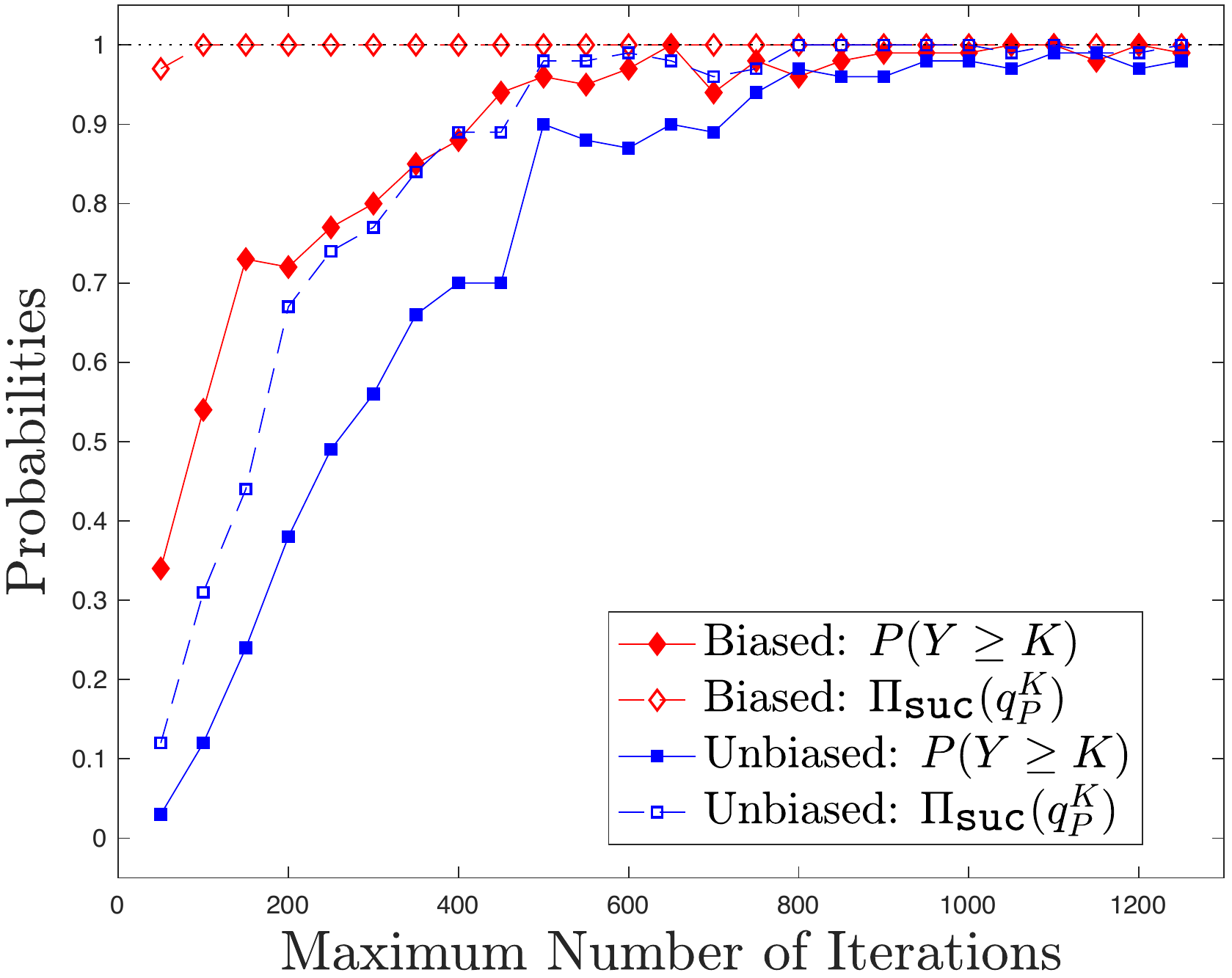}}
       \subfigure[$q_P^K=(q_\text{PTS}^K,q_B^K)=((13,15),13),~p_3, K=7$]{
    \label{fig:probb3}
  \includegraphics[width=0.33\linewidth]{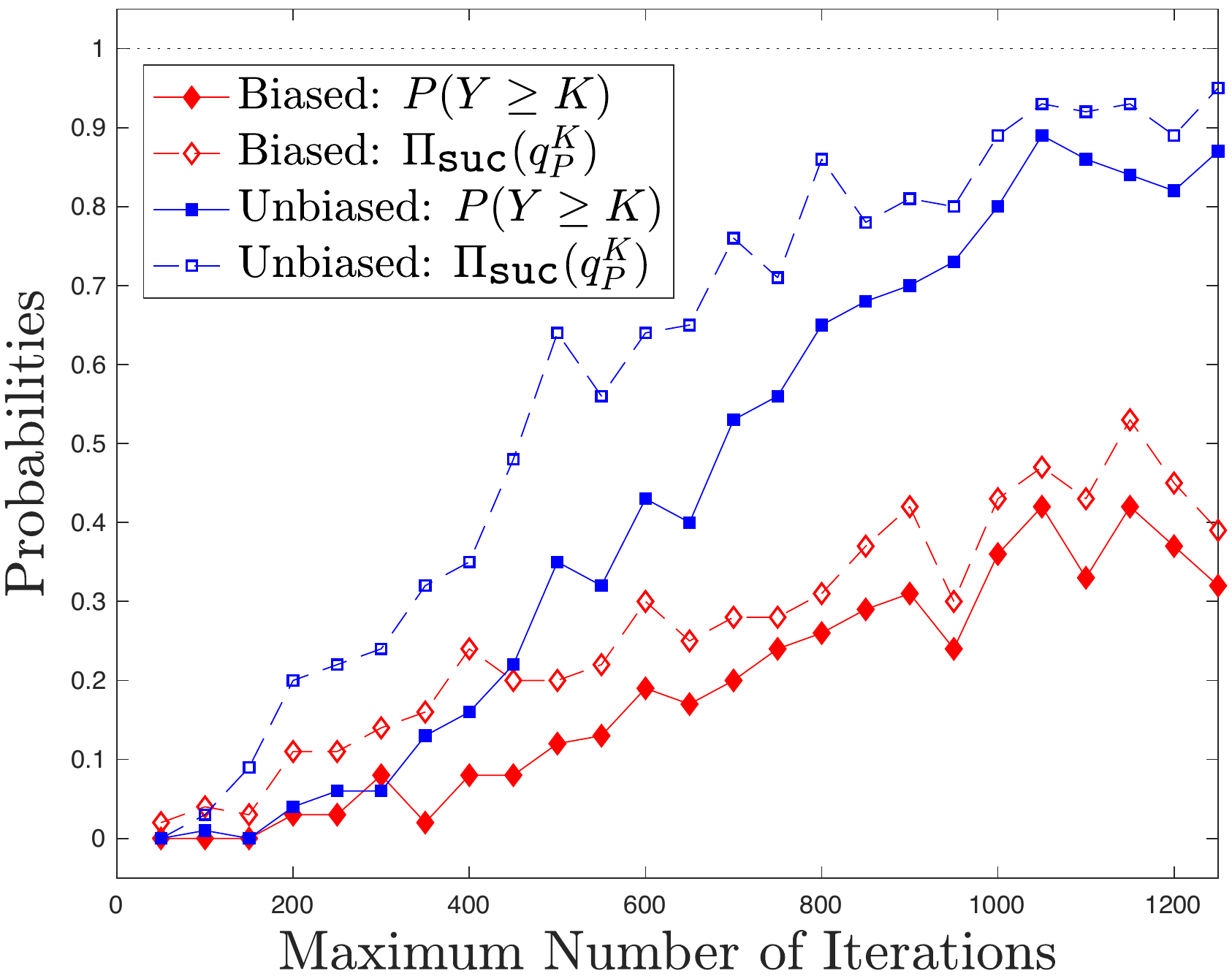}}  
    \caption{ Figures \ref{fig:proba1}-\ref{fig:proba3} and Figures \ref{fig:probb1}-\ref{fig:probb3}  refer to the case where sampling is biased towards $q_B^{\text{F,feas}}=6$ and $q_B^{\text{F,feas}}=9$, respectively.  Illustration of the probabilities $\mathbb{P}(Y\geq K)$  and $\Pi_{\text{suc}}(q_P^K)$ when biased and unbiased sampling is employed for various choices of $q_P^K$ and $n_{\text{max}}^{\text{pre}}$. The probability $\mathbb{P}(Y\geq K)$ is approximated as $\mathbb{P}(Y\geq K)\approx e_{\text{suc}}(n_{\text{max}}^{\text{pre}})/100$, where $e_{\text{suc}}(n_{\text{max}}^{\text{pre}})$ is the number of experiments in which all states of the given feasible prefix path $p_j$, $j\in\{1,2,3\}$ of length $K$ that connects $q_P^K$ to the root are added to the tree.  
    The probability $\Pi_{\text{suc}}(q_P^K)$ is approximated as $\Pi_{\text{suc}}(q_P^K)\approx g_{\text{suc}}(n_{\text{max}}^{\text{pre}})/100$, where $g_{\text{suc}}(n_{\text{max}}^{\text{pre}})$ is the number of experiments in which the final state $q_P^K\in\ccalQ_P^F$ is added to the tree. The probabilities $\mathbb{P}(Y\geq K)$ and $\Pi_{\text{suc}}(q_P^K)$ were estimated from a total of $100$ experiments for every value of $n_{\text{max}}^{\text{pre}}$. To generate the results of Figure \ref{fig:proba} the rewiring step is deactivated. 
}
  \label{fig:proba}
  \end{figure*}
  
 \subsection{Comparison with off-the-shelf model checkers: Summary}

\begin{figure}[t]
  \centering
  \includegraphics[width=0.9\linewidth]{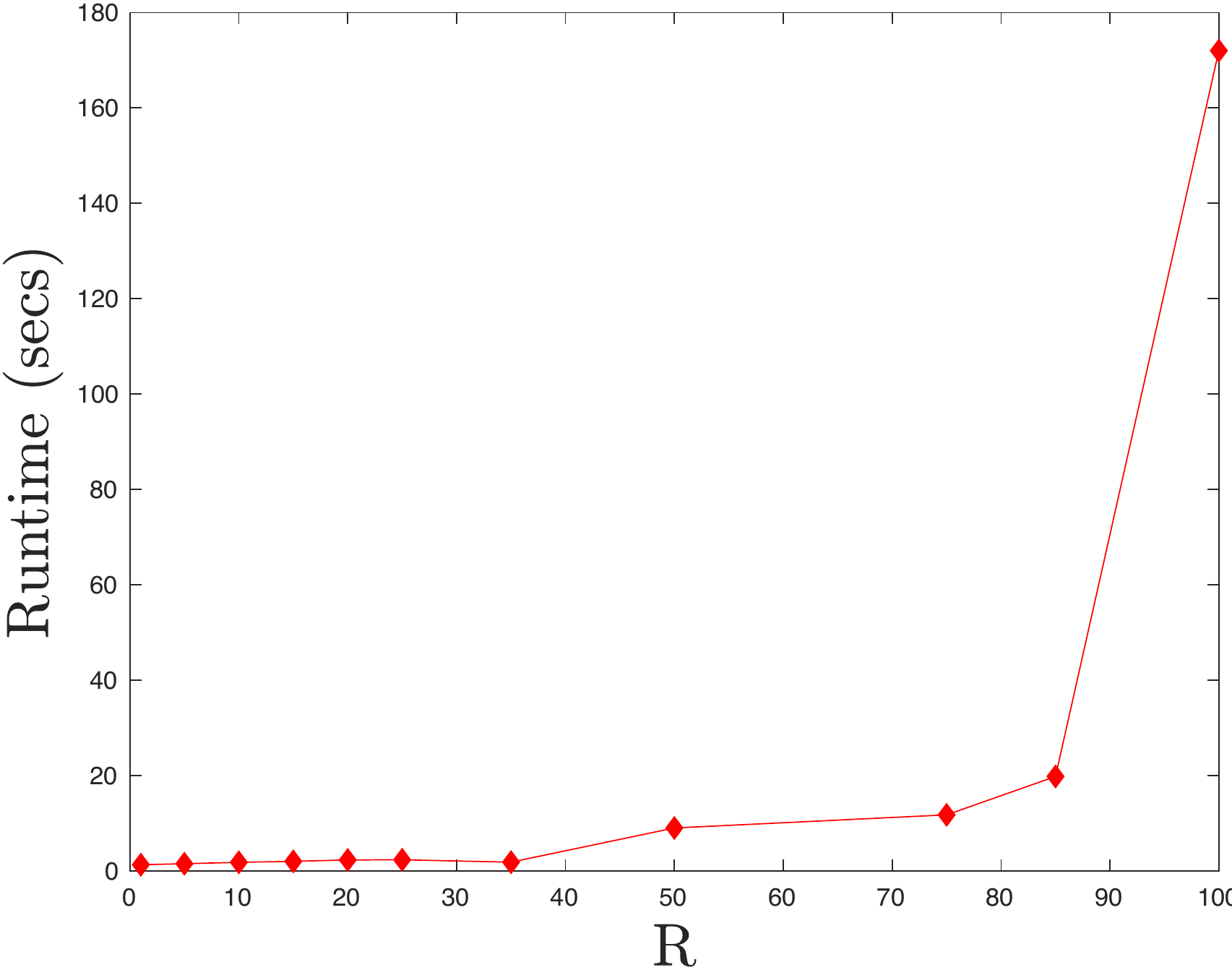}
  \caption{Case Study I [$N=10$ robots, $|\ccalQ_i|=1000$ states]: Comparison of the average runtime for $5$ experiments required to detect the first final state when collision avoidance constraints are imposed for various choices of $R$. }
  \label{fig:colAvR}
\end{figure} 

The numerical experiments presented in Sections \ref{sec:cos} and \ref{sec:largeNBA} show  that $\text{STyLuS}^{*}$ outperforms both NuSMV and nuXmv for large and dense wTS, whether they have self-loops or not, and regardless of the size of the NBA. On the other hand, for small and sparse transition systems with or without self-loops, NuSMV and nuXmv become faster than $\text{STyLuS}^{*}$, as the size of the NBA increases. Also, for small  and sparse transition systems without self-loops, the off-the-shelf model checkers find feasible paths faster than $\text{STyLuS}^{*}$. Nevertheless, NuSMV and nuXmv  can only find feasible paths while $\text{STyLuS}^{*}$ can detect the optimal plan with probability that converges to $1$ exponentially fast.

Finally, note that NuSMV and nuXmv cannot handle collision avoidance constraints that require all robots to maintain a distance between them that is always at least equal to $R$ units; see e.g., \cite{kantaros2018microrobots}. The reason is that a cost function cannot be embedded in the transition systems that are provided as an input to them. Instead, $\text{STyLuS}^{*}$ can check if such safety properties are satisfied every time a sample is taken. Figure \ref{fig:colAvR} shows the total time required by $\text{STyLuS}^{*}$ to detect the first final state for the planning problem with order $10^{31}$ when collision avoidance constraints are imposed. Specifically, observe that as $R$ increases the runtime of $\text{STyLuS}^{*}$ increases, since the number of samples that are rejected increases due to the imposed proximity restrictions.

\subsection{The effect of biased sampling}\label{sec:EffectBias}
Next, we illustrate the effect of introducing bias in the sampling process on the control synthesis performance. First, note that \cite{kantaros2017Csampling} using uniform/unbiased sampling functions can synthesize plans for the synthesis problems considered in Tables \ref{tab:table1}-\ref{tab:table2} with order of states up to $10^{10}$. These are more states than what existing optimal control synthesis algorithms can handle -- see \cite{kantaros2017Csampling} -- but orders of magnitudes less than the states that $\text{StyLuS}^*$ can solve.  

In what follows, we compare the performance of $\text{StyLuS}^*$ for biased and unbiased sampling; see Figure \ref{fig:proba}. In this case study, we consider the same task planning problem considered in \eqref{eq:task2} and Figure \ref{fig:optimWTS}. Note that the considered planning problem is small enough so that comparisons with uniform sampling functions can be provided. Particularly, in Figure \ref{fig:proba}, we show the probabilities $\mathbb{P}(Y\geq K)$ and $\Pi_{\text{suc}}(q_P^K)$ defined in Theorem \ref{thm:conv} for biased and unbiased sampling for various choices of $n_{\text{max}}^{\text{pre}}$ and final states $q_P^K$. Recall that (i) $\mathbb{P}(Y\geq K)$ captures the probability that all states that belong to a given feasible path $\texttt{p}$ of length $K$ that connects $q_P^K$ to the root have been added to the tree within $n_{\text{max}}^{\text{pre}}$ iterations and (ii) $\Pi_{\text{suc}}(q_P^K)\geq \mathbb{P}(Y\geq K)$, using either biased or unbiased sampling. Also, notice that as  $n_{\text{max}}^{\text{pre}}$ increases, both probabilities $\mathbb{P}(Y\geq K)$ and $\Pi_{\text{suc}}(q_P^K)$ converge to $1$, using either biased or unbiased (uniform) sampling density functions, as expected  due to Theorem \ref{thm:conv}.

In Figures \ref{fig:proba1}-\ref{fig:proba3} and  \ref{fig:probb1}-\ref{fig:probb3} the density functions are biased to the feasible final states $q_B^{F,\text{feas}}=6$ and  $q_B^{F,\text{feas}}=9$, respectively. In Figures \ref{fig:proba1} and  \ref{fig:probb1}, the state $q_P^K$ is selected as $q_P^K=((13, 10), 6)\in\ccalQ_P^F$ that is connected to the root through a feasible prefix part $\texttt{p}_1$ with length $K=13$. In Figures \ref{fig:proba2} and \ref{fig:probb2}  the state $q_P^K$ is selected as $q_P^K=((6, 10), 9)\in\ccalQ_P^F$ that is connected to the root through a feasible prefix part $\texttt{p}_2$  with length $K=6$. In fact, $\texttt{p}_2$ corresponds to the optimal prefix part. In Figures  \ref{fig:proba3} and \ref{fig:probb3}  the state $q_P^K$ is selected as $q_P^K=((13, 15), 13)\in\ccalQ_P^F$ that is connected to the root through a feasible prefix part $\texttt{p}_3$  with length $K=7$. Note that the feasible prefix paths $p_1$, $p_2$, and $p_3$ and the respective final states $q_P^K$ are randomly selected.

Observe first that in Figure \ref{fig:proba}, as $n_{\text{max}}^{\text{pre}}$ increases, both probabilities $\mathbb{P}(Y>K)$ and $\Pi_{\text{suc}}(q_P^K)$ converge to $1$ for both biased and uniform (unbiased) sampling, as expected by Theorem \ref{thm:conv}. This shows probabilistic completeness of $\text{STyLuS}^{*}$. Moreover, observe
in Figure \ref{fig:proba1} that for small $n_{\text{max}}^{\text{pre}}$, when sampling is biased towards $q_B^{F,\text{feas}}=6$, the probability $\Pi_{\text{suc}}(q_P^K)$ associated with the NBA final state $6$ is larger compared to the case of uniform sampling, as expected; see also Remark \ref{rem:conv}. The same also holds for Figure \ref{fig:proba2} although sampling is not biased towards the NBA final state $9$. 
Note, however, that this is a problem-specific behavior that depends on the structure of the NBA, and cannot be generalized. On the other hand, in Figure \ref{fig:proba3}, the probability $\Pi_{\text{suc}}(q_P^K)$ associated with NBA final state $13$ is larger for uniform sampling compared to sampling that is biased towards $q_B^{F,\text{feas}}=6$, as expected; see also Remark \ref{rem:conv}. The same observations are also made in Figures \ref{fig:probb1}-\ref{fig:probb3}. 
Observe also that when sampling is biased towards $q_B^{F,\text{feas}}=6$, the probability $\Pi_{\text{suc}}(q_P^K)$ associated with the NBA final state $6$ in Figure \ref{fig:proba1} is higher compared to the one in Figure \ref{fig:proba2}. The reason is that the probability $\Pi_{\text{suc}}(q_P^K)$ in Figures \ref{fig:proba1} and \ref{fig:proba2} is associated with NBA final state $6$ but sampling is steered towards $q_B^{F,\text{feas}}=6$ and $q_B^{F,\text{feas}}=9$, respectively. Moreover, notice in Figures \ref{fig:avgLen6} and \ref{fig:avgLen9} that when biased sampling is employed, feasible prefix paths with smaller length $K$ are detected compared to the case of uniform sampling, as also discussed in Remark \ref{rem:conv}. Finally, note that for a given $n_{\text{max}}^{\text{pre}}$, the size of the resulting trees is larger when uniform sampling is employed. For instance, for $n_{\text{max}}^{\text{pre}}=1000$, the tree consists of $4000$ and $2000$ nodes in average using uniform and biased sampling, respectively. The reason is that uniform sampling explores more states in the PBA while biased sampling  favors exploration in the vicinity of the shortest paths, in terms of hops, to the final states.

\begin{figure}[t]
\centering
  \subfigure[$q_B^{F,\text{feas}}=6$]{
    \label{fig:avgLen6}
  \includegraphics[width=0.48\linewidth]{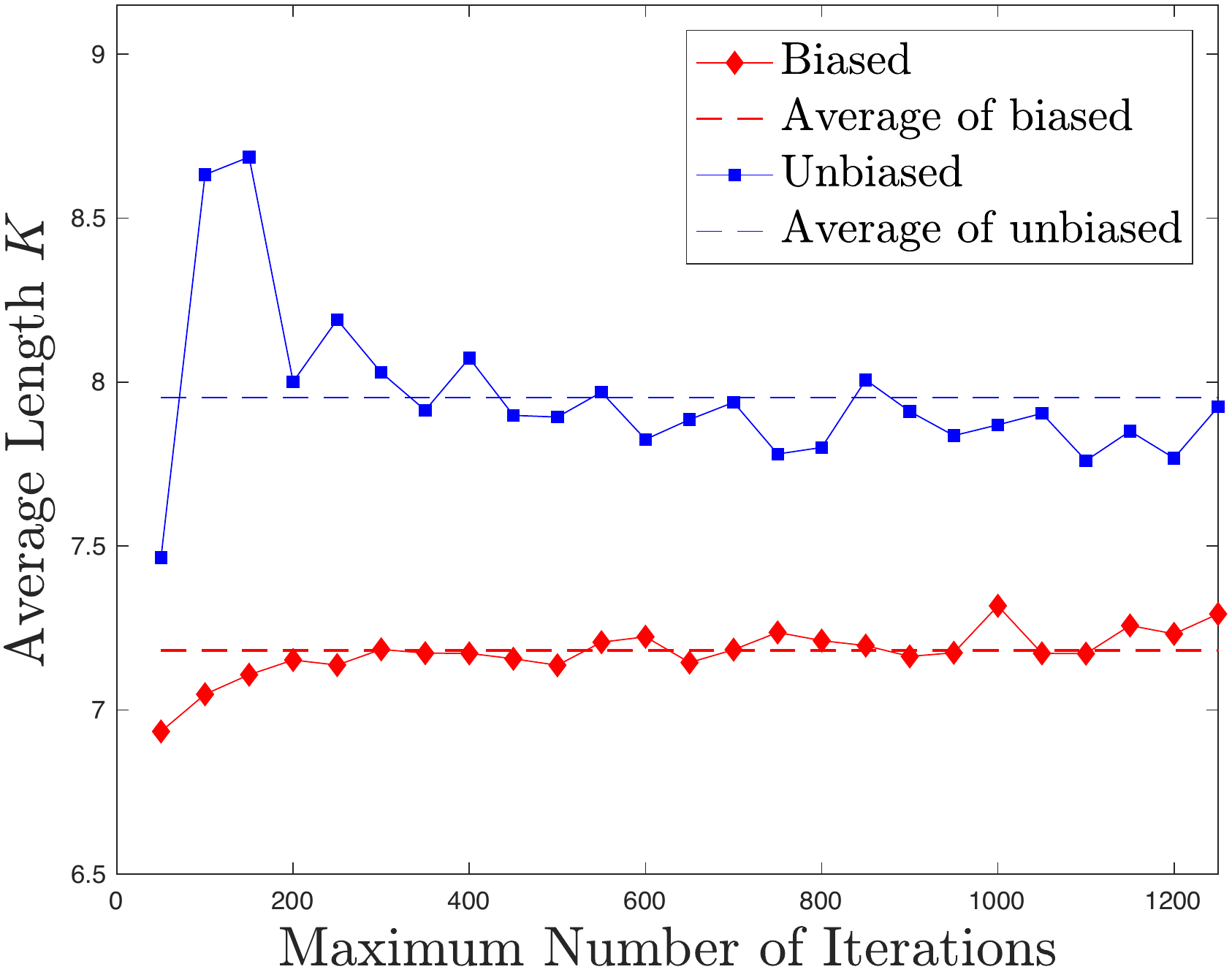}}
     \subfigure[$q_B^{F,\text{feas}}=9$]{
    \label{fig:avgLen9}
  \includegraphics[width=0.48\linewidth]{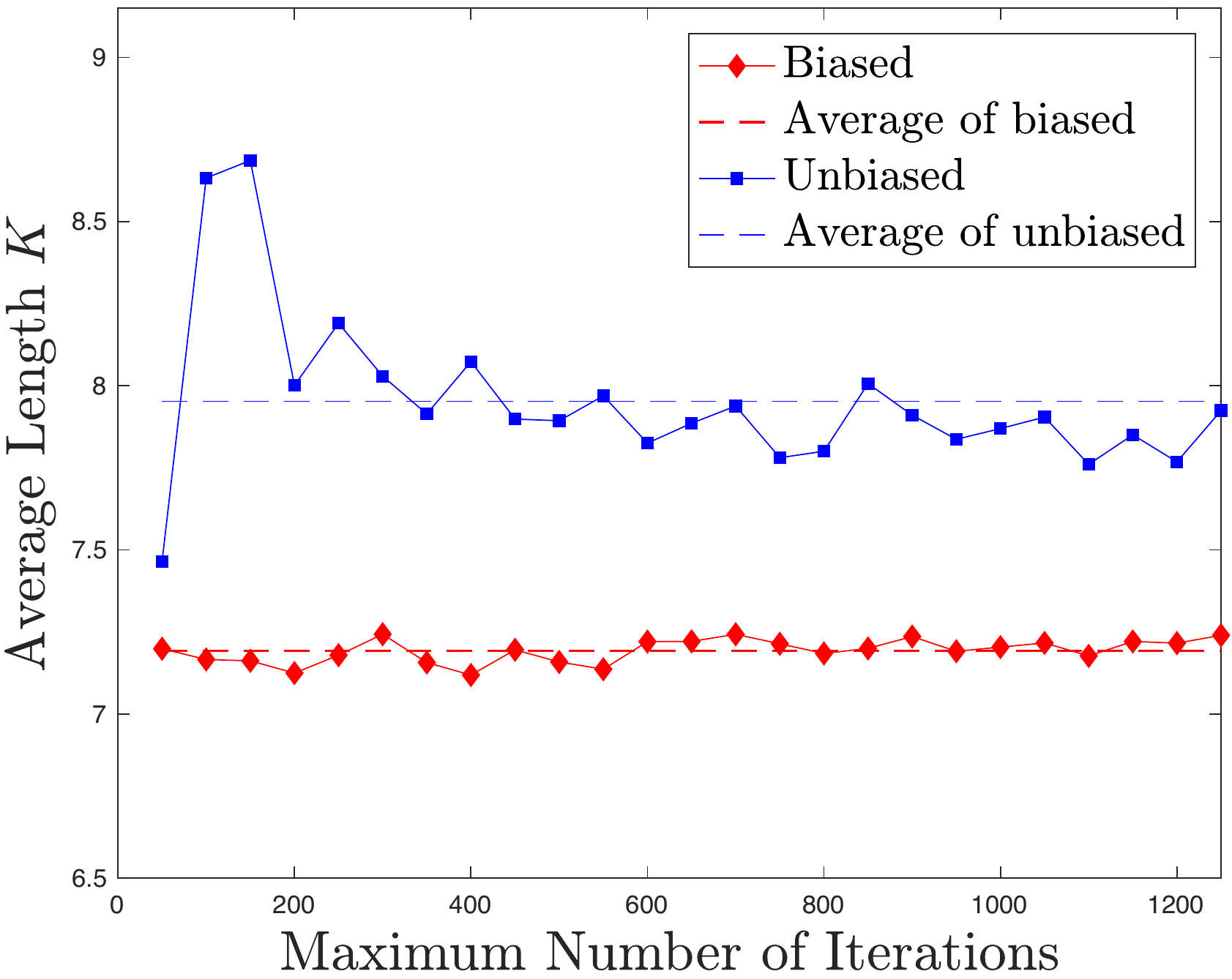}}
  \caption{The red diamonds and the blue squares in Figures \ref{fig:avgLen6} - \ref{fig:avgLen9} depict the average length $K$ of the detected prefix parts for $q_B^{F,\text{feas}}=6$ and $q_B^{F,\text{feas}}=9$, respectively, after 100 experiments for various choices of $n_{\text{max}}^{\text{pre}}$ when biased and unbiased sampling density functions are employed. To generate the results of Figure \ref{fig:avgLen} the rewiring step is deactivated.}
  \label{fig:avgLen}
  \end{figure}


\section{Conclusion}\label{sec:concl}
This paper proposed a new optimal control synthesis algorithm for
multi-robot systems with temporal logic specifications. We showed that the proposed algorithm is probabilistically complete, asymptotically optimal, and converges exponentially fast to the optimal solution. Finally we provided extensive comparative simulation studies showing that the proposed algorithm can synthesize optimal motion plans from product automata with state-spaces hundreds of orders of magnitude larger than those that state-of-the-art methods can manipulate. 

\appendices
\section{Proofs of Propositions}\label{sec:prop}

\subsection{Proof of Proposition \ref{prop:frand}}
Part (i) holds, since $p_{\text{rand}}^n>\epsilon$, by assumption, for all $n\geq 1$. Part (iii) trivially holds by definition of $f_{\text{rand}}^n$ in \eqref{eq:frand}. As for part (ii) observe that $\max(1/|\ccalD_{\text{min}}|,1/|\ccalV_T^n\setminus\ccalD_{\text{min}}^n|)\geq 1/|\ccalQ_P|$, since $\ccalD_{\text{min}}^n\subseteq\ccalV_T^n\subseteq\ccalQ_P$. Combining these two observations, we conclude that $f_{\text{rand}}^n(q_P|\ccalV_T^n)\geq \min(p_{\text{rand}}^{n+k},1-p_{\text{rand}}^{n+k})/|\ccalQ_P|\geq\min(\epsilon,1-\epsilon)/|\ccalQ_P|>0$. Therefore, there exists an infinite sequence $\{g^{n}(q_P|\ccalV_T^n)\}_{n=1}^{\infty}=\{\min(\epsilon,1-\epsilon)/|\ccalQ_P|\}_{n=1}^{\infty}$ such that $f_{\text{rand}}^n(q_P|\ccalV_T^n)\geq g^{n}(q_P|\ccalV_T^n)$ and $\sum_{n=1}^{\infty}g^{n}(q_P|\ccalV_T^n)=\infty$, since  $\min(\epsilon,1-\epsilon)/|\ccalQ_P|$ is a strictly positive constant term completing the proof.

\subsection{Proof of Proposition \ref{prop:fnew}}
Part (i) holds since $f_{\text{new},i}^n(q_{i}|q_{i}^{\text{rand},n})$ is bounded away from zero on $\ccalR_{\text{wTS}_i}(q_i^{\text{rand},n})$, for all robots $i$ as $p_{\text{new}}^n>\epsilon$, for all $n\geq 1$, by assumption. Part (iii) trivially holds by definition of $f_{\text{new}_i}^n$. As for part (ii), following the same logic as in the proof of Proposition \ref{prop:frand}, we can show that for the functions $f_{\text{new},i}^{n+k}$ in \eqref{eq:fnew}, \eqref{eq:fnewsuf1}, and \eqref{eq:fnewsuf2}, it holds that $f_{\text{new},i}^{n+k}(q_{\text{PTS}}|q_{\text{PTS}}^{\text{rand},n})\geq
\min(p_{\text{new}}^{n+k},1-p_{\text{new}}^{n+k})/(|\ccalQ_i|)\geq \min(\epsilon,1-\epsilon)/(|\ccalQ_i|)
$. Thus, we get that $f_{\text{new}}^n(q_{\text{PTS}}|q_{\text{PTS}}^{\text{rand}})=\Pi_{i=1}^{N}f_{\text{new},i}^n(q_{i}|q_{i}^{\text{rand},n})\geq (\min(\epsilon,1-\epsilon))^N/(\Pi_{i=1}^N|\ccalQ_i|)=(\min(\epsilon,1-\epsilon))^N/ |\ccalQ_\text{PTS}|>0$. Therefore, for any fixed and given node  $q_{\text{PTS}}^{\text{rand},n}\in\ccalV_T^n$, there exists an infinite sequence $h^{n+k}(q_{\text{PTS}}|q_{\text{PTS}}^{\text{rand},n+k})=(\min(\epsilon,1-\epsilon))^N/ |\ccalQ_\text{PTS}|>0$ so that 
$f_{\text{new}}^{n+k}(q_{\text{PTS}}|q_{\text{PTS}}^{\text{rand},n})\geq h^{n+k}(q_{\text{PTS}}|q_{\text{PTS}}^{\text{rand},n+k})$ and $\sum_{n=1}^{\infty}h^{n+k}(q_{\text{PTS}}|q_{\text{PTS}}^{\text{rand},n+k})=\infty$, since  $(\min(\epsilon,1-\epsilon))^N/ |\ccalQ_\text{PTS}|>0$ is a strictly positive constant term completing the proof.

\bibliographystyle{IEEEtran}
\bibliography{YK_bib}

\begin{thebibliography}{10}
\providecommand{\url}[1]{#1}
\csname url@samestyle\endcsname
\providecommand{\newblock}{\relax}
\providecommand{\bibinfo}[2]{#2}
\providecommand{\BIBentrySTDinterwordspacing}{\spaceskip=0pt\relax}
\providecommand{\BIBentryALTinterwordstretchfactor}{4}
\providecommand{\BIBentryALTinterwordspacing}{\spaceskip=\fontdimen2\font plus
\BIBentryALTinterwordstretchfactor\fontdimen3\font minus
  \fontdimen4\font\relax}
\providecommand{\BIBforeignlanguage}[2]{{%
\expandafter\ifx\csname l@#1\endcsname\relax
\typeout{** WARNING: IEEEtran.bst: No hyphenation pattern has been}%
\typeout{** loaded for the language `#1'. Using the pattern for}%
\typeout{** the default language instead.}%
\else
\language=\csname l@#1\endcsname
\fi
#2}}
\providecommand{\BIBdecl}{\relax}
\BIBdecl

\bibitem{kress2009temporal}
H.~Kress-Gazit, G.~E. Fainekos, and G.~J. Pappas, ``Temporal-logic-based
  reactive mission and motion planning,'' \emph{IEEE Transactions on Robotics},
  vol.~25, no.~6, pp. 1370--1381, 2009.

\bibitem{kress2007s}
------, ``Where's waldo? sensor-based temporal logic motion planning,'' in
  \emph{Robotics and Automation, 2007 IEEE International Conference on}.\hskip
  1em plus 0.5em minus 0.4em\relax IEEE, 2007, pp. 3116--3121.

\bibitem{bhatia2010sampling}
A.~Bhatia, L.~E. Kavraki, and M.~Y. Vardi, ``Sampling-based motion planning
  with temporal goals,'' in \emph{International Conference on Robotics and
  Automation (ICRA)}, Anchorage, AL, May 2010, pp. 2689--2696.

\bibitem{ulusoy2014receding}
A.~Ulusoy and C.~Belta, ``Receding horizon temporal logic control in dynamic
  environments,'' \emph{The International Journal of Robotics Research},
  vol.~33, no.~12, pp. 1593--1607, 2014.

\bibitem{chen2011synthesis}
Y.~Chen, X.~C. Ding, and C.~Belta, ``Synthesis of distributed control and
  communication schemes from global {LTL} specifications,'' in \emph{50th IEEE
  Conference on Decision and Control and European Control Conference}, Orlando,
  FL, USA, December 2011, pp. 2718--2723.

\bibitem{chen2012formal}
Y.~Chen, X.~C. Ding, A.~Stefanescu, and C.~Belta, ``Formal approach to the
  deployment of distributed robotic teams,'' \emph{IEEE Transactions on
  Robotics}, vol.~28, no.~1, pp. 158--171, 2012.

\bibitem{baier2008principles}
C.~Baier and J.-P. Katoen, \emph{Principles of model checking}.\hskip 1em plus
  0.5em minus 0.4em\relax MIT press Cambridge, 2008.

\bibitem{clarke1999model}
E.~M. Clarke, O.~Grumberg, and D.~Peled, \emph{Model checking}.\hskip 1em plus
  0.5em minus 0.4em\relax MIT press, 1999.

\bibitem{smith2011optimal}
S.~L. Smith, J.~Tumova, C.~Belta, and D.~Rus, ``Optimal path planning for
  surveillance with temporal-logic constraints,'' \emph{The International
  Journal of Robotics Research}, vol.~30, no.~14, pp. 1695--1708, 2011.

\bibitem{guo2015multi}
M.~Guo and D.~V. Dimarogonas, ``Multi-agent plan reconfiguration under local
  {LTL} specifications,'' \emph{The International Journal of Robotics
  Research}, vol.~34, no.~2, pp. 218--235, 2015.

\bibitem{kloetzer2010automatic}
M.~Kloetzer and C.~Belta, ``Automatic deployment of distributed teams of robots
  from temporal logic motion specifications,'' \emph{IEEE Transactions on
  Robotics}, vol.~26, no.~1, pp. 48--61, 2010.

\bibitem{ulusoy2013optimality}
A.~Ulusoy, S.~L. Smith, X.~C. Ding, C.~Belta, and D.~Rus, ``Optimality and
  robustness in multi-robot path planning with temporal logic constraints,''
  \emph{The International Journal of Robotics Research}, vol.~32, no.~8, pp.
  889--911, 2013.

\bibitem{ulusoy2014optimal}
A.~Ulusoy, S.~L. Smith, and C.~Belta, ``Optimal multi-robot path planning with
  ltl constraints: guaranteeing correctness through synchronization,'' in
  \emph{Distributed Autonomous Robotic Systems}.\hskip 1em plus 0.5em minus
  0.4em\relax Springer, 2014, pp. 337--351.

\bibitem{khalidi2018t}
D.~Khalidi and I.~Saha, ``T*: A heuristic search based algorithm for motion
  planning with temporal goals,'' \emph{arXiv preprint arXiv:1809.05817}, 2018.

\bibitem{schillinger2016decomposition}
P.~Schillinger, M.~B{\"u}rger, and D.~Dimarogonas, ``Decomposition of finite
  ltl specifications for efficient multi-agent planning,'' in \emph{13th
  International Symposium on Distributed Autonomous Robotic Systems}, London,
  UK, November, 2016.

\bibitem{schillinger2017simultaneous}
P.~Schillinger, M.~B{\"u}rger, and D.~V. Dimarogonas, ``Simultaneous task
  allocation and planning for temporal logic goals in heterogeneous multi-robot
  systems,'' \emph{The International Journal of Robotics Research}, 2018.

\bibitem{kantaros2017Csampling}
\BIBentryALTinterwordspacing
Y.~Kantaros and M.~M. Zavlanos, ``Sampling-based optimal control synthesis for
  multi-robot systems under global temporal tasks,'' \emph{IEEE Transactions on
  Automatic Control}, 2018. [Online]. Available: \url{DOI:
  10.1109/TAC.2018.2853558}
\BIBentrySTDinterwordspacing

\bibitem{kantaros15asilomar}
------, ``Intermittent connectivity control in mobile robot networks,'' in
  \emph{49th Asilomar Conference on Signals, Systems and Computers}, Pacific
  Grove, CA, USA, November, 2015, pp. 1125--1129.

\bibitem{kantaros2017sampling}
------, ``Sampling-based control synthesis for multi-robot systems under global
  temporal specifications,'' in \emph{Proc. 8th ACM/IEEE International
  Conference on Cyber-Physical Systems}, Pittsburgh, PA, USA, April 2017, pp.
  3--13.

\bibitem{cimatti2002nusmv}
A.~Cimatti, E.~Clarke, E.~Giunchiglia, F.~Giunchiglia, M.~Pistore, M.~Roveri,
  R.~Sebastiani, and A.~Tacchella, ``Nusmv 2: An opensource tool for symbolic
  model checking,'' in \emph{International Conference on Computer Aided
  Verification}.\hskip 1em plus 0.5em minus 0.4em\relax Springer, 2002, pp.
  359--364.

\bibitem{cavada2014nuxmv}
R.~Cavada, A.~Cimatti, M.~Dorigatti, A.~Griggio, A.~Mariotti, A.~Micheli,
  S.~Mover, M.~Roveri, and S.~Tonetta, ``The nuxmv symbolic model checker,'' in
  \emph{International Conference on Computer Aided Verification}.\hskip 1em
  plus 0.5em minus 0.4em\relax Springer, 2014, pp. 334--342.

\bibitem{kantaros2017Dsampling}
Y.~Kantaros and M.~M. Zavlanos, ``Distributed optimal control synthesis for
  multi-robot systems under global temporal tasks,'' in \emph{9th ACM/IEEE
  International Conference on Cyber-Physical Systems (ICCPS)}, Porto, Portugal,
  April 2018, pp. 162--173.

\bibitem{karaman2012sampling}
S.~Karaman and E.~Frazzoli, ``Sampling-based algorithms for optimal motion
  planning with deterministic $\mu$-calculus specifications,'' in
  \emph{American Control Conference (ACC)}, Montreal, Canada, June 2012, pp.
  735--742.

\bibitem{vasile2013sampling}
C.~I. Vasile and C.~Belta, ``Sampling-based temporal logic path planning,'' in
  \emph{IEEE/RSJ International Conference on Intelligent Robots and Systems},
  Tokyo, Japan, November 2013, pp. 4817--4822.

\bibitem{karaman2011sampling}
S.~Karaman and E.~Frazzoli, ``Sampling-based algorithms for optimal motion
  planning,'' \emph{The International Journal of Robotics Research}, vol.~30,
  no.~7, pp. 846--894, 2011.

\bibitem{janson2015fast}
L.~Janson, E.~Schmerling, A.~Clark, and M.~Pavone, ``Fast marching tree: A fast
  marching sampling-based method for optimal motion planning in many
  dimensions,'' \emph{The International journal of robotics research}, vol.~34,
  no.~7, pp. 883--921, 2015.

\bibitem{conner2003composition}
D.~C. Conner, A.~A. Rizzi, and H.~Choset, ``Composition of local potential
  functions for global robot control and navigation,'' in \emph{Intelligent
  Robots and Systems, Proceedings. 2003 IEEE/RSJ International Conference on},
  vol.~4, Las Vegas, NV, USA, October 2003, pp. 3546--3551.

\bibitem{belta2004constructing}
C.~Belta and L.~Habets, ``Constructing decidable hybrid systems with velocity
  bounds,'' in \emph{Decision and Control (CDC) 43rd Conference on},
  vol.~1.\hskip 1em plus 0.5em minus 0.4em\relax Bahamas: IEEE, December 2004,
  pp. 467--472.

\bibitem{belta2005discrete}
C.~Belta, V.~Isler, and G.~J. Pappas, ``Discrete abstractions for robot motion
  planning and control in polygonal environments,'' \emph{IEEE Transactions on
  Robotics}, vol.~21, no.~5, pp. 864--874, 2005.

\bibitem{kloetzer2006reachability}
M.~Kloetzer and C.~Belta, ``Reachability analysis of multi-affine systems,'' in
  \emph{International Workshop on Hybrid Systems: Computation and
  Control}.\hskip 1em plus 0.5em minus 0.4em\relax Springer, 2006, pp.
  348--362.

\bibitem{boskos2015decentralized}
D.~Boskos and D.~V. Dimarogonas, ``Decentralized abstractions for multi-agent
  systems under coupled constraints,'' in \emph{Conference on Decision and
  Control (CDC),}, Osaka, Japan, December 2015, pp. 7104--7109.

\bibitem{van2005using}
J.~P. Van~den Berg and M.~H. Overmars, ``Using workspace information as a guide
  to non-uniform sampling in probabilistic roadmap planners,'' \emph{The
  International Journal of Robotics Research}, vol.~24, no.~12, pp. 1055--1071,
  2005.

\bibitem{zucker2008adaptive}
M.~Zucker, J.~Kuffner, and J.~A. Bagnell, ``Adaptive workspace biasing for
  sampling-based planners,'' in \emph{Robotics and Automation, 2008. ICRA 2008.
  IEEE International Conference on}.\hskip 1em plus 0.5em minus 0.4em\relax
  IEEE, 2008, pp. 3757--3762.

\bibitem{ichter2017learning}
B.~Ichter, J.~Harrison, and M.~Pavone, ``Learning sampling distributions for
  robot motion planning,'' in \emph{Robotics and Automation (ICRA), 2018 IEEE
  International Conference on}.\hskip 1em plus 0.5em minus 0.4em\relax IEEE,
  2018.

\bibitem{huh2019probabilistically}
J.~Huh, O.~Arslan, and D.~D. Lee, ``Probabilistically safe corridors to guide
  sampling-based motion planning,'' \emph{arXiv preprint arXiv:1901.00101},
  2019.

\bibitem{zagoruyko2019monte}
S.~Zagoruyko, Y.~Labb{\'e}, I.~Kalevatykh, I.~Laptev, J.~Carpentier, M.~Aubry,
  and J.~Sivic, ``Monte-carlo tree search for efficient visually guided
  rearrangement planning,'' \emph{arXiv preprint arXiv:1904.10348}, 2019.

\bibitem{best2019dec}
G.~Best, O.~M. Cliff, T.~Patten, R.~R. Mettu, and R.~Fitch, ``Dec-mcts:
  Decentralized planning for multi-robot active perception,'' \emph{The
  International Journal of Robotics Research}, vol.~38, no. 2-3, pp. 316--337,
  2019.

\bibitem{browne2012survey}
C.~B. Browne, E.~Powley, D.~Whitehouse, S.~M. Lucas, P.~I. Cowling,
  P.~Rohlfshagen, S.~Tavener, D.~Perez, S.~Samothrakis, and S.~Colton, ``A
  survey of monte carlo tree search methods,'' \emph{IEEE Transactions on
  Computational Intelligence and AI in games}, vol.~4, no.~1, pp. 1--43, 2012.

\bibitem{kocsis2006bandit}
L.~Kocsis and C.~Szepesv{\'a}ri, ``Bandit based monte-carlo planning,'' in
  \emph{European conference on machine learning}.\hskip 1em plus 0.5em minus
  0.4em\relax Springer, 2006, pp. 282--293.

\bibitem{kantaros2018largescale}
Y.~Kantaros and M.~M. Zavlanos, ``Temporal logic optimal control for
  large-scale multi-robot systems: $10^{400}$ states and beyond,'' in
  \emph{55th IEEE Conference on Decision and Control (CDC)}, Miami, FL,
  December 2018, pp. 833--837, (accepted).

\bibitem{vardi1986automata}
M.~Y. Vardi and P.~Wolper, ``An automata-theoretic approach to automatic
  program verification,'' in \emph{1st Symposium in Logic in Computer Science
  (LICS)}.\hskip 1em plus 0.5em minus 0.4em\relax IEEE Computer Society, 1986.

\bibitem{ltl2nbaBelta}
A.~Ulusoy, S.~S.L., X.~Ding, C.~Belta, and D.~Rus,
  \texttt{\url{http://sites.bu.edu/hyness/lomp/} }, 2011.

\bibitem{gastin2001fast}
P.~Gastin and D.~Oddoux, ``Fast {LTL} to b{\"u}chi automata translation,'' in
  \emph{International Conference on Computer Aided Verification}.\hskip 1em
  plus 0.5em minus 0.4em\relax Springer, 2001, pp. 53--65.

\bibitem{sedgewick2011algorithms}
R.~Sedgewick and K.~Wayne, \emph{Algorithms}.\hskip 1em plus 0.5em minus
  0.4em\relax Addison-Wesley Professional, 2011.

\bibitem{motwani2010randomized}
R.~Motwani and P.~Raghavan, \emph{Randomized algorithms}.\hskip 1em plus 0.5em
  minus 0.4em\relax Chapman \& Hall/CRC, 2010.

\bibitem{choi2002approximating}
K.~Choi and A.~Xia, ``Approximating the number of successes in independent
  trials: Binomial versus poisson,'' \emph{Annals of Applied Probability}, pp.
  1139--1148, 2002.

\bibitem{hong2011computing}
Y.~Hong, ``On computing the distribution function for the sum of independent
  and non-identical random indicators,'' \emph{Dep. Statit., Virginia Tech,
  Blacksburg, VA, USA, Tech. Rep. 11\_2}, 2011.

\bibitem{guo2017distributed}
M.~Guo and M.~M. Zavlanos, ``Distributed data gathering with buffer constraints
  and intermittent communication,'' in \emph{Robotics and Automation (ICRA),
  IEEE International Conference on}, Singapore, May-June 2017, pp. 279--284.

\bibitem{kantaros2016distributedInterm}
Y.~Kantaros and M.~M. Zavlanos, ``Distributed intermittent connectivity control
  of mobile robot networks,'' \emph{IEEE Transactions on Automatic Control},
  vol.~61, no.~7, pp. 3109--3121, July 2017.

\bibitem{kantaros2019temporal}
Y.~Kantaros, M.~Guo, and M.~M. Zavlanos, ``Temporal logic task planning and
  intermittent connectivity control of mobile robot networks,'' \emph{IEEE
  Transactions on Automatic Control}, 2019.

\bibitem{holzmann2004spin}
G.~J. Holzmann, \emph{The SPIN model checker: Primer and reference
  manual}.\hskip 1em plus 0.5em minus 0.4em\relax Addison-Wesley Reading, 2004,
  vol. 1003.

\bibitem{fraser2009evaluation}
G.~Fraser and A.~Gargantini, ``An evaluation of model checkers for
  specification based test case generation,'' in \emph{2009 International
  Conference on Software Testing Verification and Validation}.\hskip 1em plus
  0.5em minus 0.4em\relax IEEE, 2009, pp. 41--50.

\bibitem{choi2007nusmv}
Y.~Choi, ``From nusmv to spin: Experiences with model checking flight guidance
  systems,'' \emph{Formal Methods in System Design}, vol.~30, no.~3, pp.
  199--216, 2007.

\bibitem{kantaros2018microrobots}
\BIBentryALTinterwordspacing
Y.~Kantaros, B.~Johnson, S.~Chowdhury, D.~Cappelleri, and M.~M. Zavlanos,
  ``Control of magnetic microrobot teams for temporal micromanipulation
  tasks,'' \emph{IEEE Transactions on Robotics}, 2018. [Online]. Available:
  \url{DOI: 10.1109/TAC.2018.2799561}
\BIBentrySTDinterwordspacing

\end{thebibliography}
\end{document}